\documentclass[10pt,letterpaper,twoside]{article} 

\usepackage{amsmath,amsthm,amssymb,bbold,mathrsfs} 
\usepackage{graphicx}
\graphicspath{{./}, {./GBE-figures/}}

\usepackage[title]{appendix}
\usepackage{enumitem}
\setlist[itemize]{itemsep=0.0cm}
\usepackage{hyperref}
\usepackage{natbib}
\usepackage{authblk}

\newtheorem{thm}{Theorem}[section]
\newtheorem{cor}{Corollary}[section]
\newtheorem{lem}{Lemma}[section]

\newtheorem{cond}{Condition}[section]

\newtheorem{rem}{Remark}[section]
\newtheorem{myexample}{Example}[section]

\numberwithin{equation}{section}

\def\1{\mathbf{1}}
\def\I{\mathbb{1}}
\def\asto{\overset{a.s.}{\to}}

\def\e{e}
\def\fe{\phi}
\def\k{\kappa}
\def\re{\Re}
\def\rn{\Re^n}
\def\tr{\top}
 \def\va{v_{\text{\tiny TD}}}
\def\C{C}
\def\E{\mathbb{E}}
\def\F{\mathcal{F}}
\def\L{\mathcal{L}}
\def\M{\mathcal{M}}

\def\Pr{\mathbf{P}}
\def\S{\mathcal{S}}

\def\P{P}
\def\Gm{\Gamma}

\makeatletter
\newcommand\appendix@section[1]{%
  \refstepcounter{section}%
  \orig@section*{Appendix \@Alph\c@section: #1}%
  \addcontentsline{toc}{section}{Appendix \@Alph\c@section: #1}%
}
\let\orig@section\section
\g@addto@macro\appendix{\let\section\appendix@section}
\makeatother

\begin{document} 

\markboth{Generalized Bellman Equations and TD Learning}{}

\title{On Generalized Bellman Equations and \\ Temporal-Difference Learning\thanks{An abridged version of this article appeared at \emph{The $30$th Canadian Conference on Artificial Intelligence (CAI)}, May 2017, and a reformatted copy is to appear in \emph{Journal on Machine Learning Research}, 2018.}}

\author[1]{Huizhen Yu}
\author[2]{A. Rupam Mahmood}
\author[1]{Richard S. Sutton}
\affil[1]{RLAI Lab, Department of Computing Science\\ 
University of Alberta, Canada}
\affil[2]{Kindred Inc.\\ 
243 College St\\
Toronto, ON M5T 1R5, Canada}

\date{}
\maketitle

\begin{abstract}
We consider off-policy temporal-difference (TD) learning in discounted Markov decision processes, where the goal is to evaluate a policy in a model-free way by using observations of a state process generated without executing the policy. To curb the high variance issue in off-policy TD learning, we propose a new scheme of setting the $\lambda$-parameters of TD, based on generalized Bellman equations. Our scheme is to set $\lambda$ according to the eligibility trace iterates calculated in TD, thereby easily keeping these traces in a desired bounded range. Compared with prior work, this scheme is more direct and flexible, and allows much larger $\lambda$ values for off-policy TD learning with bounded traces. 
As to its soundness, using Markov chain theory, we prove the ergodicity of the joint state-trace process under nonrestrictive conditions, and we show that associated with our scheme is a generalized Bellman equation (for the policy to be evaluated) that depends on both the evolution of $\lambda$ and the unique invariant probability measure of the state-trace process. These results not only lead immediately to a characterization of the convergence behavior of least-squares based implementation of our scheme, but also prepare the ground for further analysis of gradient-based implementations.
\end{abstract}


\clearpage
\tableofcontents
\clearpage

\section{Introduction}

We consider discounted Markov decision processes (MDPs) and off-policy temporal-difference (TD) learning methods for approximate policy evaluation with linear function approximation. The goal is to evaluate a policy in a model-free way by using observations of a state process generated without executing the policy. Off-policy learning is an important part of the reinforcement learning methodology \citep{SUB} and has been studied in the areas of operations research and machine learning. (For an incomplete list of references, see e.g., \citealp{gi-sampling,offpolicytd-pss,offpolicytd-psd,rj-tdimportance,gtd08,gtd09,maei11,Yu-siam-lstd,dnp14,bruno14,pmrl,wis14,pmtd3,SuMW14,sbeed18}.)
Available TD algorithms, however, tend to have very high variances due to the use of importance sampling, an issue that limits their applicability in practice. The purpose of this paper is to introduce a new TD learning scheme that can help address this problem.

Our work is motivated by the recently proposed Retrace algorithm \citep{offpolicytd-mshb} and ABQ algorithm \citep{abq}, and by the Tree-Backup algorithm \citep{offpolicytd-pss} that existed earlier. These algorithms, as explained by \citet{abq}, all try to use the $\lambda$-parameters of TD to curb the high variance issue in off-policy learning. In particular, they all choose the values of $\lambda$ according to the current state or state-action pair in such a way that guarantees the boundedness of the eligibility traces in TD learning, which can help reduce significantly the variance of the TD iterates. A limitation of these algorithms, however, is that they tend to be over-conservative and restrict $\lambda$ to small values, whereas small $\lambda$ can result in large approximation bias in TD solutions. 

In this paper, we propose a new scheme of setting the $\lambda$-parameters of TD, based on generalized Bellman equations. Our scheme is to set $\lambda$ according to the eligibility trace iterates calculated in TD, thereby easily keeping those traces in a desired bounded range. Compared with the schemes used in the previous work just mentioned, this is a direct way to bound the traces in TD, and it is also more flexible and allows much larger $\lambda$ values for off-policy learning.

Regarding generalized Bellman equations, in our context, they will correspond to a family of dynamic programming equations for the policy to be evaluated. These equations all have the true value function as their unique solution, and their associated operators have contraction properties, like the standard Bellman operator. We will refer to the associated operators as generalized Bellman operators or Bellman operators for short. Some authors have considered, at least conceptually, the use of an even broader class of equations for policy evaluation. For example, \citet{gen-td} have considered treating the policy evaluation problem as a parameter estimation problem in the statistical framework of estimating equations, and in their framework, any equation that has the true value function as the unique solution can be used to estimate the value function. The family of generalized Bellman equations we consider has a more specific structure. They generalize multistep Bellman equations, and they are associated with randomized stopping times and arise from the strong Markov property (see Section~\ref{sec-3.1} for details).  

Generalized Bellman equations and operators are powerful tools. In classic MDP theory they have been used in some intricate optimality analyses (e.g., \citealp{ScS87}). Their computational use, however, seems to emerge primarily in the field of reinforcement learning. 
Through the $\lambda$-parameters and eligibility traces, TD learning is naturally connected with, not a single Bellman operator, but a family of Bellman operators, with different choices of $\lambda$ or different rules of calculating the eligibility trace iterates corresponding to different Bellman operators. Early efforts that use this aspect to broaden the scope of TD algorithms and to analyze such algorithms include Sutton's work \citeyearpar{Sut95} on learning at multiple timescales and Tsitsiklis' work on generalized TD algorithms in the tabular case (see the book by~\citealp[Chap.\ 5.3]{BET}). 
In the context of off-policy learning, there are more recent approaches that try to utilize this connection of TD with generalized Bellman operators to make TD learning more efficient \citep{offpolicytd-pss,yb-bellmaneq,offpolicytd-mshb,abq}. This is also our aim, in proposing the new scheme of setting the $\lambda$-parameters. 

Our analyses of the new TD learning scheme will focus on its theoretical side. 
Using Markov chain theory, we prove the ergodicity of the joint state and trace process under nonrestrictive conditions (see Theorem~\ref{thm-1}), and we show that associated with our scheme is a generalized Bellman equation (for the policy to be evaluated) that depends on both the evolution of $\lambda$ and the unique invariant probability measure of the state-trace process (see Theorem~\ref{thm-2} and Corollary~\ref{cor-tdcmp}). These results not only lead immediately to a characterization of the convergence behavior of least-squares based implementation of our scheme (see Corollary~\ref{cor-1} and Remark~\ref{rem-1}), but also prepare the ground for further analysis of gradient-based implementations. (The latter analysis has been carried out recently by~\cite{gtd-conv17}; see Remark~\ref{rmk-gtd}.)

In addition to the theoretical study, we also present the results from a preliminary numerical study that compares several ways of setting $\lambda$ for the least-squares based off-policy algorithm. The results demonstrate the advantages of the proposed new scheme with its greater flexibility.

We remark that although we focus exclusively on policy evaluation in this paper, approximate policy evaluation methods are highly pertinent to finding near-optimal policies in MDPs. They can be applied in approximate policy iteration, in policy-gradient algorithms for gradient estimation or in direct policy search (see e.g., \citealp{konda-thesis,ce03}). In addition to solving MDPs, they can also be used in artificial intelligence and robotics applications as a means to generate experience-based world models (see e.g., \citealp{Sut09}). It is, however, beyond the scope of this paper to discuss these applications of our results.

The rest of the paper is organized as follows. 
In Section~\ref{sec-2}, after a brief background introduction, we present our scheme of TD learning with bounded traces, and we establish the ergodicity of the joint state-trace process.
In Section~\ref{sec-3}, we first discuss generalized Bellman operators associated with randomized stopping times, and we then derive the generalized Bellman equation associated with our scheme. 
In Section~\ref{sec-4}, we present the experimental results on the least-squares based implementation of our scheme.
Appendices~\ref{appsec-1}-\ref{appsec-oblproj} include a proof for generalized Bellman operators and materials about approximation properties of TD solutions that are too long to include in the main text.

\section{Off-Policy TD Learning with Bounded Traces} \label{sec-2}

We describe the off-policy policy evaluation problem and the algorithmic form of TD learning in Section~\ref{sec-2.1}. We then present our scheme of history-dependent $\lambda$ in Section~\ref{sec-2.2}, and analyze the properties of the resulting eligibility trace iterates and the convergence of the corresponding least-squares based algorithm in Section~\ref{sec-2.3}.

\subsection{Preliminaries} \label{sec-2.1}

The off-policy learning problem we consider in this paper concerns two Markov chains on a finite state space $\S = \{ 1, \ldots, N\}$. The first chain has transition matrix $\P$, and the second $\P^o$.
Whatever physical mechanisms that induce the two chains shall be denoted by $\pi$ and $\pi^o$, and referred to as the target policy and behavior policy, respectively. The second Markov chain we can observe; however, it is the system performance of the first Markov chain that we want to evaluate. 

Specifically, we consider a one-stage reward function $r_\pi : \S \to \Re$ and an associated discounted total reward criterion with state-dependent discount factors $\gamma(s) \in [0,1], s \in \S$. 
Let $\Gm$ denote the $N \times N$ diagonal matrix with diagonal entries $\gamma(s)$. We assume that $\P$ and $\P^o$ satisfy the following conditions:
{\samepage
\begin{cond}[Conditions on the target and behavior policies] \label{cond-pol} \hfill\vspace*{-3pt}
\begin{itemize}
\item[\rm (i)] $P$ is such that the inverse $(I - \P \Gm)^{-1}$ exists, and 
\item[\rm (ii)] $\P^o$ is such that for all $s,s' \in \S$, $\P^o_{ss'} = 0 \Rightarrow \P_{ss'}=0$,  and moreover, $\P^o$ is irreducible.
\end{itemize}
\end{cond}
}

The performance of $\pi$ is defined as the expected discounted total rewards for each initial state $s \in \S$:
\begin{equation} \label{def-vpi}
 \textstyle{v_{\pi}(s) : = \E^\pi_s \left[ \, r_\pi(S_0) +  \sum_{t=1}^{\infty} \gamma(S_{1})\, \gamma(S_{2}) \, \cdots \, \gamma(S_{t})  \cdot r_\pi(S_t) \right],}
\end{equation} 
where the notation $\E^\pi_s$ means that the expectation is taken with respect to (w.r.t.) the Markov chain $\{S_t\}$ starting from $S_0 =s$ and induced by $\pi$ (i.e., with transition matrix $P$). The function $v_\pi$ is well-defined under Condition~\ref{cond-pol}(i). It is called the \emph{value function} of $\pi$, and by standard MDP theory (see e.g., \citealp{puterman94}), we can write it in matrix/vector notation as
$$v_\pi = r_{\pi} + \P \Gm \, v_\pi, \qquad \text{i.e.}, \quad v_\pi = (I - \P \Gm)^{-1} r_\pi.$$
The first equation above is known as the Bellman equation (or dynamic programming equation) for a stationary policy (cf.\ Footnote~\ref{ft-1}).

We compute an approximation of $v_\pi$ of the form $v(s) = \fe(s)^\top \theta$, $s \in \S$, where $\theta \in \rn$ is a parameter vector  and $\fe(s)$ is an $n$-dimensional feature representation for each state $s$ (here $\fe(s),\theta$ are column vectors and the symbol $\text{}^\top$ stands for transpose). Data available for this computation are:\vspace*{-0.07cm}
\begin{itemize}
\item[(i)] a realization of the Markov chain $\{S_t\}$ with transition matrix $\P^o$ generated by $\pi^o$, and 
\item[(ii)] rewards $R_t = r(S_t, S_{t+1})$ associated with state transitions, where the function $r$ relates to $r_\pi(s)$ as $r_{\pi}(s) = \E^\pi_s [ r(s, S_1) ]$ for all $s \in \S$.
\footnote{One can add to $R_t$ a zero-mean finite-variance noise term. This makes little difference to our analyses, so we have left it out for notational simplicity.}
\end{itemize}
To find a suitable parameter $\theta$ for the approximation $\fe(s)^\top \theta$, we use the off-policy TD learning scheme. 
Define $\rho(s, s') = \P_{ss'}/\P^o_{ss'}$ (the importance sampling ratio),
\footnote{\label{ft-1}%
Our problem formulation entails both value function and state-action value function estimation for a stationary policy in the standard MDP context. In these applications, it is the state-action space of the MDP that corresponds to the state space $\S$ here. 
In particular, for value function estimation, $S_t$ here corresponds to the pair of previous action and current state in the MDP, whereas for state-value function estimation, $S_t$ here corresponds to the current state-action pair in the MDP. The ratio $\rho(s,s') = \P_{ss'}/\P^o_{ss'}$ then comes out as the ratio of action probabilities under $\pi$ and $\pi^o$, the same as what appears in most of the off-policy learning literature. For the details of these correspondences, see~\cite[Examples 2.1, 2.2]{Yu-siam-lstd}. 
The third application is in a simulation context where $\P^o$ corresponds to a simulated system and both $\P^o, \P$ are known so that the ratio $\rho(s,s')$ is available. 
Such simulations are useful, for example, in studying system performance under perturbations, and in speeding up the computation when assessing the impacts of events that are rare under the dynamics $\P$.}
and write 
$$\rho_t = \rho(S_t, S_{t+1}), \qquad \gamma_t=\gamma(S_t).$$
Given an initial $\e_0 \in \rn$, for each $t \geq 1$, the eligibility trace vector $\e_t \in \rn$ and the scalar temporal-difference term $\delta_t(v)$ for any approximate value function $v: \S \to \Re$ are calculated according to
\begin{align}
   \e_t & =  \lambda_t \, \gamma_t \,  \rho_{t-1} \, \e_{t-1} +  \fe(S_t), \label{eq-td1}  \\
   \delta_t(v) & = \rho_t \, \big( R_{t} + \gamma_{t+1} v(S_{t+1}) - v(S_t) \big). \label{eq-td2}
\end{align}
Here $\lambda_t \in [0,1], t \geq 1$, are important parameters in TD learning, the choice of which we shall elaborate on shortly.

There exist a number of TD algorithms that use $\e_t$ and $\delta_t$ to generate a sequence of parameters $\theta_t$ for approximate value functions. One such algorithm is LSTD~\citep{lstd,Yu-siam-lstd}, which obtains $\theta_t$ by solving the linear equation for $\theta \in \rn$,
\begin{equation} \label{eq-lstd}
     \textstyle{ \tfrac{1}{t} \sum_{k=0}^{t-1}  \, \e_k \, \delta_k (v) = 0, \quad v = \Phi \theta}
\end{equation}     
(if it admits a solution), where $\Phi$ is a matrix with row vectors $\phi(s)^\top, s \in \S$. LSTD updates the equation (\ref{eq-lstd}) iteratively by incorporating one by one the observation of $(S_t, S_{t+1}, R_t)$ at each state transition.
We will discuss primarily this algorithm in the paper, as its behavior can be characterized directly using our subsequent analyses of the joint state-trace process. 

As mentioned earlier, our analyses will also provide bases for analyzing other gradient-based TD algorithms (e.g.,~\citealp{gtd08,gtd09,maei11,pmrl}) by using stochastic approximation theory \citep{KuY03,Bor08,KaB15}. Because of the complexity of this subject, however, we will not delve into it in the present paper, and we refer the reader to the recent work~\citep{gtd-conv17} for details.

\subsection{Our Scheme of History-dependent $\lambda$} \label{sec-2.2}

We now come to the choices of $\lambda_t$ in the trace iterates (\ref{eq-td1}). For TD with function approximation, one often lets $\lambda_t$ be a constant or a function of $S_t$ \citep{Sut88,tr-disc,SUB}.
If neither the behavior policy nor the $\lambda_t$'s are further constrained, $\{\e_t\}$ can have unbounded variances and is also unbounded in many natural situations (see e.g.,~\citealp[Section 3.1]{Yu-siam-lstd}), and this makes off-policy TD learning challenging.
\footnote{However, asymptotic convergence can still be ensured for several algorithms~\citep{Yu-siam-lstd,yu-etdarx,etd-wkconv}, thanks partly to a powerful law of large numbers for stationary processes.}
If we let the behavior policy to be close enough to the target policy so that $\P^o \approx \P$, then variance can be reduced, but it is not a satisfactory solution, for the applicability of off-policy learning would be seriously limited. 

Without restricting the behavior policy, as mentioned earlier, the two recent papers \citep{offpolicytd-mshb,abq}, as well as the closely related early work by \citet{offpolicytd-pss}, exploit state-dependent $\lambda$'s to control variance. Their choices of $\lambda_t$ are such that $\lambda_t \gamma_t \rho_{t-1} < 1$ for all $t$, so that the trace iterates $\e_t$ are made bounded, which can help reduce the variance of the iterates.

Motivated by this prior work, our proposal is to set $\lambda_t$ according to $\e_{t-1}$ directly, so that we can keep $\e_t$ in a desired range straightforwardly and at the same time, allow a much larger range of values for the $\lambda$-parameters. As a simple example, if we use $\lambda_t$ to scale the vector $\gamma_t \rho_{t-1} \e_{t-1}$ to be within a ball with some given radius, then we keep $\e_t$ always bounded.

In the rest of this paper, we shall focus on analyzing the iteration (\ref{eq-td1}) with a particular choice of $\lambda_t$ of the kind just mentioned. We want to be more general than the preceding simple example. However, since the dependence on the trace $\e_{t-1}$ would make $\lambda_t$ dependent on the entire past history $(S_0, \ldots, S_{t-1})$, we also want to retain certain Markovian properties that are very useful for convergence analysis. This leads us to consider \emph{$\lambda_t$ being a certain function of the previous trace and past states}. More specifically, we will let $\lambda_t$ be a function of the previous trace $\e_{t-1}$ and a certain memory state that is a summary of the states observed so far. The formulation is as follows.

\subsubsection{Formulation and Examples}

We denote the memory state at time $t$ by $y_t$. For simplicity, we assume that $y_t$ can only take values from a finite set $\M$, and its evolution is Markovian: $y_t = g(y_{t-1}, S_t)$ for some given function $g$. The joint process $\{(S_t, y_t)\}$ is then a simple finite-state Markov chain. Each $y_t$ is a function of the history $(S_0, \ldots, S_t)$ and $y_0$. 
We further require, besides the irreducibility of $\{S_t\}$ (cf.\ Condition~\ref{cond-pol}(ii)), that
\begin{cond}[Evolution of memory states] \label{cond-mem} 
Under the behavior policy $\pi^o$, the Markov chain $\{(S_t, y_t)\}$ on $\S \times \M$ has a single recurrent class.
\end{cond} 
This recurrence condition is nonrestrictive: If the Markov chain has multiple recurrent classes, each recurrent class can be treated separately by using the same arguments we present in this paper. However, we remark that the finiteness assumption on $\M$ is a simplification. We choose to work with finite $\M$ mainly for the reason that with the traces lying in a continuous space, to study the joint state and trace process, we need to resort to properties of Markov chains on infinite spaces. With an infinite $\M$, we would need to introduce more technical conditions that are not essential to our analysis and can obscure our main arguments.

We thus let $y_t$ and $\lambda_t$ evolve as
\begin{equation}
  y_t = g(y_{t-1}, S_t), \qquad \lambda_t=\lambda(y_t, \e_{t-1})  \label{eq-td3}
\end{equation}  
where $\lambda : \M \times \rn \to [0,1]$. We require the function $\lambda$ to satisfy two conditions. 
\begin{cond}[Conditions for $\lambda(\cdot)$] \label{cond-lambda} 
For some norm $\|\cdot\|$ on $\rn$, the following hold for each memory state $y  \in \M$:\vspace*{-0.1cm}
\begin{itemize}
\item[\rm (i)] For any $\e , \e' \in \rn$, $\|\lambda(y, \e) \, \e - \lambda(y, \e') \, \e' \| \leq \| \e - \e'\|$.
\item[\rm (ii)] For some constant $\C_y$, $\| \gamma(s') \rho(s, s') \cdot \lambda(y, \e)  \, \e\| \leq \C_y$ for all $\e \in \rn$ and all possible state transitions $(s, s')$ that can lead to the memory state $y$.
\end{itemize}
\end{cond}
In the above, the second condition is to restrict $\{\e_t\}$ in a desired range (as it makes $\|\e_t\| \leq \max_{y \in \M} \C_y + \max_{s \in \S} \| \phi(s)\|$). The first condition is about the continuity of the function $\lambda(y, \e) \, \e$ in the trace variable $\e$ for each memory state $y$, and it plays a key role in the subsequent analysis, where we will use this condition to ensure that the traces $\e_t$ and the states $(S_t, y_t)$ jointly form a Markov chain with appealing properties. We shall defer a further discussion on the technical roles of these conditions to the end of Section~\ref{sec-2.3} (cf.\ Remark~\ref{rem-conditions-thm1}).

Let us give a few simple examples of choosing $\lambda$ that satisfy Condition~\ref{cond-lambda}. We will later use these examples in our experimental study (Section~\ref{sec-4}).

\begin{myexample} \rm \label{ex-scaling}
We consider again the simple scaling example mentioned earlier and describe it using the terminologies just introduced. In this example, we let $y_t = (S_{t-1}, S_t)$. For each $y = (s,s')$, we define the function $\lambda(y, \cdot)$ so that when multiplied with $\lambda(y, \e)$, the vector $\gamma(s') \rho(s,s') \, \e$ is scaled down whenever its length exceeds a given threshold $\C_{ss'}$:
\begin{equation} \label{eq-ex1}
\lambda\big(y, \e \big) =  \begin{cases}
  1 & \text{if} \  \ \gamma(s') \rho(s, s') \| \e\|_2 \leq \C_{ss'}; \\
  \tfrac{\C_{ss'}}{\gamma(s') \rho(s,s') \| \e\|_2} &  \text{otherwise}. \end{cases}
\end{equation}
Condition~\ref{cond-lambda}(i) is satisfied because for $y=(s,s')$ with $\gamma(s') \rho(s, s') =0$, $\lambda\big(y, \e\big) \e = \e$, whereas for $y=(s,s')$ with $\gamma(s') \rho(s, s') \not=0$, $\lambda\big(y, \e\big) \e$ is simply the Euclidean projection of $\e$ onto the ball (centered at the origin) with radius $\C_{ss'}/(\gamma(s') \rho(s, s'))$ and is therefore Lipschitz continuous in $\e$ with modulus $1$ w.r.t.\ $\|\cdot\|_2$.
Corresponding to (\ref{eq-ex1}), the update rule (\ref{eq-td1}) of $\e_t$ becomes
\begin{equation} \label{eq-ex1b}
 \e_t = \begin{cases}
  \gamma_t \,  \rho_{t-1} \, \e_{t-1} +  \fe(S_t) \ & \text{if} \ \ \gamma_t \rho_{t-1} \| \e_{t-1}\|_2 \leq \C_{S_{t-1}S_t}; \\
  \C_{S_{t-1}S_t} \cdot \tfrac{\e_{t-1}}{\|\e_{t-1}\|_2} +  \fe(S_t)
    &  \text{otherwise}.
    \end{cases}
 \end{equation}   

Note that this scheme of setting $\lambda$ encourages the use of large $\lambda_t$: $\lambda_t=1$ will be chosen whenever possible. A variation of the scheme is to multiply the right-hand side (r.h.s.) of (\ref{eq-ex1}) by another factor $\beta_{ss'} \in [0,1]$, so that $\lambda_t$ can be at most $\beta_{S_{t-1}S_t}$.
In particular, one such variation is to simply multiply the r.h.s.\ of (\ref{eq-ex1}) by a constant $\beta \in (0,1)$ so that $\lambda_t \leq \beta < 1$ for all $t$. 
\qed \end{myexample}

\begin{myexample} \rm \label{ex-retrace}
The Retrace algorithm \citep{offpolicytd-mshb} modifies the trace updates in off-policy TD learning by truncating the importance sampling ratios by $1$. 
In particular, for the off-policy TD($\lambda$) algorithm with a constant $\lambda=\beta \in (0,1]$, Retrace modifies the trace updates to be
\begin{equation}  \label{eq-ex2a}
   \e_t  =  \beta \, \gamma_t \cdot  \min\{1, \rho_{t-1}\}  \cdot \e_{t-1} +  \fe(S_t). 
\end{equation}   
As pointed out by \cite{abq}, to retain the original interpretation of $\lambda$ as a bootstrapping parameter in TD learning, 
we can rewrite the above update rule of Retrace equivalently as
\begin{equation}  \label{eq-ex2}
       \e_t  =  \lambda_t \, \gamma_t \,  \rho_{t-1}  \, \e_{t-1} +  \fe(S_t) \qquad \text{for} \ \ \lambda_t = \beta \cdot \tfrac{\min\{ 1, \rho_{t-1}\}}{\rho_{t-1}} \  \ \text{(with $0/0=0$)}.
\end{equation}
Each $\lambda_t$ here is a function of $(S_{t-1}, S_t)$ only and does not depend on $\e_{t-1}$, so this choice of $\lambda$-parameters automatically satisfies Condition~\ref{cond-lambda}(i) with the memory states being $y_t=(S_{t-1}, S_t)$. When the discount factors $\gamma(s)$ are all strictly less than $1$, $\|\e_t\|$ for all $t$ are bounded by a deterministic constant that depends on the initial $\e_0$. 
Then for each initial $\e_0$, Retrace's choice of $\lambda$ coincides with a choice in our framework, since the $C$-parameters in Condition~\ref{cond-lambda}(ii) can be made vacuously large so that the condition is satisfied by all the traces $\e_t$ that could be encountered by Retrace. Thus in this case our framework for choosing $\lambda$ effectively encompasses the particular choice used by Retrace.

One can make variations on Retrace's trace update rule. For example, instead of truncating each importance sampling ratio $\rho(s,s')$ by $1$, one can truncate it by a constant $K_{ss'} \geq 1$, and then use a scaling scheme similar to Example~\ref{ex-scaling} to bound the traces. The simplest such variation is to choose two memory-independent positive constants $K$ and $C$, and replace the definition of $\lambda_t$ in (\ref{eq-ex2}) by the following:
with 
$\tilde \lambda_t =  \tfrac{\min\{ K, \rho_{t-1}\}}{\rho_{t-1}}$ (where we treat $0/0=0$),
\begin{equation} \label{eq-var-retrace1}
   \lambda_t  = \begin{cases}  
    \beta \,  \tilde \lambda_t  & \text{if} \ \ \tilde \lambda_t \gamma_t \,  \rho_{t-1}  \, \|\e_{t-1}\|_2 \leq C; \\
   \beta \, \tilde \lambda_t \cdot \tfrac{C}{\tilde \lambda_t \gamma_t \,  \rho_{t-1}  \, \|\e_{t-1}\|_2}  & \text{otherwise}. \end{cases} 
\end{equation}
Correspondingly, instead of (\ref{eq-ex2a}), the update rule of $\e_t$ becomes
\begin{equation} \label{eq-var-retrace2}
 \e_t  =  \begin{cases}  
\beta \, \gamma_t \cdot  \min\{K, \rho_{t-1}\}  \cdot \e_{t-1} +  \fe(S_t) \ & \text{if} \ \ \gamma_t \cdot  \min\{K, \rho_{t-1}\}  \cdot \|\e_{t-1}\|_2 \leq C;\\
\beta \, C \cdot \tfrac{\e_{t-1}}{\|\e_{t-1}\|_2} +  \fe(S_t) & \text{otherwise}.
\end{cases}
\end{equation}
These variations of Retrace are similar to Example~\ref{ex-scaling} and satisfy Condition~\ref{cond-lambda}.
\footnote{To see this, let the memory states be $y_t=(S_{t-1},S_t)$. For each $y=(s,s')$, let $\lambda(y, \e)$ be defined according to (\ref{eq-var-retrace1}), and let $C_y$ in Condition~\ref{cond-lambda}(ii) be $C_{ss'} = \tfrac{\rho(s,s')}{\min\{K, \, \rho(s,s')\}}\cdot C$ (treat $0/0=0$). Then note that $\|\lambda(y, e) e - \lambda(y, e') e' \|_2 \leq \beta \min\big\{\tfrac{K}{\rho(s,s')}, \, 1\big\} \cdot \| e - e'\|_2 \leq \beta \| e - e'\|_2$.}
\qed \end{myexample}

\subsubsection{Comparison with Previous Work} \label{sec-compare-retrace}

For policy evaluation, the Retrace algorithm \citep{offpolicytd-mshb} and the ABQ algorithm \citep{abq} are very similar (ABQ was actually developed independently of Retrace before the \citet{offpolicytd-mshb} paper was published, although the ABQ paper itself was released much later). Both Retrace and ABQ include the Tree-Backup algorithm \citep{offpolicytd-pss} as a special case. They can use additional parameters to select $\lambda$ from a range of values, whereas Tree-Backup specifies $\lambda$, implicitly, in a particular way (which has the advantage of requiring no knowledge of the behavior policy) and does not have the freedom in choosing $\lambda$. Because of the relations between these algorithms, when comparing our method to them, we will compare it with Retrace only. In the experimental study given later in Section~\ref{sec-4} on the performance of LSTD for various ways of setting $\lambda$, we will compare our scheme of choosing $\lambda$ with that of Retrace for $\beta = 1$, which lets Retrace use the largest $\lambda$ that it can take. 

We see in Example~\ref{ex-retrace} that the eligibility trace update rule of Retrace can be written in two equivalent forms, (\ref{eq-ex2a}) and (\ref{eq-ex2}). The second form (\ref{eq-ex2}) has the advantage that the $\lambda$-parameters involved are shown explicitly. In TD learning, the $\lambda$-parameters directly affect the associated Bellman operators and can be meaningfully interpreted as stopping probabilities (see Section~\ref{sec-3}), whereas the importance sampling ratio terms in the eligibility trace iterates are essentially unchanged, for they have to be there in order to correct for the discrepancy between the behavior and target policies. For this reason,  we prefer (\ref{eq-ex2}) to (\ref{eq-ex2a}) and prefer thinking in terms of the selection of $\lambda$-parameters to that of what occurs \emph{apparently} to those importance sampling ratio terms in the trace updates.

As mentioned in Example~\ref{ex-retrace}, the \citet{offpolicytd-mshb} paper does not make the connection between (\ref{eq-ex2a}) and (\ref{eq-ex2}). \citet{abq} recognized the role of the $\lambda$-parameters and made explicit use of it to derive the ABQ algorithm. However, in the ABQ paper, the discussion and the presentation of the algorithm still emphasize the apparent changes in those importance sampling ratio terms in the trace iterates. This is an unsatisfactory point in that paper that we hope we have clarified with our present work.

We mentioned in the introduction that Retrace, ABQ and Tree-Backup are too conservative and tend to use too small $\lambda$ values. Let us now make this statement more precise and also explain the reason behind. 

These algorithms tend to behave effectively like TD($\lambda$) with small constant $\lambda$, despite that they can have $\lambda_t = 1$ at some time steps $t$. This is due to the nature of TD learning with time-varying $\lambda$, which is very different from that of TD with constant $\lambda$. For time-varying $\lambda$, a large $\lambda_t$ at one time step need not mean that we are using the information of the cumulative rewards over a long time horizon to estimate the value at the state $S_t$ encountered at time $t$. Because the next $\lambda_{t+1}$ could be very small or even zero, forcing a TD algorithm to ``bootstrap'' immediately. When large $\lambda_t$'s are interleaved with small ones, we are effectively in the situation of TD with small $\lambda$. This could occur to our proposed scheme as well if, for example, in Example~\ref{ex-scaling} the thresholds $\C_{ss'}$ are set too small. When we use larger thresholds, we allow larger $\lambda$. By comparison, Retrace, ABQ, and Tree-Backup constrain the state-dependent $\lambda$-parameters to be small enough so that all the products $\lambda_t \gamma_t \rho_{t-1} < 1$, and this makes them prone to the small-$\lambda$ issue just mentioned. (See the experiments in Section~\ref{sec-4.2} for demonstrations.)

While we consider Retrace for approximate policy evaluation, the \citet{offpolicytd-mshb} paper actually focuses primarily on finding an optimal policy for an MDP, in the tabular case, and it has demonstrated good empirical performance of Retrace and Tree-Backup for that purpose. Despite this, its results are not adequate yet to establish asymptotic optimality of these algorithms in the online optimistic policy iteration setting (personal communication with Munos), and it is still an open theoretical question whether online TD algorithms can solve an MDP like the Q-learning algorithm \citep{wat89,tsi94}, when positive $\lambda$ (small or not) and rapidly changing target policies are involved. 

We also mention that for policy evaluation, \citet[Section 3.1]{offpolicytd-mshb} have also conceived the use of generalized Bellman operators, although they did not relate these operators explicitly to history-dependent $\lambda$'s and did not study corresponding algorithms in this general case.

\subsection{Ergodicity Result} \label{sec-2.3}

The properties of the joint state-trace process $\{(S_t, y_t, \e_t)\}$ are important for understanding and characterizing the behavior of our proposed TD learning scheme. We study them in this subsection. Most importantly, we shall establish the ergodicity of the state-trace process. The result will be useful in convergence analysis of several associated TD algorithms~\citep{gtd-conv17}, although in this paper we discuss only the LSTD algorithm. In the next section we will also use the ergodicity result when we relate the LSTD equation (\ref{eq-lstd}) to a generalized Bellman equation for the target policy in order to interpret the LSTD solutions.

We note that to obtain the results in this subsection, we will follow similar lines of argument used in \citep{Yu-siam-lstd} for analyzing off-policy LSTD with constant $\lambda$. However, because $\lambda$ is now history-dependent, some proof steps in \citep{Yu-siam-lstd} no longer apply. We shall explain this in more detail after we prove the main result of this subsection.

As another side note, one can introduce nonnegative coefficients $i(y)$ for memory states $y$ to weight the state features (similarly to the use of ``interest'' weights in the ETD algorithm~\citep{SuMW14}) and update $\e_t$ according to
\begin{equation}
   \e_t =  \lambda_t \, \gamma_t \,  \rho_{t-1} \, \e_{t-1} +  i(y_t) \, \fe(S_t).
\end{equation}   
The results given below apply to this update rule as well.

Let us start with two basic properties of $\{(S_t, y_t, \e_t)\}$ that follow directly from our choice of the $\lambda$ function:\vspace*{-0.1cm}
\begin{itemize}
\item[(i)] By Condition \ref{cond-lambda}(i), for each $y$, $\lambda(y, \e) \e$ is a continuous function of $\e$, and thus $\e_t$ depends continuously on $\e_{t-1}$. This, together with the finiteness of $\S \times \M$, ensures that $\{(S_t, y_t, \e_t)\}$ is a weak Feller Markov chain.
\footnote{This means that for any bounded continuous function $f$ on $\S \times \M \times \rn$ (endowed with the usual topology), with $X_t = (S_t, y_t, \e_t)$, 
$\E \big[ f(X_1) \mid X_0 = x \big]$
is a continuous function of $x$~\cite[Prop.\ 6.1.1]{MeT09}.}
\item[(ii)] Then, by a property of weak Feller Markov chains~\cite[Theorem 12.1.2(ii)]{MeT09}, the boundedness of $\{\e_t\}$ ensured by~Condition \ref{cond-lambda}(ii) implies that $\{(S_t, y_t, \e_t)\}$ has at least one invariant probability measure. 
\end{itemize}
The third property, given in the lemma below, concerns the behavior of $\{\e_t\}$ for different initial $\e_0$. It is an important implication of Condition~\ref{cond-lambda}(i); actually, it is our purpose of introducing the condition~\ref{cond-lambda}(i) in the first place. In the lemma, $\asto$ stands for ``converges almost surely to.''

\begin{lem} \label{lem-1}
Let $\{\e_t\}$ and $\{\hat \e_t\}$ be generated by the iteration (\ref{eq-td1}) and (\ref{eq-td3}), using the same trajectory of states $\{S_t\}$ and initial $y_0$, but with different initial $\e_0$ and $\hat \e_0$, respectively. Then under Conditions~\ref{cond-pol}(i) and~\ref{cond-lambda}(i), $\e_t - \hat \e_t \asto 0$. 
\end{lem}

\begin{proof}
The proof is similar to that of \cite[Lemma 3.2]{Yu-siam-lstd}. 
Let $\Delta_t = \| \e_{t} - \hat \e_{t} \|$, and let $\F_t$ denote the $\sigma$-algebra generated by $S_k, k \leq t$.
Note that under our assumption, in the generation of the two trace sequences $\{\e_t\}$ and $\{\hat \e_t\}$, the states $\{S_t\}$ and the memory states $\{y_t\}$ are the same, but the $\lambda$-parameters are different. Let us denote them by $\{\lambda_t\}$ and $\{\hat \lambda_t\}$ for the two trace sequences, respectively.
Then by (\ref{eq-td1}),
$\e_t - \hat \e_t =  \gamma_t \,  \rho_{t-1} \, (\lambda_t \e_{t-1} - \hat \lambda_t \hat \e_{t-1})$, 
and by Condition~\ref{cond-lambda}(i), $\| \lambda_t \e_{t-1} - \hat \lambda_t \hat \e_{t-1} \| \leq \| \e_{t-1} - \hat \e_{t-1} \|$.
Hence $\| \e_{t} - \hat \e_{t} \| \leq \gamma_t \,  \rho_{t-1} \, \| \e_{t-1} - \hat \e_{t-1} \|$, 
so
$ \E \big[ \Delta_t \big| \F_{t-1} \big] \leq \E \big[ \gamma_t \,  \rho_{t-1} \big| \F_{t-1} \big] \cdot  \Delta_{t-1} \leq \Delta_{t-1}.$
This shows $\{(\Delta_t, \F_t)\}$ is a nonnegative supermartingale. By the supermartingale convergence theorem \cite[Theorem 10.5.7 and Lemma 4.3.3]{Dud02}, $\{\Delta_t\}$ converges a.s.\ to a nonnegative random variable $\Delta_\infty$ with $\E [ \Delta_\infty ] \leq \liminf_{t \to \infty} \E [ \Delta_t ]$.
From the inequality $\| \e_{t} - \hat \e_{t} \| \leq \gamma_t \,  \rho_{t-1} \, \| \e_{t-1} - \hat \e_{t-1} \|$ for all $t$, we have $\Delta_t \leq \Delta_0 \cdot \prod_{k=1}^t \gamma_k \rho_{k-1}$, from which a direct calculation shows 
$\E \big[ \Delta_t \big] \leq \Delta_0 \cdot \1^\tr (\P \Gamma)^t \1$ where $\1$ denotes the $n$-dimensional vector of all $1$'s. 
As $t \to \infty$, $(\P \Gamma)^t$ converges to the zero matrix under Condition~\ref{cond-pol}(i).
Therefore, $\liminf_{t \to \infty} \E [ \Delta_t ] = 0$ and consequently, we must have $\Delta_\infty = 0$ a.s., i.e., $\Delta_t \asto 0$. 
\end{proof}

We use Lemma~\ref{lem-1} and ergodicity properties of weak Feller Markov chains \citep{Meyn89} to prove the ergodicity theorem below. A direct application to LSTD will be discussed immediately after the theorem, before we give its proof.

To state the result, we need some terminology and notation. For $\{(S_t, y_t, \e_t)\}$ starting from the initial condition $x = (s, y, \e)$, we write $\Pr_x$ for its probability distribution, and we write ``$\Pr_x$-a.s.'' for ``almost surely with respect to $\Pr_x$.''  
The \emph{occupation probability measures} are denoted by $\{\mu_{x,t}\}$, and they are random probability measures on $\S \times \M \times \rn$ given by
$$\textstyle{\mu_{x,t}(D) : = \frac{1}{t} \sum_{k=0}^{t-1} \I \big( (S_k, y_k, \e_k) \in D \big) \qquad \forall \  \text{Borel sets $D \subset \S \times \M \times \rn$}, }$$
where $\I(\cdot)$ is the indicator function. We are interested in the asymptotic convergence of these occupation probability measures in the sense of \emph{weak convergence}: for probability measures $\{\mu_t\}$ and $\mu$ on a metric space, $\{\mu_t\}$ converges weakly to $\mu$ if $\int f d \mu_{t} \to \int f d \mu$ as $t \to \infty$, for every bounded continuous function $f$.

We shall also consider the Markov chain $\{(S_t, S_{t+1}, y_t, \e_t)\}$, whose occupation probability measures are defined likewise. 
This Markov chain is essentially the same as $\{(S_t,y_t,\e_t)\}$, but it is more convenient for applying our ergodicity result to TD algorithms because the temporal-difference term $\delta_t(v)$ involves $(S_t, S_{t+1}, \e_t)$. Regarding invariant probability measures of the two Markov chains, obviously, if $\zeta$ is an invariant probability measure of $\{(S_t, y_t, \e_t)\}$, then an invariable probability measure of $\{(S_t, S_{t+1}, y_t, \e_t)\}$ is the probability measure $\zeta_1$ composed from the marginal $\zeta$ and the conditional distribution of $S_1$ given $(S_0, y_0, \e_0)$ specified by $\P^o$; i.e., 
\begin{equation} \label{eq-zeta1}
\zeta_1(D) = \textstyle{\int \sum_{s' \in \S} \P^o_{ss'}  \, \I\big( (s, s', y, e) \in D\big)  \, \zeta\big(d (s,y,e)\big)} \qquad \forall \ \text{Borel sets $D \subset \S^2 \times \M \times \rn$}.
\end{equation}
(In the above, we used the notation $\int f(x) \,\zeta(dx)$ to write the integral of $f$ w.r.t.\ $\zeta$, and the notation $\zeta\big(d(s,y,e)\big)$ is the same as $\zeta(dx)$ with $x = (s,y,e)$.)

\begin{thm} \label{thm-1}
Let Conditions~\ref{cond-pol}-\ref{cond-lambda} hold. Then $\{(S_t, y_t, \e_t)\}$ is a weak Feller Markov chain and has a unique invariant probability measure $\zeta$. For each initial condition $x:=(s,y,\e)$ of $(S_0, y_0, \e_0)$, the occupation probability measures $\{\mu_{x,t}\}$ converge weakly to $\zeta$, $\Pr_x$-a.s.

Likewise, the same holds for $\{(S_t, S_{t+1}, y_t, \e_t)\}$, whose unique invariant probability measure is as given in (\ref{eq-zeta1}).
\end{thm}

If the initial distribution of $(S_0, y_0, \e_0)$ is $\zeta$, the state-trace process $\{(S_t, y_t, \e_t)\}$ is stationary. Let $\E_\zeta$ denote expectation w.r.t.\ this stationary process. We now state a corollary of the above theorem for LSTD, before we prove the theorem. 

Consider the sequence of equations in $v$,
$ \tfrac{1}{t} \sum_{k=0}^{t-1}  \, \e_k \, \delta_k (v) = 0$, appeared in (\ref{eq-lstd}) for LSTD. From the definition (\ref{eq-td2}) of $\delta_t(v)$, 
$$ \delta_t(v) = \rho_t \, \big( R_{t} + \gamma_{t+1} v(S_{t+1}) - v(S_t) \big),$$
we see that for fixed $v$, every $\e_k \, \delta_k(v)$ can be expressed as $f(S_k, S_{k+1}, \e_k)$ for a continuous function $f$. 
Since the traces and hence the entire process lie in a bounded set under Condition~\ref{cond-lambda}(ii), the weak convergence of the occupation probabilities measures of $\{(S_t, S_{t+1}, y_t, \e_t)\}$ shown by Theorem~\ref{thm-1} implies that this sequence of equations has an asymptotic limit that can be expressed in terms of the stationary state-trace process as follows.

\begin{cor} \label{cor-1}
Let Conditions~\ref{cond-pol}-\ref{cond-lambda} hold. Then for each initial condition of $(S_0, y_0, \e_0)$, almost surely, the sequence of linear equations in $v$, $ \tfrac{1}{t} \sum_{k=0}^{t-1}  \, \e_k \, \delta_k (v) = 0$, tends asymptotically to $\E_\zeta [ \, \e_0 \, \delta_0(v) ] = 0$ (also a linear equation in $v$), in the sense that the random coefficients in the former equations converge to the corresponding coefficients in the latter equation as $t \to \infty$.
\end{cor}

In the rest of this section we prove Theorem~\ref{thm-1}. Broadly speaking, the line of argument is as follows: We first prove the weak convergence of occupation probability measures to the same invariant probability measure, for each initial condition. This will in turn imply the uniqueness of the invariant probability measure. 

After the proof we will first comment in Remark~\ref{rem-prf-thm1} on the differences between our proof and that of a similar result in the previous work \citep{Yu-siam-lstd}. We will then comment in Remark~\ref{rem-conditions-thm1} about the technical roles of Condition~\ref{cond-lambda} (which concerns the choice of the function $\lambda(\cdot)$) and whether some part of that condition can be relaxed.

\begin{proof}[Proof of Theorem~\ref{thm-1}]
As we discussed before Lemma~\ref{lem-1}, under Conditions \ref{cond-lambda}, $\{(S_t, y_t, \e_t)\}$ is weak Feller and has at least one invariant probability measure $\zeta$. Then, by \cite[Prop.\ 4.1]{Meyn89}, there exists a set $D \subset \S \times \M \times \rn$ with $\zeta$-measure $1$ such that for each initial condition $x = (s, y, \e) \in D$, the occupation probability measures $\{\mu_{x,t}\}$ converge weakly, $\Pr_x$-a.s., to an invariant probability measure $\mu_x$ that depends only on the initial condition $x$. To prove the theorem using this result, we need to show that (i) all these $\{\mu_x \mid x \in D\}$ are the same invariant probability measure, and (ii) for all $x \not\in D$, $\{\mu_{x,t}\}$ has the same weak convergence property.

To this end, we first consider an arbitrary pair $(s, y_s)$ in the recurrent class of $\{(S_t, y_t)\}$ (cf.~Condition~\ref{cond-mem}). Let us show that for all initial conditions $x  \in \{(s, y_s, \e) \mid \e \in \rn \}$, $\{\mu_{x,t}\}$ converges weakly to the same invariant probability measure, almost surely.

Since the finite-state Markov chain $\{(S_t, y_t)\}$ has a single recurrent class (Condition~\ref{cond-mem}) and its evolution is not affected by $\{\e_t\}$, the marginal of $\zeta$ on $\S \times \M$ coincides with the unique invariant probability distribution of $\{(S_t, y_t)\}$. So the fact that $\zeta(D)=1$ and $(s, y_s)$ is a recurrent state of $\{(S_t, y_t)\}$ implies that there exists some $\hat \e$ with $(s, y_s, \hat \e) \in D$. For the initial condition $\hat x =(s, y_s, \hat \e)$, by the result of \citep{Meyn89} mentioned earlier, $\{\mu_{\hat x, t}\}$ converges weakly to $\mu_{\hat x}$, almost surely. 

Now consider $x = (s, y_s, \e)$ for an arbitrary $\e \in \rn$. Generate iterates $\{\hat \e_t\}$ and $\{\e_t\}$ according to (\ref{eq-td1}), using the same trajectory $\{(S_t, y_t)\}$ with $(S_0, y_0) = (s, y_s)$, but with $\hat \e_0 = \hat \e$ and $\e_0 = \e$. 
By Lemma~\ref{lem-1}, $\hat \e_t - \e_t \asto 0$. 
Therefore, except on a null set of sample paths, it holds for all bounded Lipschitz continuous functions $f$ on $\S \times \M \times \rn$ that
\footnote{Here we are using the same $(S_k, y_k), k \leq t$ in the occupation probability measures $\mu_{\hat x, t}$ and $\mu_{x, t}$. This is valid because the $\e_t$'s do not affect the evolution of $\{(S_t, y_t)\}$ and are functions of these states and the given initial $\e_0$. If we call the $\mu_{x, t}$ here $\tilde \mu_{x,t}$ instead and define $\mu_{x,t}$ using another independent copy of $\{(S_t, y_t)\}$, then since the two sequences of occupation probability measures will have the same probability distribution, $\{\mu_{x,t}\}$ will have the same weak convergence property as $\{\tilde \mu_{x,t}\}$.
}
\begin{equation} \label{eq-prf1}
\textstyle{ \left| \int f d \mu_{\hat x, t} - \int f d \mu_{x, t} \right| = \left| \frac{1}{t} \sum_{k=0}^{t-1} f(S_k, y_k, \hat \e_k) - \frac{1}{t} \sum_{k=0}^{t-1} f(S_k, y_k, \e_k) \right| \to 0.}
\end{equation} 
By the a.s.\ weak convergence of $\mu_{\hat x, t}$ to $\mu_{\hat x}$ proved earlier, except on a null set, $\int f d \mu_{\hat x, t} \to \int f d \mu_{\hat x}$ for all such functions $f$. Combining this with (\ref{eq-prf1}) yields that almost surely, $\int f d \mu_{x, t}  \to \int f d \mu_{\hat x}$ for all such $f$.
By \cite[Theorem 11.3.3]{Dud02}, this implies that almost surely, $\mu_{x,t} \to \mu_{\hat x}$ weakly. 

Thus we have proved that for all initial conditions $x =(s, y_s, \e), \e \in \rn$, $\{\mu_{x,t}\}$ converges weakly, almost surely, to the same invariant probability measure $\mu_{\hat x}$. Denote $\mu = \mu_{\hat x}$. Let us now show that for any initial condition $x$, $\{\mu_{x,t}\}$ also converges to $\mu$, $\Pr_x$-a.s. 

Consider $\{(S_t, y_t, \e_t)\}$ with an arbitrary initial condition $\bar x = (\bar s, \bar y, \bar \e)$.
Let $\tau = \min \{ t \mid (S_t, y_t) = (s,  y_s) \}$ (the pair $(s,y_s)$ is as in the proof above). 
Note that $\tau < \infty$ a.s., because $(s,y_s)$ is a recurrent state of $\{(S_t, y_t)\}$.
Define $(\tilde S_k, \tilde y_k) =  (S_{\tau+k}, y_{\tau+k})$, $\tilde \e_{k} = \e_{\tau+k}$ for $k \geq 0$.

By the strong Markov property (see e.g.~\citealp[Theorem 3.3]{Num84}), $\{(\tilde S_{k}, \tilde y_{k})\}_{k \geq 0}$ has the same probability distribution as the Markov chain $\{(S_t,y_t)\}$ that starts from $(S_0, y_0) = (s, y_s)$.
Therefore, by the preceding proof, $\Pr_{\bar x}$-almost surely, for all bounded continuous functions $f$ on $\S \times \M \times \rn$,
\begin{equation} \label{eq-prf2}
\lim_{m \to \infty} \textstyle{ \frac{1}{m} \sum_{k=0}^{m-1} f(\tilde S_{k}, \tilde y_{k}, \tilde \e_k)} = \textstyle{ \int f d \mu } .
\end{equation}
 Denote $a \wedge b = \min \{a, b\}$. Using (\ref{eq-prf2}) and the fact $\tau < \infty$ a.s., we have that $\Pr_{\bar x}$-almost surely, 
 \begin{align*}
\lim_{t \to \infty} \textstyle{ \frac{1}{t} \sum_{k=0}^{t-1} f(S_k, y_k, \e_k)} & = \lim_{t \to \infty}  \left( \textstyle{  \frac{1}{t} \sum_{k=0}^{t \wedge (\tau-1)} f(S_k, y_k, \e_k)} + \textstyle{ \frac{1}{t} \sum_{k=\tau}^{t-1} f(S_{k}, y_{k}, \e_k)} \right)   \\
                       & = \lim_{t \to \infty} \textstyle{ \frac{1}{t} \sum_{k=0}^{t-\tau-1} f(S_{\tau+k}, y_{\tau+k}, \e_{\tau +k})} \\
                       & = \lim_{m \to \infty} \textstyle{ \frac{1}{m} \sum_{k=0}^{m-1} f(\tilde S_{k}, \tilde y_{k}, \tilde \e_k)} =  \textstyle{\int f d \mu }.
\end{align*}                       
This proves that $\{\mu_{x,t}\}$ converges weakly to $\mu$ almost surely, for each initial condition $x$.

It now follows that $\mu$ must be the unique invariant probability measure of $\{(S_t, y_t, \e_t)\}$.
To see this, suppose $\zeta$ is another invariant probability measure. 
For any bounded continuous function $f$, by stationarity, 
$\E_{\zeta} [ \tfrac{1}{t} \sum_{k=0}^{t-1} f(S_k, y_k, \e_k) ] = \int f d\zeta$ for all $t \geq 1$.
On the other hand, 
the preceding proof has established that for all initial conditions $x$,
$$ \textstyle{\tfrac{1}{t} \sum_{k=0}^{t-1} f(S_k, y_k, \e_k)} = \int f d \mu_{x,t} \to \int f d \mu, \quad \text{$\Pr_x$-a.s.},$$ 
which implies that if $\zeta$ is the initial distribution of $(S_0, y_0,\e_0)$, then $\tfrac{1}{t} \sum_{k=0}^{t-1} f(S_k, y_k, \e_k) \to \int f d \mu$, $\Pr_\zeta$-a.s. We thus have 
\begin{align*}
 \textstyle{ \int f d\zeta = \E_{\zeta} \Big[ \frac{1}{t} \sum_{k=0}^{t-1} f(S_k, y_k, \e_k) \Big]} & = \lim_{t \to \infty} \textstyle{  \E_{\zeta} \Big[ \frac{1}{t} \sum_{k=0}^{t-1} f(S_k, y_k, \e_k) \Big]} \\
 &  = \E_{\zeta} \Big[ \lim_{t \to \infty}   \textstyle{ \frac{1}{t} \sum_{k=0}^{t-1} f(S_k, y_k, \e_k) \Big]}  = \textstyle{ \int f d \mu,}
\end{align*} 
where the third equality follows from the bounded convergence theorem.
This shows $\int f d\zeta = \int f d \mu$ for all bounded continuous functions $f$, and hence $\zeta = \mu$ by \cite[Prop.\ 11.3.2]{Dud02}, proving the uniqueness of the invariant probability measure.

The conclusions for the Markov chain $\{(S_t, S_{t+1}, y_t, \e_t)\}$ follow from the same arguments given above, if we replace $S_t$ with $(S_t, S_{t+1})$ and replace the set $\S$ with the set of possible state transitions. (We could have proved the assertions for $\{(S_t, S_{t+1}, y_t, \e_t)\}$ first and then deduced as their implications the assertions for $\{(S_t, y_t, \e_t)\}$. We treated the latter first, as it makes the notation in the proof simpler.)
\end{proof}

\begin{rem}[About the proof] \label{rem-prf-thm1} \rm
Theorem~\ref{thm-1} is similar to \cite[Theorem 3.2]{Yu-siam-lstd} for off-policy LSTD with constant $\lambda$ (the analysis given in \citep{Yu-siam-lstd} also applies to state-dependent $\lambda$). Some of the techniques used to prove the two theorems are also similar. The main difference to \citep{Yu-siam-lstd} is that in the proof here we used an argument based on the strong Markov property to extend the weak convergence property of $\{\mu_{x,t}\}$ for a subset of initial conditions $x \in \{(s, y_s, \e) \mid \e \in \rn\}$ to all initial conditions, whereas in \citep{Yu-siam-lstd} this step was proved using a result on the convergence-in-mean of LSTD iterates established first. The latter approach would not work here due to the dependence of $\lambda_t$ on the history. Indeed, due to this dependence, the proof of the convergence-in-mean of LSTD given in \citep{Yu-siam-lstd} does not carry over to our case, even though that convergence does hold as a consequence of Theorem~\ref{thm-1}, in view of the boundedness of traces by construction. Compared with the proof of the ergodicity result in \citep{Yu-siam-lstd}, the proof we gave here is more direct and therefore better.

Regarding possible alternative proofs of Theorem~\ref{thm-1}, let us also mention that if we prove first the uniqueness of the invariant probability measure, then, since $\{(S_t, y_t, \e_t)\}_{t \geq 1}$ lie in a bounded set, the weak convergence of occupation probability measures will follow immediately from \cite[Prop.~4.2]{Meyn89}. However, because the evolution of the $\lambda_t$'s depends on both states and traces, it does not seem easy to us to prove directly the uniqueness part first. 
\qed \end{rem}

\begin{rem}[About the conditions on the function $\lambda(\cdot)$] \label{rem-conditions-thm1} \rm
Our proof of Theorem~\ref{thm-1} relied on Lemma~\ref{lem-1} and the two properties discussed preceding that lemma, namely, that $\{(S_t, y_t, \e_t)\}$ is a weak Feller Markov chain and has at least one invariant probability measure. As long as these hold when we weaken or change the conditions on the function $\lambda(\cdot)$, the proof and the conclusions of the theorem will remain applicable.

We introduced Condition~\ref{cond-lambda}(ii) to bound the traces for algorithmic concerns. For the ergodicity of the state-trace process, Condition~\ref{cond-lambda}(ii) is unimportant---in fact, it can be removed from the conditions of Theorem~\ref{thm-1}. The reason is that we used this condition before Lemma~\ref{lem-1} to quickly infer that $\{(S_t, y_t, \e_t)\}$ has at least one invariant probability measure, but this is still true without Condition~\ref{cond-lambda}(ii), in view of \cite[Theorem 12.1.2(ii)]{MeT09} and the fact that under Condition~\ref{cond-pol}(i), $\{\e_t\}$ is bounded in probability (the proof of this fact is straightforward and similar to the proof of \cite[Lemma 3.1]{Yu-siam-lstd} or \cite[Prop.~A.1]{yu-etdarx}).

Condition~\ref{cond-lambda}(i) is actually two conditions combined into one. The first is the continuity of $\lambda(y, \e) \e$ in $\e$ for each $y$, which was used to ensure that the state-trace process is a weak Feller Markov chain. To be more general, instead of letting the evolutions of the traces and memory states be governed by the functions $\lambda$ and $g$, one may consider letting them be governed by stochastic kernels. Then by placing a suitable continuity condition on the stochastic kernel $\lambda$, one can ensure that the state-trace process has the desired weak Feller Markov property.

The second condition packed into Condition~\ref{cond-lambda}(i) is that for each $y$, $\lambda(y, \e) \e$ is a Lipschitz continuous function of $\e$ with modulus $1$. This condition is somewhat restrictive, and one may consider instead allowing the function to have Lipschitz modulus greater than $1$. However, additional conditions are then needed to ensure that Lemma~\ref{lem-1} holds. (If this lemma does not hold, then the state-trace process may not be ergodic and one will need a different approach than the one we took to characterize the sample path properties of the state-trace process.)

From an algorithmic perspective, if it is desirable to choose even larger $\lambda_t$'s or to have greater flexibility in choosing these $\lambda$-parameters, some of the generalizations just mentioned can be considered. For example, Condition~\ref{cond-lambda}(ii) can be replaced and stochastic kernels can be introduced to allow for occasionally large traces $\e_t$, so that instead of having the traces bounded, one only make their variances bounded in a desired range.
\qed \end{rem}

\section{Generalized Bellman Equations} \label{sec-3}

In this section, we continue the analysis started in Section~\ref{sec-2.3}. 
Recall that Corollary~\ref{cor-1} established that the asymptotic limit of the linear equations (\ref{eq-lstd}) for LSTD is the linear equation (in $v$): 
$$\E_\zeta [ \, \e_0 \, \delta_0(v) ] = 0.$$
Our goal now is to relate this equation
to a generalized Bellman equation for the target policy $\pi$. This will then allow us to interpret solutions of (\ref{eq-lstd}) computed by LSTD as solutions of approximate versions of that generalized Bellman equation. 

To this end, we will first give a general description of randomized stopping times and associated Bellman operators (Section~\ref{sec-3.1}). We will then use these notions to derive the particular Bellman operators that correspond to our choices of the $\lambda$-parameters and appear in the linear equations for LSTD (Section~\ref{sec-3.2}). We will also discuss a composite scheme of choosing the $\lambda$-parameters as a direct application and extension of our results.

To simplify notation in subsequent derivations, we shall use the following shorthand notation: For $k \leq m$, denote $S_k^m = (S_k, S_{k+1}, \ldots S_m)$, 
\begin{equation}
\textstyle{\rho_k^m = \prod_{i=k}^m  \rho_i, \qquad \lambda_k^m = \prod_{i=k}^m  \lambda_i, \qquad \gamma_k^m = \prod_{i=k}^m  \gamma_i}. \label{eq-prod}
\end{equation}
Also, we shall treat $\rho_k^m = \lambda_k^m = \gamma_k^m = 1$ if $k > m$.

\subsection{Randomized Stopping Times and Associated Bellman Operators} \label{sec-3.1}

Consider the Markov chain $\{S_t\}$ induced by the target policy $\pi$.  Let Condition~\ref{cond-pol}(i) hold. Recall that for the value function $v_\pi$, we have that for each state $s \in \S$,
$$v_{\pi}(s) = \textstyle{\E^\pi_s \big[ \sum_{t=0}^\infty \gamma_1^t \, r_\pi(S_t) \big]} \ \ \text{(by definition)}$$
and 
$$ v_\pi(s) = r_\pi(s) + \E_s^\pi [ \gamma_1 v_\pi(S_1)]. $$ 
The second equation is the standard one-step Bellman equation.

To write generalized Bellman equations for $\pi$, we shall make use of \emph{randomized stopping times} for $\{S_t\}$, a notion that generalizes naturally stopping times for $\{S_t\}$ in that whether to stop at time $t$ depends not only on the past states $S_0^t$ but also on certain random outcomes. A simple example is to toss a coin at each time and stop as soon as the coin lands on heads, regardless of the history $S_0^t$. (The corresponding Bellman equation is the one associated with TD($\lambda$) for a constant~$\lambda$; cf.~Example~\ref{ex-td-tau}.) Of interest here is the general case where the stopping decision does depend on the entire history. 

To define a randomized stopping time formally, first, the probability space of $\{S_t\}$ is enlarged to take into account whatever randomization scheme that is used to make the stopping decision. (The enlargement will be problem-dependent, as the next subsection will demonstrate.) 
Then, on the enlarged space, a randomized stopping time $\tau$ for $\{S_t\}$ is a stopping time
\footnote{A random time $\tau$ is called a stopping time relative to a sequence $\{\F_t\}$ of increasing $\sigma$-algebras if the event $\{\tau = t\} \in \F_t$ for every $t$.}
relative to some increasing sequence of $\sigma$-algebras $\mathcal{F}_0 \subset \mathcal{F}_1 \subset \cdots$, where the sequence $\{\F_t\}$ is such that\vspace*{-0.15cm} 
\begin{itemize}
\item[(i)] for all $t \geq 0$, $\mathcal{F}_t \supset \sigma(S_0^t)$ (the $\sigma$-algebra generated by $S_0^t$), and 
\item[(ii)] relative to $\{\mathcal{F}_t\}$, $\{S_t\}$ remains to be a Markov chain with transition probability $\P$, i.e., for all $s \in \S$, $\text{Prob}(S_{t+1} =s \mid \mathcal{F}_t) = P_{S_t s}$.\vspace*{-0.15cm}
\end{itemize}
See \cite[Chap.\ 3.3]{Num84}; in particular, see Prop.\ 3.6 in p.\ 31-32 therein for several equivalent definitions of randomized stopping times.

Note that if $\mathcal{F}_t = \sigma(S_0^t)$ for all $t$, then the history of states $S_0^t$ fully determines whether $\tau \leq t$ and $\tau$ reduces to a stopping time for the Markov chain $\{S_t\}$. The properties (i)-(ii) in the above definition encapsulate our earlier intuitive discussion about making stopping decisions, namely, stopping decisions are made based on the history $S_0^t$ and additional random outcomes that do not affect the evolution of the Markov chain. 

Like stopping times, the strong Markov property also holds for randomized stopping times for a Markov chain. This is an important basic property. It says that in the event $\tau < \infty$, conditioned on the $\sigma$-algebra $\mathcal{F}_\tau$  associated with the stopping time $\tau$ relative to $\{\mathcal{F}_t\}$ (which is the $\sigma$-algebra generated by the events that ``happen before $\tau$''),  the conditional distribution of $(S_\tau, S_{\tau+1}, \ldots)$ is the same as the probability distribution of a Markov chain $(S_0, S_1, \ldots)$ with initial state $S_0 = S_\tau$~\cite[Theorem 3.3]{Num84}.

The above abstract definition of a randomized stopping time allows us to write Bellman equations in general forms without worrying about the details of the enlarged space, which are not important at this point. For notational simplicity, when there is no confusion, we shall still write $\Pr^\pi$ for the probability measure on the enlarged probability space and use $\E^\pi$ and $\E^\pi_s$ to denote the expectation and conditional expectation given $S_0 = s$, respectively, for that space.

If $\tau$ is a randomized stopping time for $\{S_t\}$, the strong Markov property \cite[Theorem 3.3]{Num84} allows us to express $v_\pi$ in terms of $v_\pi(S_\tau)$ and the total discounted rewards $R^\tau$ prior to stopping:
\begin{align}
   v_\pi(s) & = \textstyle{ \E^\pi_s \left[ \sum_{t=0}^{\tau-1} \gamma_1^t \, r_\pi(S_t) + \sum_{t=\tau}^\infty \gamma_1^{\tau} \cdot \gamma_{\tau+1}^t \, r_\pi(S_t) \right] }  \notag \\    
              & = \E^\pi_s \big[  R^{\tau} + \gamma_1^{\tau} \, v_{\pi}(S_{\tau})  \big],  \label{eq-gbe1}
\end{align}
where $R^\tau = \sum_{t=0}^{\tau-1} \gamma_1^t \, r_\pi(S_t)$ for $\tau \in \{0, 1, 2, \ldots \} \cup \{+\infty\}$.
\footnote{In the case $\tau = 0$, $R^0 = 0$. In the case $\tau = \infty$, by Condition~\ref{cond-pol}(i), $R^\infty = \sum_{t=0}^{\infty} \gamma_1^t \, r_\pi(S_t)$ is almost surely well-defined, while the second term $\gamma_1^{\tau} \, v_\pi(S_{\tau})$ in (\ref{eq-gbe1}) is $0$ because $\gamma_1^\infty := \prod_{k=1}^\infty \gamma_k = 0$ a.s., under Condition~\ref{cond-pol}(i).
Equation (\ref{eq-gbe1}) is derived as follows: By the strong Markov property~\cite[Theorem 3.3]{Num84}, on $\{\tau < \infty\}$,
$$\textstyle{\E^\pi \left[ \sum_{t=\tau}^\infty \gamma_1^{\tau} \cdot \gamma_{\tau+1}^t \, r_\pi(S_t) \mid \mathcal{F}_\tau \right] = \gamma_1^{\tau} \cdot \E^\pi_{S_\tau} \left[ \sum_{t=0}^\infty  \gamma_{1}^t \, r_\pi(S_t) \right] = \gamma_1^{\tau} v_{\pi}(S_{\tau}).}$$
Then, since the term $\E^\pi_s \!\left[ \,\sum_{t=\tau}^\infty \gamma_1^{\tau} \cdot \gamma_{\tau+1}^t \, r_\pi(S_t) \right] = \E^\pi_s \!\left[ \, \I(\tau < \infty) \cdot \sum_{t=\tau}^\infty \gamma_1^{\tau} \cdot \gamma_{\tau+1}^t \, r_\pi(S_t) \right]$, we use the property of the conditional expectation given $\mathcal{F}_\tau$ and the fact $\mathcal{F}_\tau \supset \sigma(S_0)$ to rewrite this term as 
$$\textstyle{ \E^\pi_s \!\big[ \, \I(\tau < \infty)  \cdot \E^\pi\! \left[ \,\sum_{t=\tau}^\infty \gamma_1^{\tau} \cdot \gamma_{\tau+1}^t \, r_\pi(S_t) \mid \mathcal{F}_\tau \right] \big] = \E^\pi_s \!\left[ \,\I(\tau < \infty) \cdot \gamma_1^{\tau} v_{\pi}(S_{\tau}) \right] = \E^\pi_s \left[ \gamma_1^{\tau} v_{\pi}(S_{\tau}) \right]},$$
where in the last equality we also used the fact $\gamma_1^\infty = 0$ a.s. This gives (\ref{eq-gbe1}).
}
We can also write the Bellman equation (\ref{eq-gbe1}) in terms of $\{S_t\}$ only, by taking expectation over $\tau$:
\begin{align}
 v_\pi(s) & = \textstyle{ \E^\pi_s \left[ \sum_{t=0}^\infty \Big( \I(\tau > t) \cdot \gamma_1^t \, r_\pi(S_t) + \I(\tau = t) \cdot \gamma_1^t \, v_\pi(S_t) \Big)  \right],} \notag \\
  & = \textstyle{ \E^\pi_s \left[ \sum_{t=0}^\infty \Big( q^+_{t}(S_0^t) \cdot \gamma_1^t \, r_\pi(S_t) + q_{t}(S_0^t) \cdot \gamma_1^t \, v_\pi(S_t) \Big)  \right],}   \label{eq-gbe2} 
\end{align} 
where 
\begin{equation}
 q^+_{t}(S_0^t) = \Pr^\pi( \tau > t \mid S_0^t),   \qquad q_{t}(S_0^t) = \Pr^\pi( \tau = t \mid S_0^t). 
\end{equation}
The r.h.s.\ of (\ref{eq-gbe1}) or (\ref{eq-gbe2}) defines a generalized Bellman operator $T : \Re^{N} \to \Re^{N}$ associated with $\tau$, which has several equivalent expressions; e.g.,
\begin{align*}
 (Tv)(s) \! = \!\textstyle{ \E^\pi_s \big[  R^{\tau} + \gamma_1^{\tau} \, v(S_{\tau}) \big]} 
 \! =  \! \textstyle{ \E^\pi_s \!\left[ \sum_{t=0}^\infty \!\Big( q^+_{t}(S_0^t) \cdot \gamma_1^t \, r_\pi(S_t) + q_{t}(S_0^t) \cdot \gamma_1^t \, v(S_t) \Big)\!  \right]\!,} \quad s \in \S.
\end{align*}
Depending on the context, one expression can be more convenient to use than the other. For example, the first expression is convenient for defining $T$ through the associated $\tau$ and for deducing the contraction property of $T$, whereas expressions like the second will be of interest when we want to know more explicitly the particular $T$ for our TD learning scheme and its dependence on the $\lambda$-parameters. 

In common with one-step Bellman operator, the generalized Bellman operator $T$ is affine and involves a substochastic matrix.
If $\tau \geq 1$ a.s., then the value function $v_\pi$ is the unique fixed point of $T$, i.e., the unique solution of $v = T v$, and $T$ is a sup-norm contraction. In fact, this can be shown for slightly more general $\tau$:

\begin{thm} \label{thm-gbo}
Let Condition~\ref{cond-pol}(i) hold, and let the randomized stopping time $\tau$ be such that $\Pr^\pi(\tau \geq 1 \mid S_0 = s) > 0$ for all states $s \in \S$. Then $v_\pi$ is the unique fixed point of the generalized Bellman operator $T$ associated with $\tau$, and $T$ is a contraction w.r.t.\ a weighted sup-norm on $\Re^{N}$.
\end{thm}

We prove this theorem in Appendix~\ref{appsec-1}. The proof amounts to showing that if a state process evolves according to the substochastic matrix $\tilde P$ involved in the affine operator $T$, then all the states in $\S$ are transient (equivalently, the spectral radius of $\tilde P$ is less than $1$ and $I - \tilde P$ is invertible \cite[Appendix A.4]{puterman94}). From this the conclusions of the theorem follow as a basic fact from nonnegative matrix theory \cite[Theorem 1.1]{Sen06}, and one specific choice of the weights of the sup-norm in the theorem is simply the expected time for the process to leave $\S$ from each initial state (see e.g.,~the proof of \cite[Prop. 2.2]{BET}).

For TD algorithms that do not use history-dependent $\lambda$, the random times $\tau$ and the corresponding Bellman operators $T$ have simple descriptions:

\begin{myexample}[TD with constant or state-dependent $\lambda$]
\label{ex-td-tau} \rm 
Depending on the choice of $\lambda$, TD($\lambda$) algorithms are associated with different randomized stopping times $\tau$. 
In the case of constant $\lambda$, starting from time $1$, we stop the system with probability $1 - \lambda$ if it has not stopped yet; i.e., 
$$\tau \geq 1 \ \ \ \text{and} \ \ \ \Pr^\pi(\tau = t \mid \tau > t-1, S_0^t) = 1 - \lambda, \ \ \  \forall \, t \geq 1.$$ 
In particular, we always stop at $t=1$ if $\lambda =0$, and we never stop if $\lambda = 1$. 
Similarly, for state-dependent $\lambda$ where $\lambda_t = \lambda(S_t)$, a function of the current state, the preceding stopping probability is replaced by $1 - \lambda(S_t)$: $\Pr^\pi(\tau = t \mid \tau > t-1, S_0^t) = 1 - \lambda(S_t)$ for $t \geq 1$.
In these cases, by taking expectations over $\tau$, the corresponding Bellman operators can be expressed solely in terms of $\lambda$ and the model parameters for the target policy. \qed
\end{myexample}

\subsection{Bellman Equation for the Proposed TD Learning Scheme} \label{sec-3.2}

With the terminology of randomized stopping times, we are now ready to write down the generalized Bellman equation associated with the TD learning scheme proposed in Section~\ref{sec-2.2}. It corresponds to a particular randomized stopping time. We shall first describe this random time, from which a generalized Bellman equation follows as seen in the preceding subsection. That this is indeed the Bellman equation for our TD learning scheme will then be proved. 

Consider the Markov chain $\{S_t\}$ \emph{under the target policy} $\pi$. We define a randomized stopping time $\tau$ for $\{S_t\}$:
\begin{itemize}
\item[$\bullet$] Let $y_t,  \lambda_t, \e_t, t \geq 1,$ evolve according to (\ref{eq-td3}) and (\ref{eq-td1}):
$$ y_t = g(y_{t-1}, S_t), \qquad \lambda_t=\lambda(y_t, \e_{t-1}), \qquad \e_t  =  \lambda_t \, \gamma_t \,  \rho_{t-1} \, \e_{t-1} +  \fe(S_t), \quad  t \geq 1.$$
\item[$\bullet$] Let the initial $(S_0, y_0, \e_0)$ be distributed according to $\zeta$, the unique invariant probability measure in Theorem~\ref{thm-1} for the state-trace process induced by the behavior policy.
\item[$\bullet$] At time $t \geq 1$, we stop the system with probability $1 - \lambda_t$ if it has not yet been stopped. Let $\tau$ be the time when the system stops ($\tau = \infty$ if the system never stops).
\end{itemize}
To make the dependence on the initial distribution $\zeta$ explicit, we write $\Pr^\pi_\zeta$ for the probability measure of this process.

Note that by definition \emph{$\lambda_t$ and $\lambda_1^t = \prod_{k=1}^t \lambda_k$ are functions of the initial $(y_0, \e_0)$ and states $S_0^t$}.
From how the random time $\tau$ is defined, we have for all $t \geq 1$,
\begin{align}
 \Pr^\pi_\zeta( \tau > t \mid S_0^t, y_0, \e_0) & = \lambda_1^t  = : h^+_{t}(y_0, \e_0, S_0^t),  \label{eq-ptau1} \\
 \Pr^\pi_\zeta( \tau = t \mid S_0^t, y_0, \e_0) & =   \lambda_1^{t-1} (1 - \lambda_t) = : h_t(y_0, \e_0, S_0^t), \label{eq-ptau2}
\end{align} 
and hence
\begin{align}
   q^+_{t}(S_0^t) & : = \Pr^\pi_\zeta( \tau > t \mid S_0^t)  = \int h^+_{t}(y, \e, S_0^t)  \, \zeta\big(d(y, \e) \mid S_0 \big),  \label{eq-Ta} \\ 
   q_{t}(S_0^t)  & : = \Pr^\pi_\zeta( \tau = t \mid S_0^t) = \int h_t(y, \e, S_0^t)  \, \zeta\big(d(y, \e) \mid S_0 \big), \label{eq-Tb}
\end{align}
where $\zeta(d(y,e) \mid s)$ is the conditional distribution of $(y_0, \e_0)$ given $S_0 = s$, w.r.t.\ the initial distribution $\zeta$.
As before, we can write the generalized Bellman operator $T$ associated with $\tau$ in several equivalent forms. 
Let $\E^\pi_\zeta$ denote expectation under $\Pr^\pi_\zeta$.
Similarly to the derivation of (\ref{eq-gbe2}), we can rewrite (\ref{eq-gbe1}) in this case by taking expectation over $\tau$ conditioned on $(S_0^t, y_0, \e_0)$ to derive that 
for all $v: \S \to \Re, s \in \S$,
\begin{equation} \label{eq-T0}
(Tv)(s) =  \textstyle{ \E^\pi_\zeta \Big[ \sum_{t=0}^\infty  \lambda_1^t \gamma_1^t \, r_\pi(S_t) + \sum_{t=1}^\infty \lambda_1^{t-1} (1 - \lambda_t)  \gamma_1^{t} \, v(S_{t})  \mid S_0 = s\Big]}.
\end{equation}
Or express $T$ in the form of (\ref{eq-gbe2}) by further integrating over $(y_0, \e_0)$ and using (\ref{eq-Ta})-(\ref{eq-Tb}): 
\begin{equation}
 (Tv)(s) =  \textstyle{ \E^\pi_\zeta \left[ \sum_{t=0}^\infty \Big( q^+_{t}(S_0^t) \cdot \gamma_1^t \, r_\pi(S_t) + q_{t}(S_0^{t}) \cdot \gamma_1^{t} \, v(S_{t}) \Big) \, \big| \, S_0 = s \right]},  \label{eq-T}
\end{equation} 
for all $v: \S \to \Re, s \in \S$, where in the case $t=0$, $q^+_0(S_0) = 1$ and $q_0(S_0)=0$ since $\tau > 0$ by construction.

It will be useful later to express $TV -V$ in terms of temporal differences. From (\ref{eq-T0}), by writing $\lambda_1^{t-1} (1 - \lambda_t)  \gamma_1^{t} \, v(S_{t}) =  \lambda_1^{t-1} \gamma_1^{t} \, v(S_{t}) -  \lambda_1^{t} \gamma_1^{t} \, v(S_{t})$ and rearranging terms, we have for all $v: \S \to \Re, s \in \S$,
\begin{align}
(Tv)(s) - v(s) & =  \textstyle{ \E^\pi_\zeta \Big[ \sum_{t=0}^\infty  \lambda_1^t \gamma_1^t \, r_\pi(S_t) + \sum_{t=0}^\infty \lambda_1^{t}  \gamma_1^{t+1} \, v(S_{t+1})  - \sum_{t=0}^\infty \lambda_1^{t} \gamma_1^{t} \, v(S_{t}) \mid S_0 = s\Big]} \notag \\
& = \textstyle{ \E^\pi_\zeta \left[ \sum_{t=0}^\infty  \lambda_1^t \gamma_1^t \cdot \Big( r_\pi(S_t) + \gamma_{t+1} \, v(S_{t+1}) - v(S_{t}) \Big) \, \big|\, S_0 = s\right]} \label{eq-Tdiff}.
\end{align}
In a similar way, from (\ref{eq-T}), we can write
\footnote{Since $\tau$ is a randomized stopping time for the Markov chain $\{S_t\}$, we have $\Pr^\pi_\zeta( \tau > t \mid S_0^{t+1}) = \Pr^\pi_\zeta( \tau > t \mid S_0^t)$, so $\Pr^\pi_\zeta( \tau > t \mid S_0^t) - \Pr^\pi_\zeta( \tau = t +1 \mid S_0^{t+1}) =  \Pr^\pi_\zeta( \tau > t+1 \mid S_0^{t+1})$, i.e., $q^+_{t}(S_0^t) - q_{t+1}(S_0^{t+1}) = q^+_{t+1}(S_0^{t+1})$. 
Thus we can write the term $q_{t}(S_0^{t})$ in (\ref{eq-T}) for $t \geq 1$ as $q^+_{t-1}(S_0^{t-1}) - q^+_{t}(S_0^t)$, and the expression for $(Tv -v)(s)$ then follows by rearranging terms.}
$$ (Tv)(s) - v(s) = \textstyle{ \E^\pi_\zeta \left[ \sum_{t=0}^\infty q^+_{t}(S_0^t) \cdot \gamma_1^t \cdot \Big( r_\pi(S_t) + \gamma_{t+1} \, v(S_{t+1}) - v(S_{t}) \Big)  \, \big|\, S_0 = s\right]}.$$

\begin{rem} \label{rem-expT} \rm
Comparing the two expressions (\ref{eq-T0}) and (\ref{eq-T}) of $T$, we remark that the expression (\ref{eq-T0}) reflects the role of the $\lambda_t$'s in determining the stopping time, 
whereas the expression (\ref{eq-T}), which has eliminated the auxiliary variables $y_t$ and $\e_t$, shows more clearly the dependence of the stopping time on the entire history $S_0^t$. It can also be seen, from the initial distribution $\zeta$, the dependence of $\lambda_t$ on the traces and the dependence of the traces on the function $\rho(\cdot)$ (which describes importance sampling ratios), that both the behavior policy and the choice of the feature representation assert a significant role in determining the Bellman operator $T$ for the target policy. This is in contrast with off-policy TD learning that uses a constant $\lambda$, where the behavior policy and the approximation subspace affect only how one approximates the Bellman equation underlying TD, not the Bellman equation itself, which is solely determined by $\lambda$ (cf.\ Example~\ref{ex-td-tau}). 

Furthermore, note that as the invariant distribution of the state-trace process, $\zeta$ is associated with the dynamic behavior of the states \emph{and traces} under the behavior policy. Generally, there is no explicit expression of $\zeta$ in terms of $\P^o$ and the parameters in the $\lambda$ function. As a result, in general we cannot express the operator $T$ in terms of these parameters in the learning scheme. This is different from the case of TD($\lambda$) where $\lambda$ is a function of the present state only.
\qed \end{rem}

We now proceed to show how the Bellman equation $v=T v$ given above relates to the off-policy TD learning scheme in Section~\ref{sec-2.2}.
Some notation is needed. Denote by $\zeta_\S$ the invariant probability measure of the Markov chain $\{S_t\}$ \emph{induced by the behavior policy}; note that it coincides with the marginal of $\zeta$ on $\S$.
For two functions $v_1, v_2$ on $\S$, we write $v_1 \perp_{\zeta_\S} v_2$ if $\sum_{s \in \S} \zeta_\S(s) \, v_1(s) \, v_2(s) = 0$. 
If $\L$ is a linear subspace of functions on $\S$ and $v \perp_{\zeta_\S} v'$ for all $v' \in \L$, we write $v \perp_{\zeta_\S} \L$. 
Recall that $\phi$ is a function that maps each state $s$ to an $n$-dimensional feature vector. Denote by $\L_\phi$ the subspace spanned by the $n$ component functions of $\phi$, which is the space of approximate value functions for our TD learning scheme. Recall also that $\E_\zeta$ denotes expectation w.r.t.\ the \emph{stationary} state-trace process $\{(S_t, y_t, \e_t)\}$ under the behavior policy (cf.\ Theorem~\ref{thm-1}).

\begin{thm} \label{thm-2}
Let Conditions~\ref{cond-pol}-\ref{cond-lambda} hold. Then as a linear equation in $v$, $\E_\zeta \big[ \e_0 \, \delta_0(v) \big] = 0$ is equivalently 
$T v - v \perp_{\zeta_\S} \L_\phi,$ where $T$ is the generalized Bellman operator for $\pi$ given in (\ref{eq-T0}) or (\ref{eq-T}). 
\end{thm}

\begin{rem}[On LSTD] \label{rem-1} \rm
Note that 
$$T v - v \perp_{\zeta_\S} \L_\phi, \ \ \ v \in \L_\phi$$  
is a projected version of the generalized Bellman equation $T v - v =0$ (projecting the left-hand side onto the approximation subspace $\L_\phi$ w.r.t.\ the $\zeta_\S$-weighted Euclidean norm). Theorem~\ref{thm-2} and Corollary~\ref{cor-1} together show that this is what LSTD solves in the limit. 

Note also that although the generalized Bellman operator $T$ is a contraction (Theorem~\ref{thm-gbo}), the composition of projection with $T$ is in general not a contraction (cf.\ Example~\ref{counter-ex-contraction} in Appendix~\ref{appsec-oblproj}). Thus we cannot use contraction-based arguments to analyze approximation properties. For that purpose, we use the oblique projection viewpoint of \cite{bruno-oblproj}. Specifically, if the preceding projected Bellman equation admits a unique solution $\bar v$, then $\bar v$ can be viewed as an oblique projection of $v_\pi$ \citep{bruno-oblproj} and the approximation error $\bar v - v_\pi$ can be characterized as in \citep{yb-errbd} by using the oblique projection viewpoint. The details of these are given in Appendix~\ref{appsec-oblproj}. 
\qed \end{rem}

\begin{rem}[On gradient-based TD] \label{rmk-gtd} \rm
While Theorem~\ref{thm-2} is about the LSTD algorithm, it also helps in prepare the ground for analyzing gradient-based algorithms similar to those discussed in \citep{maei11,pmrl}. Like LSTD, these algorithms aim to solve the same projected generalized Bellman equation as characterized by Theorem~\ref{thm-2} (cf.\ Remark~\ref{rem-1}). Their average dynamics, which is important for analyzing their convergence using the mean ODE approach from stochastic approximation theory \citep{KuY03}, can be studied based on the ergodicity result of Theorem~\ref{thm-1}, in essentially the same way as we did in Section~\ref{sec-2.3} for the LSTD algorithm. For details of the convergence analysis of these gradient-based TD algorithms, see the recent work~\citep{gtd-conv17}.
\qed \end{rem}

In the rest of this subsection, we give a corollary to Theorem~\ref{thm-2}, deferring the proofs of both the theorem and the corollary to the next subsection.  The corollary concerns a composite scheme of setting $\lambda$, which is slightly more general than what Section~\ref{sec-2.2} described. It results in a Bellman operator that is a composition of the components of other Bellman operators, and it can be useful in practice for variance control. Let us describe the scheme first, before explaining our motivation for it.
 
Partition the state space into $m$ nonempty disjoint sets: $\S = \cup_{i=1}^m \S_i$. Associate each set $\S_i$ with a possibly different scheme of setting $\lambda$ that is of the type described in Section~\ref{sec-2.2}, and denote its memory states by $y^{(i)}_t$ and $\lambda$-function by $\lambda^{(i)}(\cdot, \cdot)$.
Keep $m$ trace vectors $\e^{(1)}_t, \ldots, \e^{(m)}_t$, one for each set, and update them according to
\begin{equation}
    \e^{(i)}_{t}  =  \lambda^{(i)}_t \, \gamma_t \,  \rho_{t-1} \, \e^{(i)}_{t-1} +  \fe(S_t) \, \I(S_t \in \S_i),  \qquad 1 \leq i \leq m, \label{eq-tdcmp1}  
\end{equation} 
where $\lambda^{(i)}_t = \lambda^{(i)}\big(y^{(i)}_t, \e^{(i)}_{t-1} \big)$. We then have $m$ ergodic state-trace processes that share the same state variables, $\big\{\big(S_t, y^{(i)}_t, \e^{(i)}_t \big) \big\}$, $i =1, 2, \ldots, m$. Each process has a unique invariant probability measure $\zeta^{(i)}$ (Theorem~\ref{thm-1}) and an associated randomized stopping time $\tau^{(i)}$ and generalized Bellman operator $T^{(i)}$, as discussed in this subsection.
Define now an operator $T$ by concatenating the component mappings of $T^{(i)}$ for $\S_i$ as follows: for all $v \in \Re^N$ and $s \in \S$,
\begin{equation} \label{eq-tdcmp-T}
    (T v)(s) : = (T^{(i)} v)(s) \ \ \ \text{if} \ s \in \S_i.
\end{equation}
Consider an LSTD algorithm that defines the trace $\e_t$ to be the sum of the $m$ trace vectors,
\begin{equation} 
   \e_t = \textstyle{ \sum_{i=1}^m \e^{(i)}_t,}  \label{eq-tdcmp2}
\end{equation}
and uses the traces to form the linear equation as before,
$$ \textstyle{ \tfrac{1}{t} \sum_{k=0}^{t-1}  \, \e_k \, \delta_k (v) = 0, \quad v = \Phi \theta.}$$
Note that $\tfrac{1}{t} \sum_{k=0}^{t-1}  \e_k  \delta_k (v) = 0$ is the same as
$ \sum_{i=1}^m \tfrac{1}{t} \sum_{k=0}^{t-1}   \e_k^{(i)}  \delta_k (v)  = 0.$
By Corollary~\ref{cor-1}, as a linear equation in $v$, it tends asymptotically (as $t \to \infty$) to the linear equation $\sum_{i=1}^m \E_{\zeta^{(i)}} \big[ \e_0^{(i)}  \delta_0(v) \big] = 0$.

\begin{cor} \label{cor-tdcmp}
Let Condition~\ref{cond-pol} hold. Consider the composite scheme of setting $\lambda$ discussed above, and let Conditions~\ref{cond-mem}-\ref{cond-lambda} hold for each of the $m$ schemes involved. Let LSTD calculate traces according to (\ref{eq-tdcmp1}) and (\ref{eq-tdcmp2}). Then the limiting linear equation (in $v$) associated with LSTD, $\sum_{i=1}^m \E_{\zeta^{(i)}} \big[ \e_0^{(i)}  \delta_0(v) \big] = 0$, is equivalently 
$T v - v \perp_{\zeta_\S} \L_\phi,$ where $T$ is the generalized Bellman operator for $\pi$ given by (\ref{eq-tdcmp-T}) and has the same fixed point and contraction properties as stated in Theorem~\ref{thm-gbo}.
\end{cor}

The use of composite schemes will be demonstrated by experiments in Section~\ref{sec-mcar}.
Here let us explain informally our motivation for such schemes.

\begin{rem}[About composite schemes of setting $\lambda$] \label{rmk-tdcmp} \rm
Our motivation for using the composite schemes is revealed by the equation (\ref{eq-tdcmp-T}). Typically each $T^{(i)}$ is designed to be simple to implement in TD learning.
For example, if we ignore for now the bounding of traces introduced in Section~\ref{sec-2.2} and just consider TD($\lambda$) with constant $\lambda$, $T^{(i)}$ can be the Bellman operator $T^{(\lambda)}$ for TD($\lambda$) with some constant $\lambda$. A simple, extreme example is to partition the state space into two sets, and associate one with $T^{(\lambda)}, \lambda =1$, and the other with $T^{(\lambda)}, \lambda =0$. 
Using the combination (\ref{eq-tdcmp-T}) of the two operators in TD then means that for the first set of states whose $\lambda = 1$, we want to estimate their values by using the information about the total rewards received when starting from those states, whereas for the second set of states whose $\lambda = 0$, we only use the information about their one-stage rewards and how these states relate to the ``neighboring'' states in the transition graph. While this way of using different kinds of information for different states is natural and useful for TD-based policy evaluation, it cannot be realized by keeping a single trace sequence as before and only letting $\lambda_t$ evolve with states or histories. Indeed, in that case, as discussed in Section~\ref{sec-compare-retrace}, interleaving large and small $\lambda_t$'s would make the algorithm behave effectively like TD with small $\lambda$ over the entire state space.

In the context of the more complex scheme of setting $\lambda$ discussed in this paper, our motivation and reasons for considering composite schemes are the same. Each $T^{(i)}$ can be designed to be simple to implement, such as in the simple scaling example in Section~\ref{sec-2.2}. The parameters in the $i$th scheme can be chosen so that they encourage the use of large $\lambda_t$'s throughout time or dictate the use of only small $\lambda_t$'s. By combining component mappings of $T^{(i)}$ through (\ref{eq-tdcmp-T}), composite schemes allow us to use cumulative rewards and transition structures at different timescales for different states. This provides additional flexibility in managing the bias-variance trade-off when estimating the value function (see Figure~\ref{fig-mcar3b} and Figure~\ref{fig-mcar4b} in Section~\ref{sec-mcar} for a demonstration).

Finally, we mention that for off-policy LSTD($\lambda$) with constant $\lambda$, composite schemes were proposed in \citep{yb-bellmaneq} and analyzed in \cite[Proposition~4.5, Section~4.3]{Yu-siam-lstd}. Our Corollary~\ref{cor-tdcmp} extends that result. The convergence analysis of the gradient-based algorithms for the composite schemes is given in \citep{gtd-conv17}.
\qed
\end{rem}

\subsection{Proofs of Theorem~\ref{thm-2} and Corollary~\ref{cor-tdcmp}}

We divide the proof of Theorem~\ref{thm-2} into two steps. The first step deals with an expression for the trace vector, given in the following lemma. It is more subtle than the other step in the proof, which involves mostly calculations.

We start by extending the stationary state-trace process $\{(S_t, y_t, \e_t)\}_{t \geq 0}$ to $t = -1$, $-2, \ldots$, and work with a double-ended stationary process $\{(S_t, y_t, \e_t)\}_{- \infty < t < \infty}$ (by Kolmogorov's existence theorem \cite[Theorem 12.1.2]{Dud02}, such a process exists). Note that as before this is a Markov chain whose transition probability is defined by the behavior policy $\pi^o$ together with the update rules (\ref{eq-td1}) and (\ref{eq-td3}) for $\e_t$, $y_t$ and $\lambda_t$, and the marginal distribution of each $(S_t, y_t, \e_t)$ is $\zeta$. We keep using the notation $\Pr_\zeta$ and $\E_\zeta$ for this double-ended stationary Markov chain. 

Recall the shorthand notation (\ref{eq-prod}) introduced at the beginning of Section~\ref{sec-3}: For $k \leq m$, $\rho_k^m = \prod_{i=k}^m \rho_i$, $\lambda_k^m = \prod_{i=k}^m \lambda_i$, $\gamma_k^m = \prod_{i=k}^m  \gamma_i$, and in addition, $\lambda_1^0=\gamma_1^0=\rho_0^{-1}=1$ by convention.

\begin{lem} \label{lem-2}
$\Pr_\zeta$-almost surely, $\sum_{t=1}^{\infty} \lambda_{1-t}^0  \gamma_{1-t}^0 \rho_{-t}^{-1}  \fe(S_{-t})$ is well-defined and finite, and
\begin{equation} 
 \e_0 =  \fe(S_0) + \textstyle{ \sum_{t=1}^{\infty} \lambda_{1-t}^0  \gamma_{1-t}^0 \rho_{-t}^{-1} \, \fe(S_{-t})}.
\label{eq-e}
\end{equation}
\end{lem}

\begin{proof}
First, we show $\E_\zeta \big[ \sum_{t=1}^\infty \gamma_{1-t}^0 \rho_{-t}^{-1} \big] < \infty$. Indeed, 
\begin{align*}
  \textstyle{ \E_\zeta \big[ \sum_{t=1}^\infty \gamma_{1-t}^0 \rho_{-t}^{-1} \big]} =  \textstyle{ \sum_{t=1}^\infty \E_\zeta \big[ \gamma_{1-t}^0 \rho_{-t}^{-1} \big]} 
   =  \sum_{t=1}^\infty \zeta_{\S}^\top (\P \Gamma)^t \1 < \infty,
\end{align*}
where the first equality follows from the monotone convergence theorem, 
the second equality from Condition~\ref{cond-pol}(ii) and a direct calculation, and the last inequality follows from Condition~\ref{cond-pol}(i) (since $(I - \P \Gamma)^{-1} = \sum_{t=0}^\infty (\P \Gamma)^t$). This implies $\sum_{t=1}^\infty \gamma_{1-t}^0 \rho_{-t}^{-1} < \infty$, $\Pr_\zeta$-a.s., so $\gamma_{1-t}^0 \rho_{-t}^{-1}  \to 0$ as $t \to \infty$, $\Pr_\zeta$-a.s. Since $\lambda_{1-t}^0 \leq 1$ for all $t$, it also implies that 
$$\textstyle{ \E_\zeta \big[ \sum_{t=1}^{\infty} \lambda_{1-t}^0  \gamma_{1-t}^0 \rho_{-t}^{-1} \, \| \fe(S_{-t}) \| \big]  \leq \max_{s \in \S} \| \fe(s) \| \cdot \E_\zeta \big[ \sum_{t=1}^{\infty}  \gamma_{1-t}^0 \rho_{-t}^{-1} \big] < \infty}.$$ 
It then follows from a theorem on integration \cite[Theorem 1.38, p.\ 28-29]{Rudin66} that $\Pr_\zeta$-almost surely, the infinite series $\sum_{t=1}^{\infty} \lambda_{1-t}^0  \gamma_{1-t}^0 \rho_{-t}^{-1}  \fe(S_{-t})$ converges to a finite limit.

We now prove the expression for $\e_0$. By unfolding the iteration (\ref{eq-td1}) for $\e_t$ backwards in time, we have for all $m \geq 1$,
\begin{equation} \label{eq-prf-e}
 \textstyle{\e_0 = \fe(S_0) + \sum_{t=1}^{m-1} \lambda_{1-t}^0  \gamma_{1-t}^0 \rho_{-t}^{-1} \, \fe(S_{-t})  +  \lambda_{1-m}^0  \gamma_{1-m}^0 \rho_{-m}^{-1}  \,\e_{-m}.}
\end{equation} 
Let $m \to \infty$ in the r.h.s.\ of (\ref{eq-prf-e}). For the last term, the trace $\e_{-m}$  lies in a bounded set by Condition~\ref{cond-lambda}(ii), $\lambda_{1-m}^0 \leq 1$, and as we just showed, $\gamma_{1-m}^0 \rho_{-m}^{-1}  \to 0$, $\Pr_\zeta$-a.s. So the last term converges to zero $\Pr_\zeta$-a.s. Also as we just showed, the second term converges $\Pr_\zeta$-almost surely to $\sum_{t=1}^{\infty} \lambda_{1-t}^0  \gamma_{1-t}^0 \rho_{-t}^{-1}  \fe(S_{-t})$. The expression (\ref{eq-e}) for $\e_0$ then follows.
\end{proof}

\begin{proof}[Proof of Theorem~\ref{thm-2}]
Treating $\lambda_1^0 =\gamma_1^0=\rho_0^{-1}=1$, we write the expression of $\e_0$ given in Lemma~\ref{lem-2} as $\e_0  = \sum_{t=0}^{\infty} \lambda_{1-t}^0  \gamma_{1-t}^0 \rho_{-t}^{-1} \, \fe(S_{-t})$, $\Pr_\zeta$-a.s. 
We use this expression to calculate first $\E_\zeta \big[ \e_0 \cdot \rho_0 f(S_0^1) \big]$ for an arbitrary function $f$ on $\S \times \S$. (Note that $f$ is bounded and measurable, since $\S$ is finite.)
We have
\begin{align}
 \E_\zeta \big[ \e_0 \cdot \rho_0 f(S_0^1) \big] & = \textstyle{  \sum_{t=0}^\infty \E_\zeta \Big[ \lambda_{1-t}^0  \gamma_{1-t}^0 \rho_{-t}^{-1} \, \fe(S_{-t}) \cdot \rho_0 f(S_0^1) \Big]} \notag \\
     & = \textstyle{   \sum_{t=0}^\infty \E_\zeta \Big[ \lambda_{1}^t  \gamma_{1}^t \rho_{0}^{t-1} \fe(S_0) \cdot \rho_t  f(S_t^{t+1}) \Big]} \notag  \\
     & =  \textstyle{  \sum_{t=0}^\infty \E_\zeta \Big[ \fe(S_0) \cdot \E_\zeta \left[ \lambda_{1}^t  \gamma_{1}^t \rho_{0}^{t} \,  f(S_t^{t+1}) \mid S_0, y_0, \e_0 \right] \Big] }
\label{cal-2} 
\end{align} 
where we used the stationarity of the double-ended state-trace process to derive the second equality, and we changed the order of expectation and summation in the first equality. This change is justified by the dominated convergence theorem, and so are similar interchanges of expectation and summation that will appear in the rest of this proof.

To proceed with the calculation, we relate the expectations in the summation in (\ref{cal-2}) to expectations w.r.t.\ the process with probability measure $\Pr^\pi_\zeta$ introduced in Section~\ref{sec-3.2} (which we recall is induced by the target policy $\pi$ and involves the randomized stopping time $\tau$).
Let $\tilde\E^\pi_\zeta$ denote expectation w.r.t.\ the marginal of $\Pr^\pi_\zeta$ on the space of $\{(S_t, y_t, \e_t)\}_{t \geq 0}$. From the change of measure performed through $\rho_{0}^{t}$, we have
\begin{equation}
 \E_\zeta \left[ \lambda_{1}^t  \gamma_{1}^t \rho_{0}^{t} \,  f(S_t^{t+1}) \mid S_0, y_0, \e_0 \right]  = \tilde\E^\pi_\zeta \left[ \lambda_{1}^t  \gamma_{1}^t  \,  f(S_t^{t+1}) \mid S_0, y_0, \e_0 \right], \quad  t \geq 0.  \label{cal-2b}
\end{equation}
Combining this with (\ref{cal-2}) and using the fact that $\zeta$ is the marginal distribution of $(S_0, y_0, \e_0)$ in both processes, we obtain
\begin{align}
  \E_\zeta \big[ \e_0 \cdot \rho_0 f(S_0^1) \big] & = \textstyle{\sum_{t=0}^\infty \tilde\E^\pi_\zeta \Big[ \fe(S_0) \cdot \tilde\E^\pi_\zeta \left[ \lambda_{1}^t  \gamma_{1}^t  \,  f(S_t^{t+1}) \mid S_0, y_0, \e_0 \right] \Big]}  \notag \\ 
  & =  \textstyle{ \tilde\E^\pi_\zeta \Big[ \fe(S_0) \cdot \sum_{t=0}^\infty \tilde\E^\pi_\zeta \left[ \lambda_{1}^t  \gamma_{1}^t  \,  f(S_t^{t+1}) \mid S_0 \right] \Big].} \label{cal-2c}
\end{align}  

We now use (\ref{cal-2c}) to calculate $\E_\zeta \big[ \e_0  \, \delta_0(v)  \big]$ for a given function $v$. 
Recall from (\ref{eq-td2}) that $\delta_0(v) = \rho_0 \cdot \big(r(S_0^1) + \gamma_1 v(S_1) - v(S_0) \big)$,
so we let $f(S_t^{t+1}) = r(S_t^{t+1}) + \gamma_{t+1} v(S_{t+1}) - v(S_t)$ in (\ref{cal-2c}). 
Since $\E^\pi_\zeta[ r(S_t^{t+1}) \mid S_0^t] = r_\pi(S_t)$, we have
\begin{align*}
 \textstyle{ \sum_{t=0}^\infty \tilde\E^\pi_\zeta \left[ \lambda_{1}^t  \gamma_{1}^t  \,  f(S_t^{t+1}) \mid S_0 \right]} & =
 \textstyle{ \sum_{t=0}^\infty \tilde\E^\pi_\zeta \left[ \lambda_{1}^t  \gamma_{1}^t  \Big( r_\pi(S_t) + \gamma_{t+1} v(S_{t+1}) - v(S_t) \Big) \, \Big| \, S_0 \right]} \\
 & = (Tv - v)(S_0),
\end{align*} 
where the last equality follows from the expression (\ref{eq-Tdiff}) for $TV - V$. Therefore, by (\ref{cal-2c}),
\begin{equation} \label{cal-3}
 \E_\zeta \big[ \e_0 \, \delta_0(v)   \big]  =  \textstyle{\sum_{s \in \S} \zeta_\S(s) \, \phi(s) \cdot (Tv-v)(s), }
\end{equation}  
and this shows that $\E_\zeta \big[ \e_0 \, \delta_0(v)  \big] = 0$ is equivalent to $Tv - v \perp_{\zeta_\S} \L_\phi$.
\end{proof}

We now prove Corollary~\ref{cor-tdcmp}.

\begin{proof}[Proof of Corollary~\ref{cor-tdcmp}]
We apply Theorem~\ref{thm-2} to each state-trace process $\big\{\big(S_t, y^{(i)}_t, \e^{(i)}_t \big) \big\}$ for $i =1, 2, \ldots, m$. Specifically, by (\ref{cal-3}) and the definition (\ref{eq-tdcmp1}) of $\e^{(i)}_t$,
$$ \E_{\zeta^{(i)}} \big[ \e_0^{(i)}  \delta_0(v) \big]  = \textstyle{\sum_{s \in \S} \zeta_\S(s) \cdot \phi(s) \I(s \in \S_i) \cdot \big(T^{(i)}v - v \big)(s).}$$
Hence
\begin{align*}
 \textstyle{ \sum_{i=1}^m \E_{\zeta^{(i)}} \big[ \e_0^{(i)}  \delta_0(v) \big]} & = \textstyle{  \sum_{s \in \S} \zeta_\S(s) \, \phi(s) \cdot \Big[ \sum_{i=1}^m \I(s \in \S_i)  \cdot \big(T^{(i)}v - v \big)(s) \Big] } \\
 & = \textstyle{ \sum_{s \in \S} \zeta_\S(s) \, \phi(s) \cdot (Tv-v)(s),}
\end{align*} 
where the last equality follows from the definition (\ref{eq-tdcmp-T}) of the operator $T$. 
This shows that the linear equation in $v$, $\sum_{i=1}^m \E_{\zeta^{(i)}} \big[ \e_0^{(i)}  \delta_0(v) \big] = 0$, is equivalently 
$T v - v \perp_{\zeta_\S} \L_\phi$.

We now prove that $T$ has $v_\pi$ as its unique fixed point and is a contraction with respect to a weighted sup-norm---in other words, Theorem~\ref{thm-gbo} applies to $T$. For this, it suffices to show that $T$ satisfies the conditions of Theorem~\ref{thm-gbo}, namely, $T$ is a generalized Bellman operator associated with a randomized stopping time $\tau$ that satisfies $\Pr^\pi(\tau \geq 1 \mid S_0 = s) > 0$ for all states $s \in \S$. 
We can define such a random time $\tau$ from the randomized stopping times $\tau^{(i)}$ associated with the Bellman operators $T^{(i)}$. 
In particular, by enlarging the probability space if necessary, we can regard $\tau^{(i)}$, $i=1, 2, \ldots, m$, as being defined on the same probability space.
\footnote{That we can do so is clear from the definition of each $\tau^{(i)}$ as described at the beginning of Section~\ref{sec-3.2} and from the fact that for each $i$, the initial distribution $\zeta^{(i)}$ on $(S_0, y_0^{(i)}, \e_0^{(i)})$ has the same marginal on $\S$, which is $\zeta_\S$.} 
We then let $\tau = \tau^{(i)}$ if $S_0 \in \S_i$. With this definition, we have $\Pr^\pi(\tau \geq 1 \mid S_0 = s) > 0$ for all states $s \in \S$ (since $\tau^{(i)} \geq 1$ a.s.\ for all $i$). For each set $\S_i$, by (\ref{eq-gbe1}), the component mappings of the generalized Bellman operator $T_\tau$ associated with $\tau$ are given by 
$$ (T_\tau v)(s) \! = \!\textstyle{ \E^\pi_s \big[  R^{\tau} + \gamma_1^{\tau} \, v(S_{\tau}) \big]} = \!\textstyle{ \E^\pi_s \big[  R^{\tau^{(i)}} + \gamma_1^{\tau^{(i)}} \, v(S_{\tau^{(i)}}) \big]} = (T^{(i)}v)(s), \quad s \in \S_i.$$
So $T_\tau = T$ by the definition (\ref{eq-tdcmp-T}) of $T$; i.e., $T$ is the Bellman operator associated with the randomized stopping time $\tau$.
\end{proof}

\section{Numerical Study} \label{sec-4}

In this section, we first use a toy problem to illustrate the behavior of traces calculated by off-policy LSTD($\lambda$) for constant $\lambda$ and for $\lambda$ that evolves according to a simple special case of our proposed scheme described in Example~\ref{ex-scaling}. We then compare the behavior of LSTD for various choices of $\lambda$, on the toy problem and on the Mountain Car problem.

\subsection{Behavior of Traces}

The toy problem we use in this study has $21$ states, arranged as shown in Figure~\ref{fig-trgrph} (left). One state is located at the centre, and the rest of the states split evenly into four groups, indicated by the four loops in the figure. The topology of the transition graph is the same for the target and behavior policies. We have drawn the transition graph only for the northeast group in Figure~\ref{fig-trgrph} (left); the states in each of the other three groups are arranged in the same manner and have the same transition structure. Given this symmetry, to specify the transition matrices $P$ and $P^o$ for the target and behavior policies $\pi$ and $\pi^o$ respectively, it suffices to specify the submatrices for the central state and one of the groups.
If we label the central state as state $1$ and the states in the northeast group clockwise as states $2$-$6$, the submatrices of $P$ and $P^o$ for these states are given, respectively, by
\begin{align*}
\pi: & \quad \left( 
\begin{array}{cccccc}
  0  &  0.25 & 0  &  0  & 0  &  0\\
  0 &  0 & 1 & 0  & 0  &  0\\
     0 & 0.2 & 0  & 0.8  & 0  & 0\\
     0  & 0  & 0.2  & 0  & 0.8 & 0\\
     0  & 0 &  0  & 0.2 & 0  & 0.8\\
     0.8 & 0 &  0 & 0  & 0.2  & 0 
     \end{array} \right), &
\pi^o: & \quad    
     \left(
    \begin{array}{cccccc}
  0  &  0.25 & 0  &  0  & 0  &  0\\
  0 &  0 & 1 & 0  & 0  &  0\\
     0 & 0.5 & 0  & 0.5  & 0  & 0\\
     0  & 0  & 0.5  & 0  & 0.5 & 0\\
     0  & 0 &  0  & 0.5 & 0  & 0.5\\
     0.5 & 0 &  0 & 0  & 0.5  & 0 
     \end{array} \right).
\end{align*}         
Intuitively speaking, from the central state, the system enters one group of states by moving diagonally in one of the four directions with equal probability, and after spending some time in that group, eventually returns to the central state and the process repeats. The behavior policy on average spends more time wandering inside each group than the target policy, while the target policy tends to traverse clockwise through the group more quickly. 

\begin{figure}[!th]
   \centering
      \includegraphics[width=0.27\linewidth]{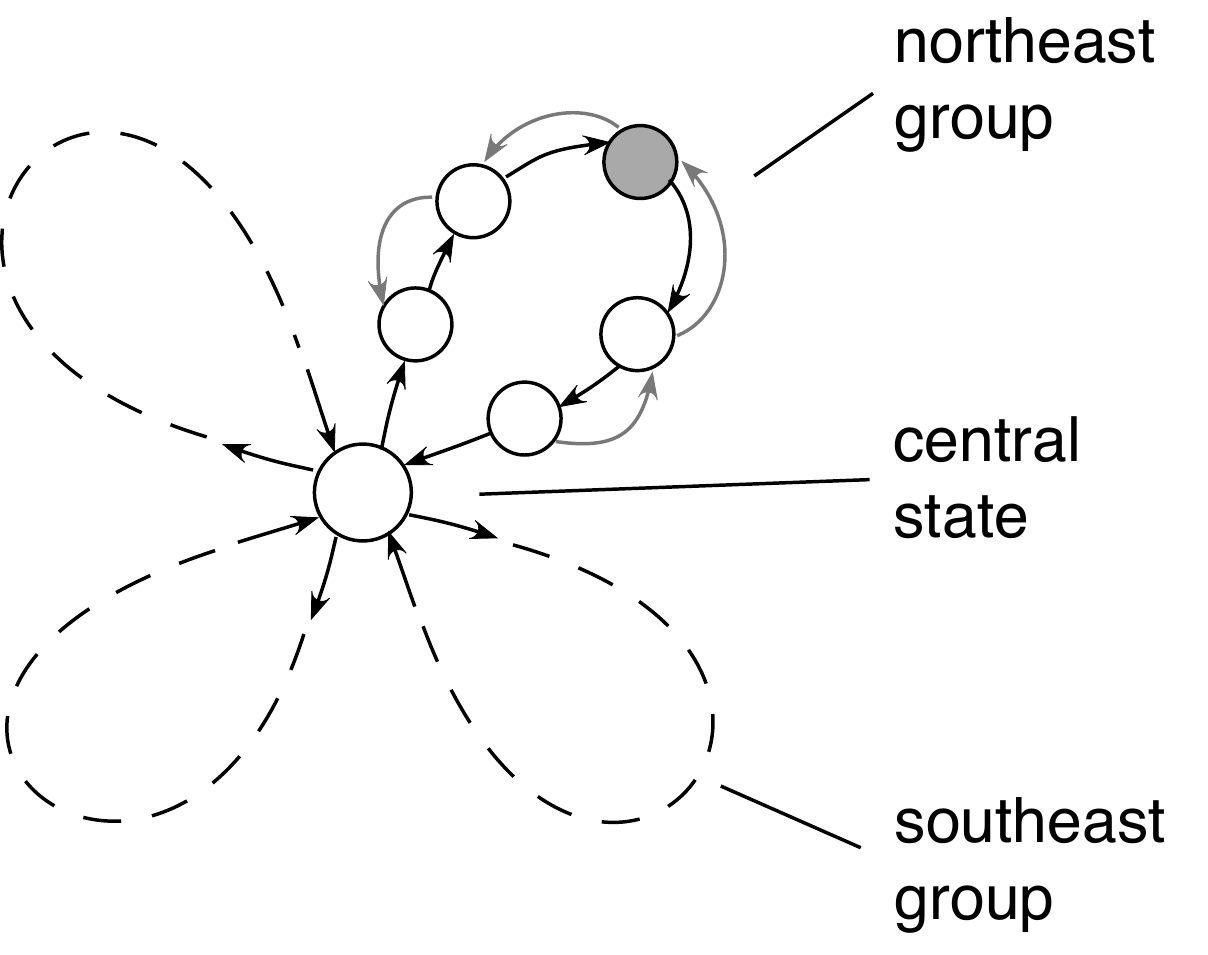} \qquad \qquad \qquad
   \raisebox{11pt}{\includegraphics[width=0.17\linewidth]{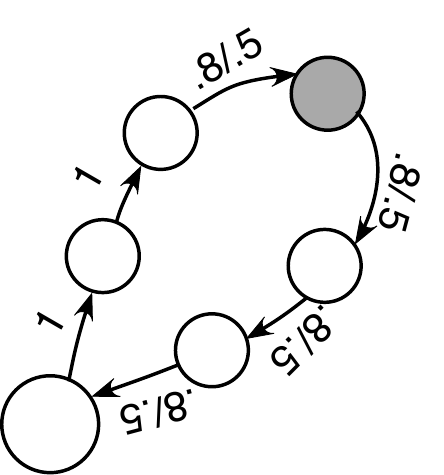}}
 \caption{The transition graph of a toy problem (left) and a cycle pattern in it (right). The numbers appearing in the right graph indicate the importance sampling ratios for the state transitions represented by each directed edge. From such cycle patterns one can infer whether the trace sequence is unbounded almost surely.}
   \label{fig-trgrph}
\end{figure}

All the rewards are zero except for the middle state in each group---for the northeast group, this is the shaded state in Figure~\ref{fig-trgrph} (left). For the two northern (southern) groups, their middle states have reward $1$ ($-1$).
The discount factor is $\gamma=0.9$ for all states.
As to features, we aggregate states into $5$ groups: the $4$ groups mentioned earlier and the central state forming its own group, and 
we let each state have $5$ binary features indicating its membership.
 
We now discuss and illustrate the behavior of traces in this toy problem. For comparison, we first do this for the off-policy TD($\lambda$) with a constant $\lambda$. It can help explain the challenges in off-policy TD learning and our motivation for proposing the new scheme of setting $\lambda$.

\subsubsection{Traces for Constant $\lambda$} \label{sec-4.1.1}

In this experiment we let $\lambda = 1$ and consider the trace iterates $\{\e_t\}$ calculated by TD($1$). In general, by identifying certain cycle patterns in the transition graph, one can infer whether $\{\e_t\}$ will be unbounded over time almost surely \cite[Section 3.1]{Yu-siam-lstd}.  
Figure~\ref{fig-trgrph} (right) shows such a cycle of states in the transition graph of the toy problem. 
It consists of the central state and the northeast group of states. Labeled on each edge of the cycle is the importance sampling ratio for that state transition. Traversing through the cycle once from any starting state, and multiplying together the importance sampling ratios of each edge and the discount factors of the destination states, we get
$\left( \tfrac{0.8}{0.5} \right)^4 \cdot \gamma^6 = \left( \tfrac{0.8}{0.5} \right)^4 \cdot 0.9^6 > 1$. From this one can infer that $\{\e_t\}$ calculated by off-policy TD($1$) will be unbounded in this problem (cf.\ \citealp[Prop.\ 3.1]{Yu-siam-lstd}).

We plotted in the upper left graph of Figure \ref{fig-trace1} the Euclidean norm $\|\e_t\|$ of the traces over $8 \times 10^5$ iterations for TD($1$). One can see the recurring spikes and the exceptionally large values of some of these spikes in the plot. This is consistent with the unboundedness of $\{\e_t\}$ just discussed.

\begin{figure}[!t]
   \includegraphics[width=0.5\linewidth]{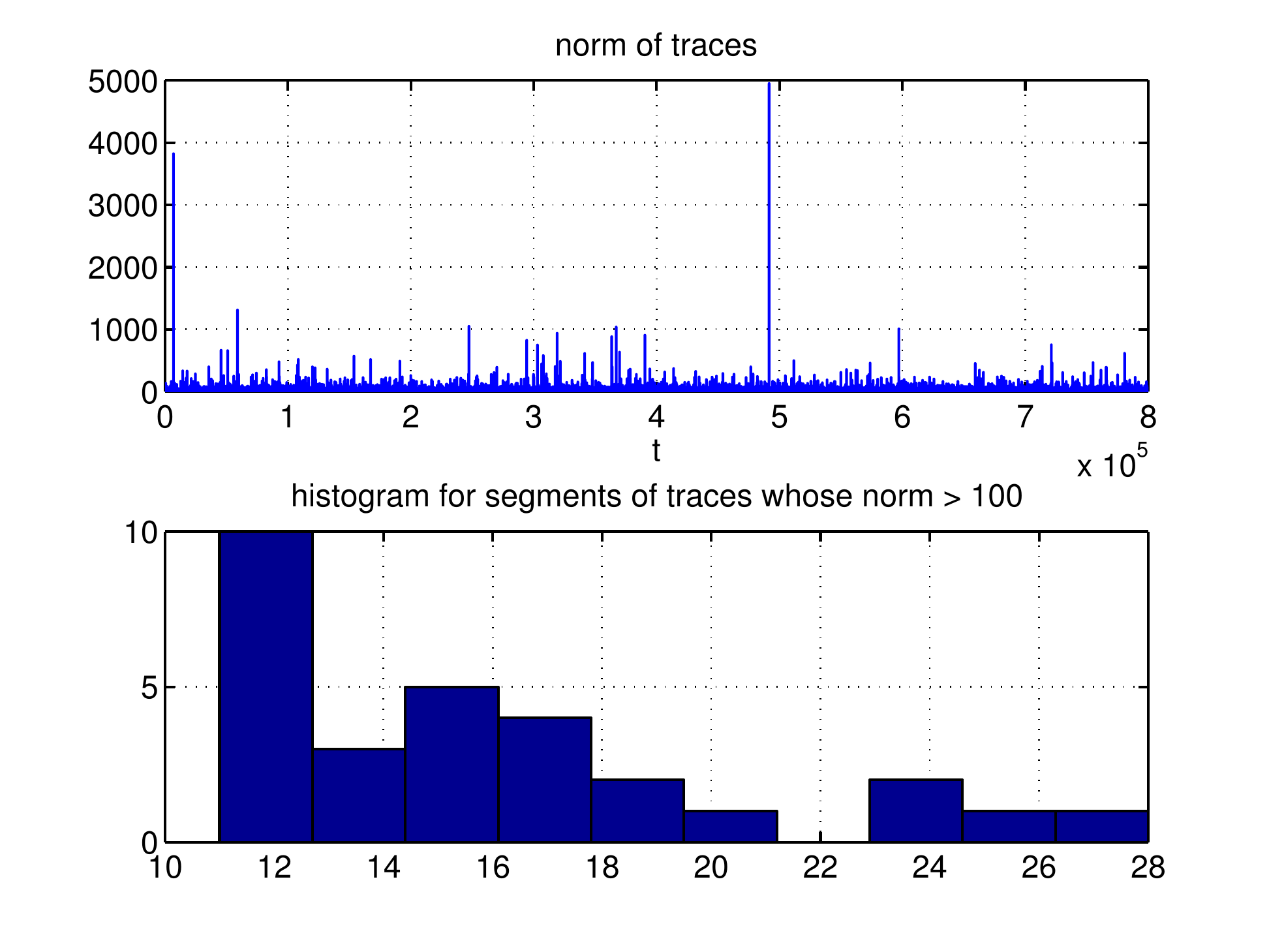} \hfill
   \raisebox{0pt}{\includegraphics[width=0.5\linewidth]{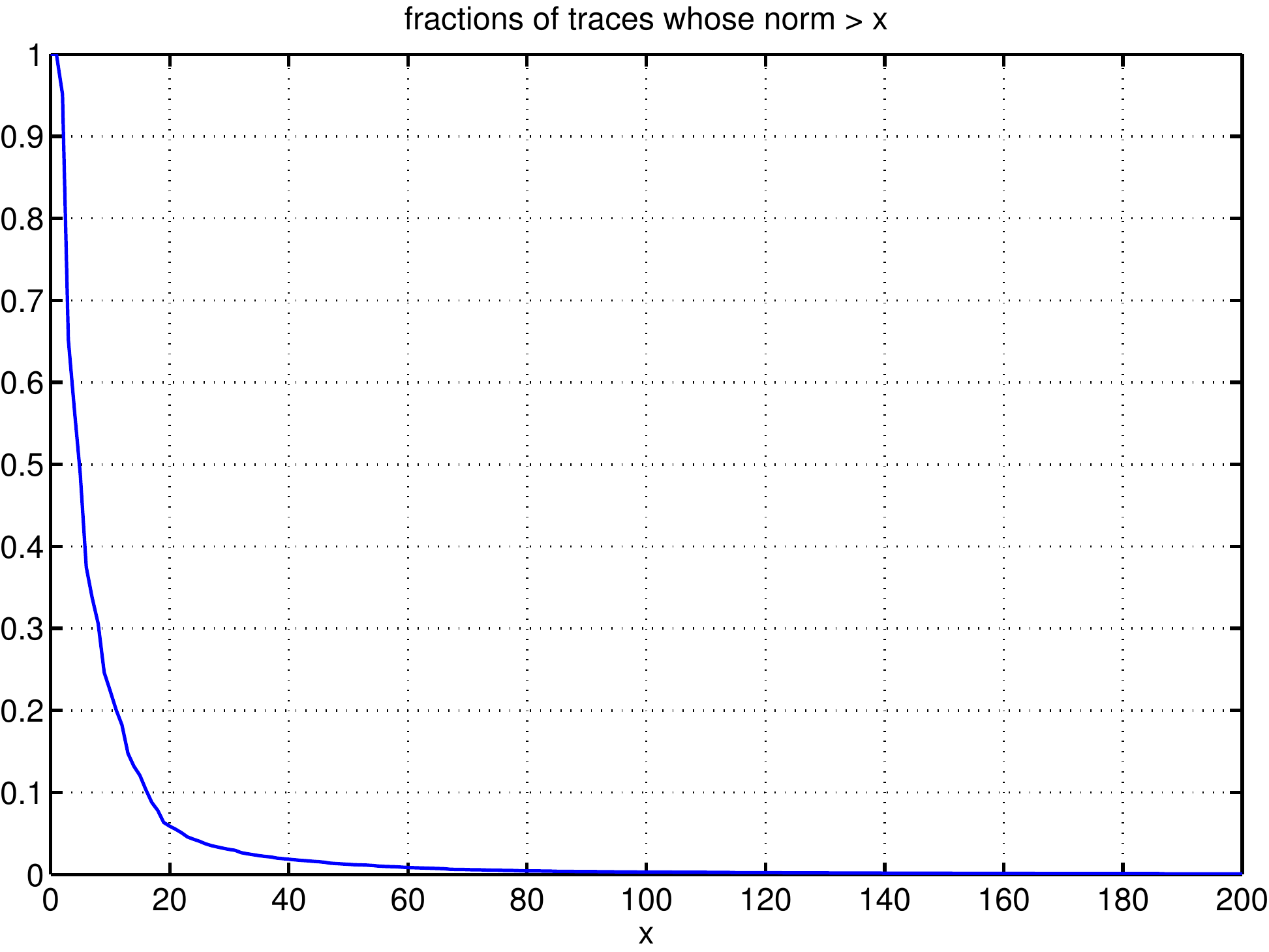}}\vspace*{-0.3cm}
   \caption{Statistics of traces for TD($1$) in a toy problem (see the text in Section~\ref{sec-4.1.1} for detailed explanations).} \label{fig-trace1}
\end{figure}%

The unboundedness of $\{\e_t\}$ tells us that the invariant probability measure $\zeta$ of the state-trace process $\{(S_t, \e_t)\}$ has an unbounded support. 
Despite this unboundedness, $\{\e_t\}$ is bounded in probability~\cite[Lemma 3.4]{Yu-siam-lstd} and under the invariant distribution $\zeta$, $\E_\zeta \big[ \| \e_0 \| \big] < \infty$ \cite[Prop.\ 3.2]{Yu-siam-lstd}. The latter property implies that under the invariant distribution, the probability of $\| \e_0 \| > x$ decreases as $o(1/x)$ for large $x$.
Since the empirical distribution of the state-trace process converges to $\zeta$ almost surely \cite[Theorem 3.2]{Yu-siam-lstd}, during a run of many iterations, we expect to see the fraction of traces with $\| \e_t \| > x$  drop in a similar way as $x$ increases.

The simulation result shown in the right part of Figure~\ref{fig-trace1} agrees with the preceding discussion.
Plotted in the graph are fractions of traces with $\|\e_t \| > x$ during $8\times 10^5$ iterations (the vertical axis indicates the fraction, and the horizontal axis indicates $x$). It can be seen that despite the recurring spikes in $\|\e_t\|$ during the entire run, the fraction of traces with large magnitude $x$ drops sharply with the increase in $x$.

While only a small fraction of traces have exceptionally large magnitude, they can occur in consecutive iterations. This is illustrated by the histogram in the lower left part of Figure~\ref{fig-trace1}. The histogram concerns the excursions of the trajectory $\{\e_t\}$ outside of the ball $\{  \e \in \re^{n} \mid  \| \e \| \leq 100\}$. The horizontal axis indicates the lengths of the excursions (where the length is the number of iterations an excursion contains), 
and the vertical axis indicates how many excursions of length $x$ occurred during the $8\times 10^5$ iterations of the experimental run. 
We plotted the histogram for lengths $x > 10$. It can be seen that one can have large traces during many consecutive iterations. Such behavior, although tolerable by LSTD, is especially detrimental to TD algorithms and can disrupt their learning. This is our main motivation for suggesting the use of $\lambda$-parameters to bound the traces directly.

\subsubsection{Traces with Evolving $\lambda$} \label{sec-4.1.2}

We now proceed to illustrate the behavior of traces and LSTD for $\lambda$ that evolves according to our proposed scheme. Specifically, for this demonstration, we will use the simple scaling example given by (\ref{eq-ex1})-(\ref{eq-ex1b}) in Example~\ref{ex-scaling}, with all the thresholds $\C_{ss'}$ being the same constant $\C$. That is, the update rule for $\e_t$ used in this experiment is
\begin{equation} \label{eq-ex1c}
 \e_t = \begin{cases}
  \gamma_t \,  \rho_{t-1} \, \e_{t-1} +  \fe(S_t) \ & \text{if} \ \ \gamma_t \rho_{t-1} \| \e_{t-1}\|_2 \leq \C; \\
  \C \cdot \tfrac{\e_{t-1}}{\|\e_{t-1}\|_2} +  \fe(S_t)
    &  \text{otherwise}.
    \end{cases}
\end{equation}
We first simulate the state-trace process to illustrate the ergodicity of this process stated by Theorem~\ref{thm-1}.
We will shortly study the performance of LSTD for different values of $\C$ in Section~\ref{sec-4.2.1}.
The results of these two experiments are shown in Figure~\ref{fig-trace2} and Figures~\ref{fig-cmpsol}-\ref{fig-cmpsol2}, respectively, and the details are as follows.

According to the ergodicity result of Theorem~\ref{thm-1}, no matter from which initial state and trace pair $(S_0, \e_0)$ we generate a trajectory $\{(S_t, \e_t)\}_{0 \leq t \leq \bar t}$ according to the behavior policy, the empirical distribution of state and trace pairs in this trajectory should converge, as $\bar t \to \infty$, to the same distribution on $\S \times \rn$, which is the marginal of the invariant probability measure $\zeta$ on that space. Since $\S$ is discrete, we can verify this fact by examining the empirical conditional distribution of the trace given the state. In other words, for each state $s$, we examine the empirical distribution of the trace for the sub-trajectory $(S_{t_k}, \e_{t_k})$, $k=1, 2, \ldots$, where $S_{t_k} = s$ and it is the $k$th visit to state $s$ by the trajectory. We check if this empirical distribution converges to the same one as we increase the length $\bar t$ of the trajectory and as we vary the initial condition $(S_0, \e_0)$. 

\begin{figure}[thb]
   \hspace*{-1.8cm} \includegraphics[width=1.2\linewidth]{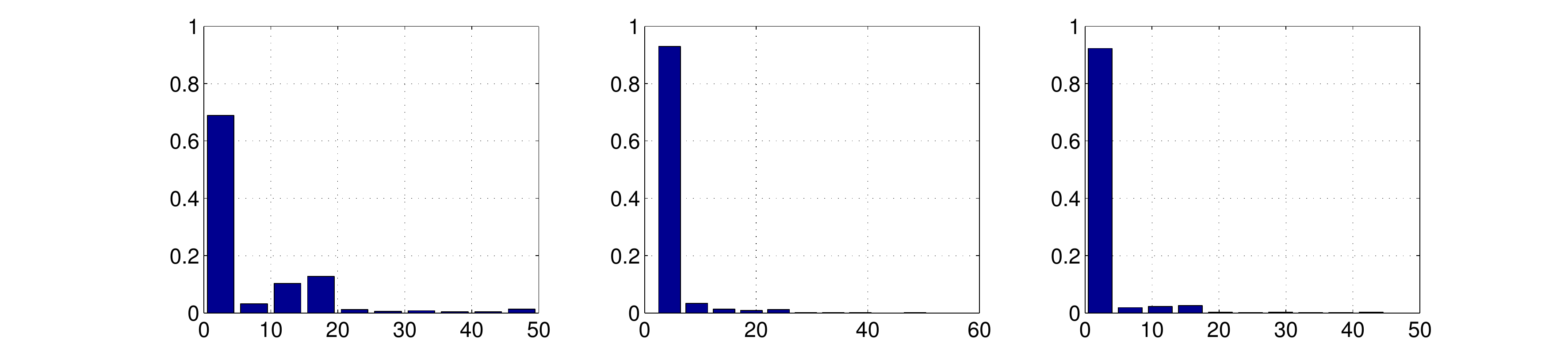}\\*[5pt]
   \hspace*{-1.8cm} \includegraphics[width=1.2\linewidth]{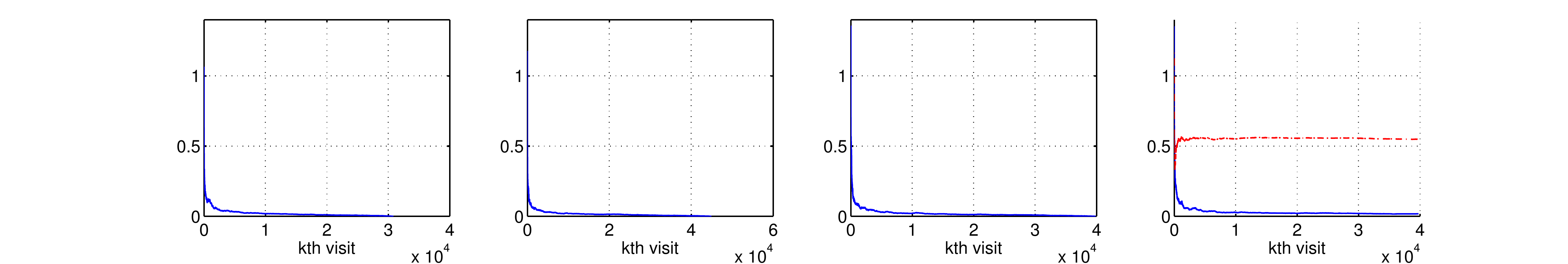}\vspace*{-0.2cm}
   \caption{Demonstration of convergence of empirical conditional distributions on the trace space. (See the text in Section~\ref{sec-4.1.2} for detailed explanations.)} \label{fig-trace2}
\end{figure}%

To give a sense of what these limiting distributions over the trace space look like, we set the parameter $C=50$ and generated a long trajectory with $8 \times 10^5$ iterations.
In the top row of Figure~\ref{fig-trace2}, we plotted three normalized histograms of the first trace component for the sub-trajectories associated with three states of the toy problem, respectively: the central state (left), the middle state of the northeast group (middle), and the first state of the southeast group (right).

To check whether the empirical conditional distributions on the trace space converge along the sub-trajectory $\{(S_{t_k}, \e_{t_k})\}$ for a given state, we compare the characteristic functions $f_k$ of these distributions with the characteristic function $f$ of the empirical conditional distribution obtained at the end of the sub-trajectory during the experimental run. In particular, we evaluate all these (complex-valued) characteristic functions at $500$ points, which are chosen randomly according to the multivariate normal distribution on $\Re^5$ with mean $0$ and covariance matrix $200^2 I$. We take the maximal difference between $f_k$ and $f$ at these $500$ points as an indicator of the deviation between the two corresponding distributions.
\footnote{Recall that a tight sequence $\{p_k\}$ of probability distributions on $\re^m$ converges to a probability distribution $p$ if and only if the characteristic functions of $p_k$ converge pointwise to the characteristic function of $p$ \cite[Lemma 9.5.5]{Dud02}. Recall also that we are dealing with convergence in distribution here, which is much weaker than convergence in total variation, so we cannot use total variation as a metric on the distribution space in this case.\label{footnote-char}}
In the bottom row of Figure~\ref{fig-trace2}, the first three plots show the difference curves obtained in the way just described, for three different states, respectively. These three states are the same ones mentioned earlier in the description of the top row of Figure~\ref{fig-trace2}. The horizontal axis of these plots indicates $k$, the number of visits to the corresponding state. As can be seen, the difference curves all tend to zero as $k$ increases, which is consistent with the predicted convergence of the empirical conditional distributions on the trace space.

So far we compared the empirical distributions along the same trajectory. Next we compare them against the one obtained at the end of another trajectory that starts from a different initial condition. The difference curve for one state (the first state in the southeast group) is plotted in the last graph in the bottom row of Figure~\ref{fig-trace2}, and it is the lower curve in that graph. As can be seen, the curve tends to zero, suggesting that the limiting distribution of these empirical distributions does not change if we vary the initial condition, which is consistent with Theorem~\ref{thm-1}. 

For comparison, we also plotted the difference curve when these same empirical conditional distributions are compared against the empirical conditional distribution obtained from the same trajectory but for a different state (specifically, the middle state of the northeast group). This is the upper curve in the last graph in the bottom row of Figure~\ref{fig-trace2}. It clearly indicates that for the two states, the associated limiting conditional distributions on the trace space are different. It also shows that the characteristic function approach we adopted in this experiment can effectively distinguish between two different distributions (cf.\ Footnote~\ref{footnote-char}).

\subsection{LSTD with Evolving $\lambda$} \label{sec-4.2}

We now present experiments on the LSTD algorithm. 

\subsubsection{A Toy Problem} \label{sec-4.2.1}
Let us first continue with the toy problem of the previous subsection and show how the LSTD algorithm performs in this problem as we vary the parameter $C$ in the $\lambda$ function for bounding the traces.
We ran LSTD for $C = 10, 20, \ldots, 100$, using the same trajectory, for $3 \times 10^5$ iterations, and we computed the (Euclidean) distance of these LSTD solutions to the asymptotic TD($1$) solution (in the space of the $\theta$-parameters), normalized by the norm of the latter. We then repeat this calculation $10$ times, each time with an independently generated trajectory. Plotted in Figure~\ref{fig-cmpsol} (left) against the values of $C$ are the means and standard deviations of the normalized distances of LSTD solutions thus obtained.

\begin{figure}[thb]
   \includegraphics[width=0.49\linewidth]{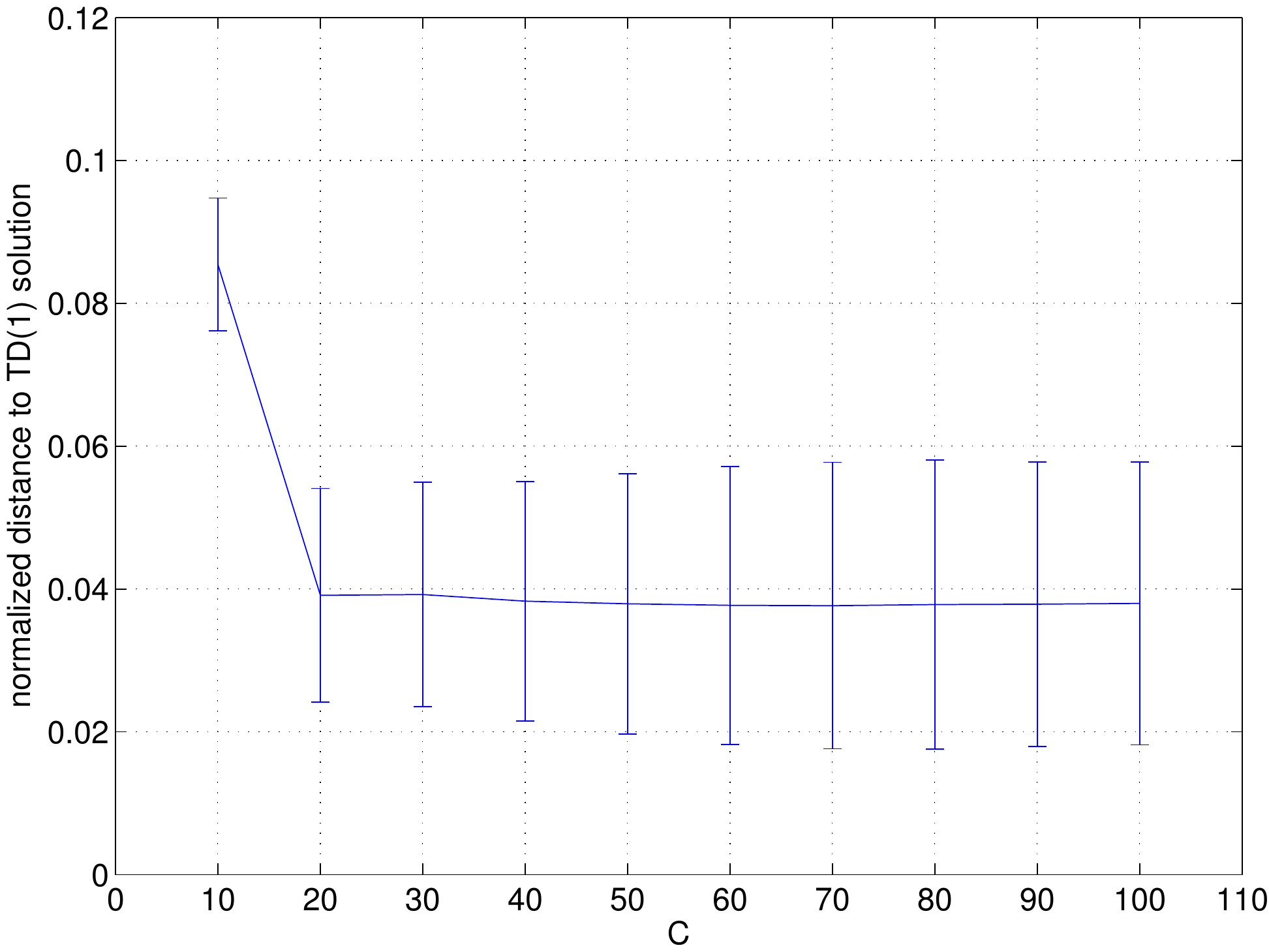} \hfill
   \includegraphics[width=0.49\linewidth]{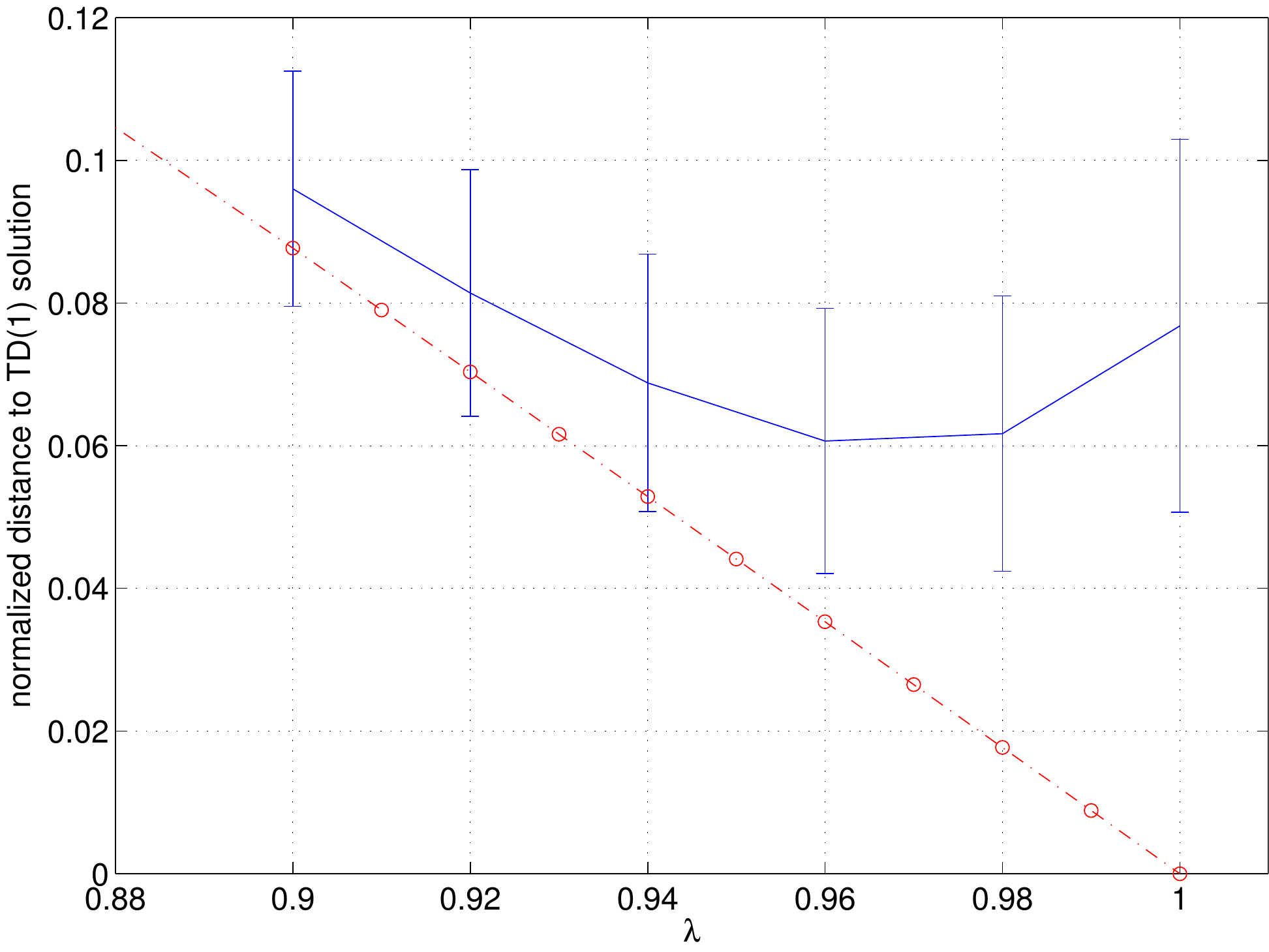}\vspace*{-0.2cm}
   \caption{Compare LSTD solutions with evolving $\lambda$ (left) and with constant $\lambda$ (right). For constant $\lambda$, the red dot-dash curve in the right plot shows the quality of the asymptotic TD($\lambda$) solutions, and LSTD($\lambda$) would approach this curve in the limit, but due to variance issues, it can require an impractically large number of iterations to exhibit this convergent behavior. LSTD with evolving $\lambda$ outperforms LSTD with constant $\lambda$ in this case and effectively archives the quality of TD($\lambda$) solutions for large constant $\lambda$.} \label{fig-cmpsol}
\end{figure}%
\begin{figure}[thb]
   \includegraphics[width=0.49\linewidth]{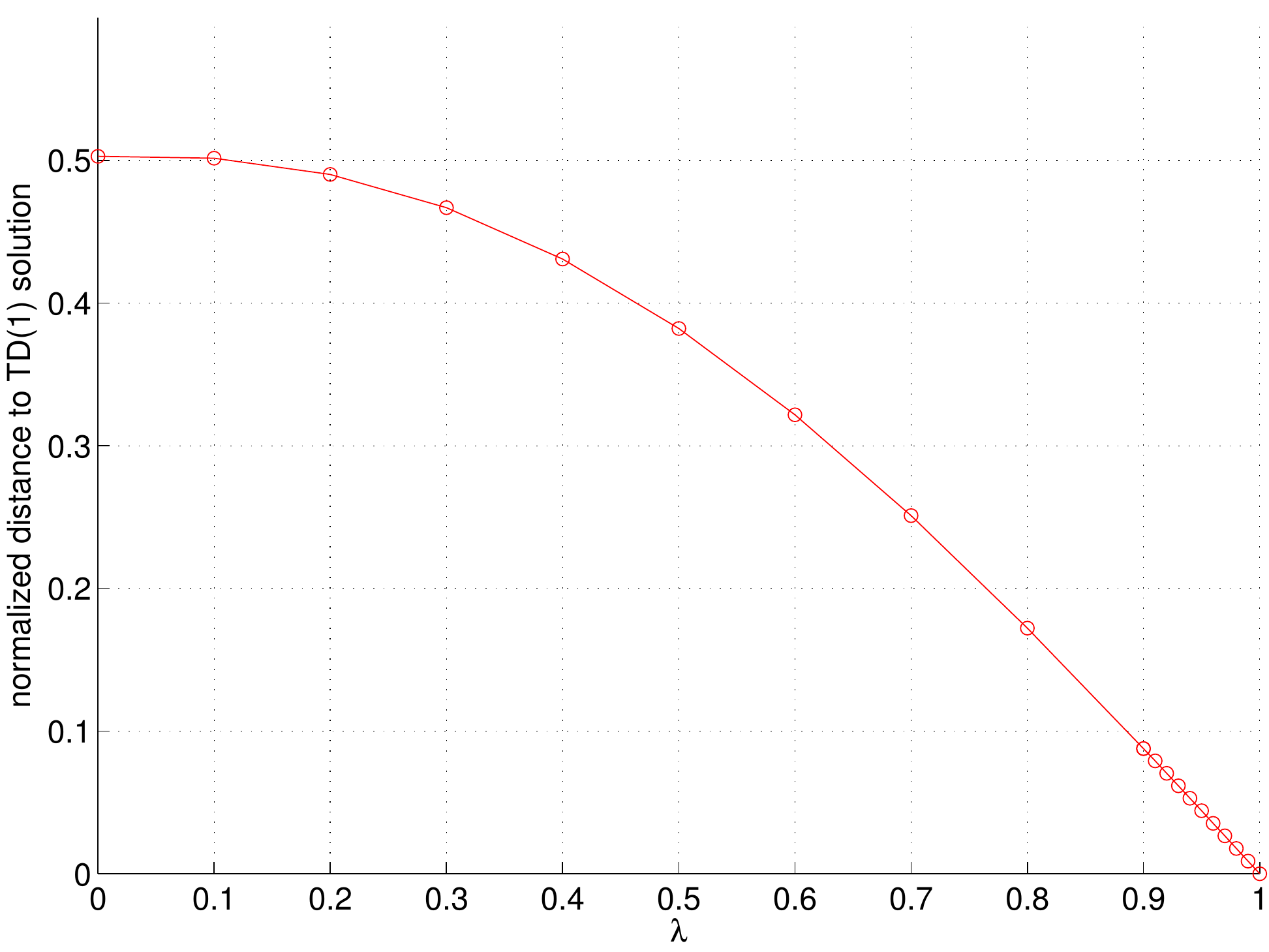} \hfill
   \includegraphics[width=0.49\linewidth]{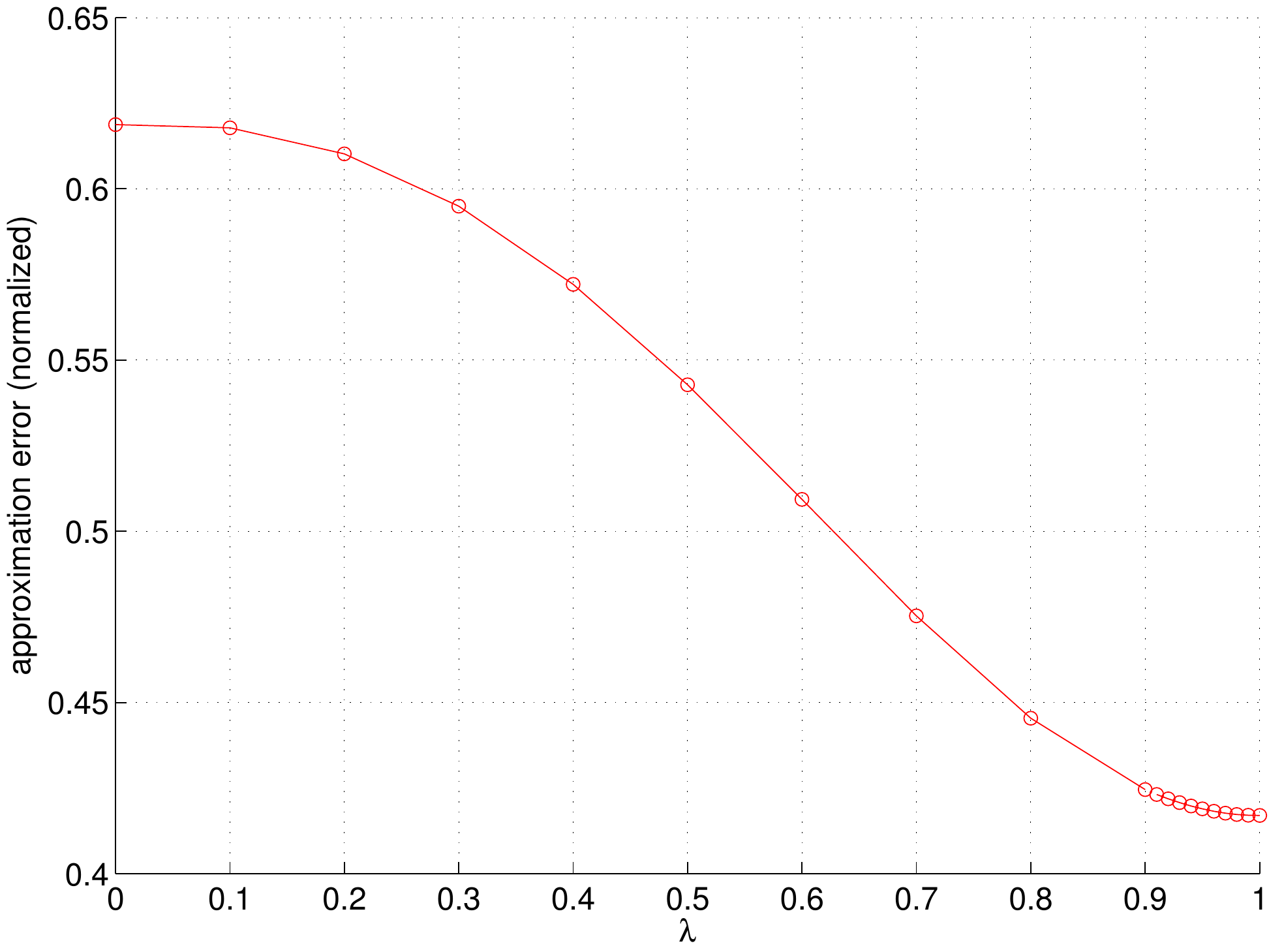}\vspace*{-0.2cm}
   \caption{Approximation quality of asymptotic TD($\lambda$) solutions with constant $\lambda$ for the toy problem. Larger $\lambda$ yields better approximations.} \label{fig-cmpsol2}
\end{figure}%

For comparison, we did the same for LSTD with a large constant $\lambda$: $\lambda = 0.9, 0.92, \ldots, 1$. Figure~\ref{fig-cmpsol} (right) shows the result, where the dash-dot curve indicates the normalized distance of the asymptotic TD($\lambda$) solution---the solution that LSTD($\lambda$) would obtain in the limit. It can be seen that the performance of LSTD deteriorates as $\lambda$ gets close to $1$. We think that this is because due to the high variance issue, the convergence of LSTD with a large constant $\lambda$ is too slow and requires far more iterations than the $3 \times 10^5$ iterations performed. In comparison, LSTD with evolving $\lambda$ behaves better: it works effectively for $C \geq 20$, and the approximation quality it achieved with such $C$ is comparable to that of asymptotic TD($\lambda$) solutions for a large constant $\lambda$ around $0.96$. 

Figure~\ref{fig-cmpsol2} shows the quality of the asymptotic solutions of TD($\lambda$) with constant $\lambda$, for the full range of $\lambda$ values. Plotted in the left graph is the normalized distance of the TD($\lambda$) solution to the TD($1$) solution. Plotted in the right graph is the normalized approximation error of the corresponding approximate value function, where the error is measured by the weighted Euclidean norm with weights specified by $\zeta_{\S}$ (the invariant distribution on $\S$ under the behavior policy), and the normalization is over the weighted norm $\| v_{\pi} \|_{\zeta_{\S}}$ of the true value function $v_\pi$ of the target policy. It can be seen that using $\lambda > 0.9$ provides considerably better approximations for this problem than using small $\lambda$. 

We also found that for this problem, if we set $\lambda$ according to the Retrace algorithm with $\beta=1$ (cf.~(\ref{eq-ex2a}) in Example~\ref{ex-retrace}), then the performance of LSTD is comparable to TD($\lambda$) with a small $\lambda$ around $0.5$.
(Specifically, for Retrace, the normalized distance to the TD($1$) solution is $0.37$, and the normalized approximation error is $0.54$, which are comparable to the numbers for TD($0.5$), as Figure~\ref{fig-cmpsol2} shows.)
This is not surprising, because, as we discussed earlier in Section~\ref{sec-compare-retrace}, in keeping $\lambda_t \rho_{t-1} \leq 1$ always, Retrace and ABQ can be too ``conservative,'' resulting in an overall effect that is like using a small $\lambda$, even though $\lambda_t$ may appear to be large at times. Recall also that this can happen to our proposed scheme too.  In the present experiment, for instance, this can happen when $C$ is small; in particular, the case $C=0$ reduces to LSTD($0$).

\subsubsection{Mountain Car Problem} \label{sec-mcar}

In this subsection we demonstrate LSTD with evolving $\lambda$ on a problem adapted from the well-known Mountain Car problem \citep{SUB}. 
The details of this adaptation, including the target and behavior policies involved, can be found in the report~\cite[Section 5.1, p.\ 23-26]{etd-exp16}; most of these details are not crucial for our experiments, so to avoid distraction, we only describe briefly the experimental setup here.
    
In Mountain Car, the goal is to drive an underpowered car to reach the top of a steep hill, from the bottom of a valley.
A state consists of the position and velocity of the car, whose values lie in the intervals $[-1.2, 0.5]$, $[-0.07, 0.07]$, respectively.
The position $0.5$ corresponds to the desired hill top destination, while the position $-\pi/6$ ($\approx -0.52$) lies at the bottom of a valley that is between the destination and a second hill peaked at $-1.2$ in the opposite direction (see the illustration in Figure~\ref{fig-mcar1}).
Except for the destination state, each state has three available actions: $\{\texttt{back}, \texttt{coast},  \texttt{forward} \}$, and the rewards depend only on the action taken and are $-1.5$, $0$ and $-1$ for the three actions, respectively. The dynamics is as given in \citep{SUB}. 
We consider undiscounted expected total rewards, so the discount factor is $1$ except at the destination state, where the discount factor is $0$ and from where the car enters a rewardless termination state permanently. 

\begin{figure}[thb]
   \centering
  \raisebox{15pt}{\includegraphics[width=0.27\linewidth]{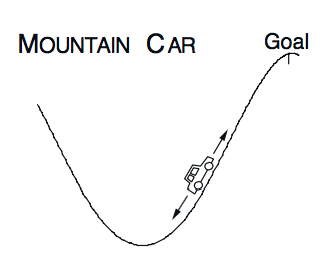}} \qquad \ \  
  \includegraphics[width=0.47\linewidth]{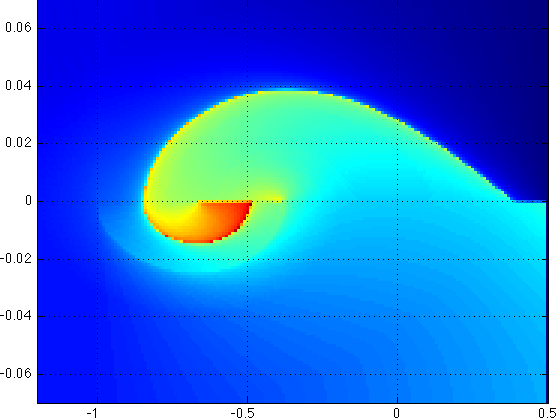} \qquad
    \raisebox{-15pt}{\includegraphics[width=0.062\linewidth]{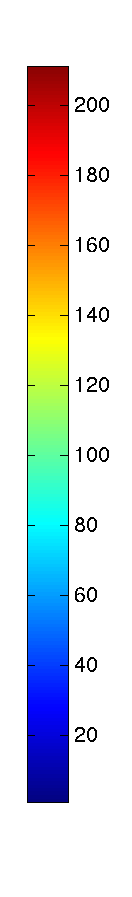}}\vspace*{-0.5cm}
   \caption{Left: Illustration of the Mountain Car problem. Right: Costs $-v_{\pi}$ for this problem are estimated and visualized as a color image, with the coloring scheme indicated by the colorbar. (The horizontal and vertical axes of the image correspond to position and velocity, respectively, for the $2$-dimensional state space of this problem.)} \label{fig-mcar1}
\end{figure}%
\begin{figure}[h!]
   \centering
   \includegraphics[width=0.5\linewidth]{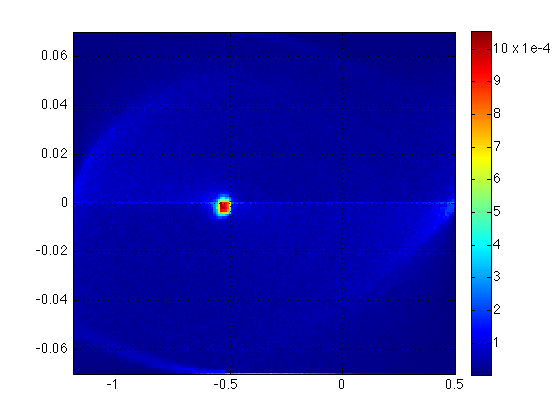} \quad
    \raisebox{7pt}{\includegraphics[width=0.445\linewidth]{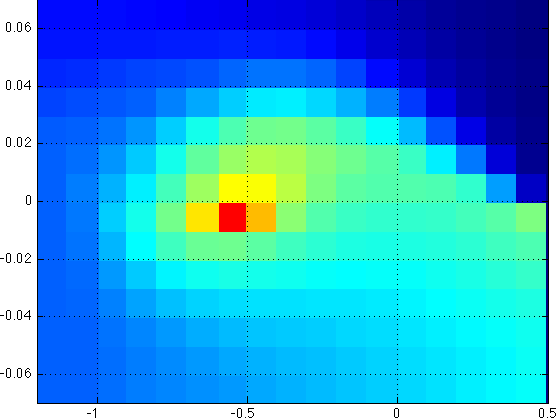}}\vspace*{-0.5cm} 
   \caption{Left: Visualization of weights on states induced by the behavior policy. Right: The color image visualizes the approximation of $- v_{\pi}$ obtained from a discretized model for the Mountain Car problem, where the coloring scheme is the same as that shown in Figure~\ref{fig-mcar1} (right). The quality of this approximation is close to that of TD($0$).} \label{fig-mcar2}
\end{figure}%

The target policy $\pi$ is a simple but reasonably well-behaved policy. On either slopes between the two hills, it tries to increase its energy (kinetic plus gravitational potential energy) by accelerating in the direction of its current motion. If this brings it up to the opposite hill (position $< -1$), it coasts; otherwise, if its velocity drops to near zero, it goes forward or backward with equal probability. Figure~\ref{fig-mcar1} visualizes the total costs $-v_\pi$ of the target policy,
\footnote{The values of $v_{\pi}$ shown in Figure~\ref{fig-mcar1} are estimated by simulating the target policy for each starting state in a set of  $171 \times 141$ points evenly spaced in the position-velocity space. In particular, the position (velocity) interval is evenly divided into subintervals of length $0.01$ ($0.001$), and for each stating state, the target policy is simulated $600$ times.} 
where the horizontal (vertical) axis indicates position (velocity) and the colorbar on the right shows the value corresponding to each color. The discontinuity of the function in certain regions can be seen in this figure.

The behavior policy $\pi^o$ is an artificial policy that takes a random action (chosen with equal probability from the three actions) $90\%$ of the time, and explores the state space by jumping to some random state $10\%$ of the time. It also restarts when it is at the destination: with equal probability, it either restarts near the bottom of the valley or restarts from a random point sampled uniformly from the state space.
\footnote{The behavior policy is exactly the same as described in \cite[p.\ 24-25]{etd-exp16} except for the possibility of restarting near the bottom of the valley whenever the destination is reached.}   

We now explain how we will measure approximation qualities. Since the Mountain Car problem has a continuous state space, it is actually not covered by our analysis, which is for finite-state problems. Although we can treat it as essentially a finite-state problem (since the simulation is done with finite precision in computers), the number of states would still be too large to calculate the weights $\zeta_{\S}$. So, to measure weighted approximation errors $\| v - v_{\pi} \|$ for approximate value functions $v$ produced by various LSTD algorithms in the subsequent experiments, we will compare $v$ and $v_{\pi}$ at a grid of points in the state space and calculate a weighted Euclidean distance between the function values at these grid points, using a set of weights precalculated by simulating the behavior policy.
\footnote{Specifically, we chose a grid of $171 \times 141$ points evenly spaced in the position-velocity space. We ran the behavior policy for $8 \times 10^5$ \emph{effective} iterations, where an iteration is considered to be \emph{ineffective} if the behavior policy takes an action, e.g., the restart action, that is impossible for the target policy. A visit to a state at an effective iteration was counted as a visit to the nearest grid point. At the end of the run, visits to a boundary point $(-1.2, 0)$ were disregarded as they were due to boundary effects in the dynamics of this problem, and the final counts were normalized to produce a set of weights on the grid points that sum to $1$. \label{footnote-weights}}
The image in Figure~\ref{fig-mcar2} (left) visualizes these weights.

\begin{figure}[t]
   \centering
   \includegraphics[width=0.328\linewidth]{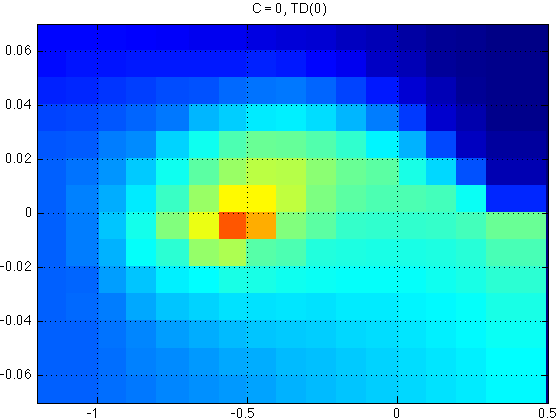} 
   \includegraphics[width=0.328\linewidth]{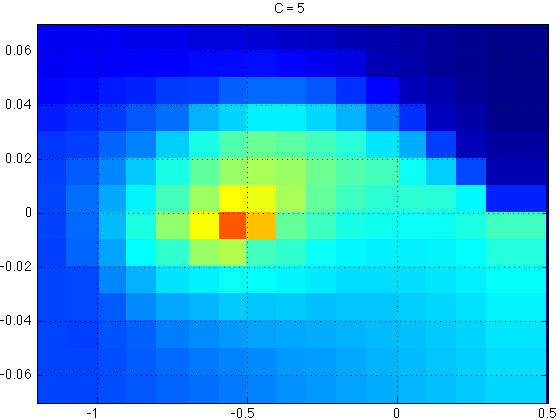}  
   \includegraphics[width=0.328\linewidth]{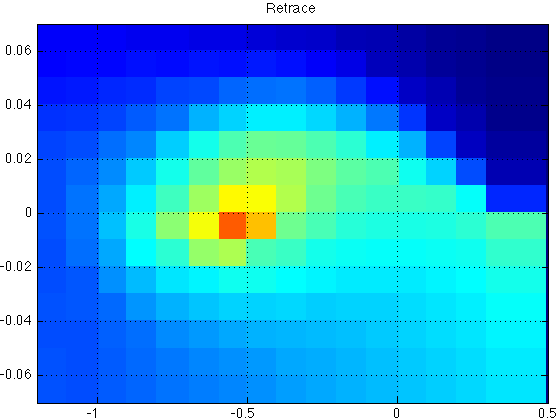} \\*[5pt]
   \includegraphics[width=0.328\linewidth]{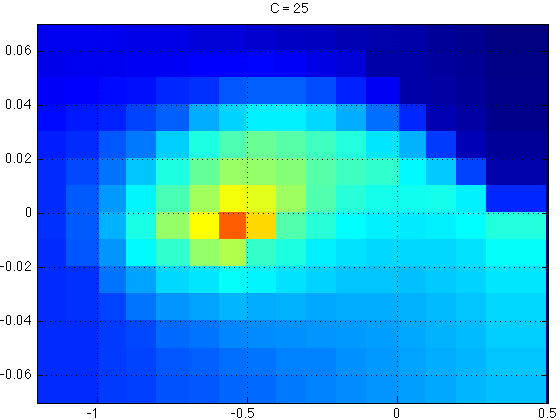} 
   \includegraphics[width=0.328\linewidth]{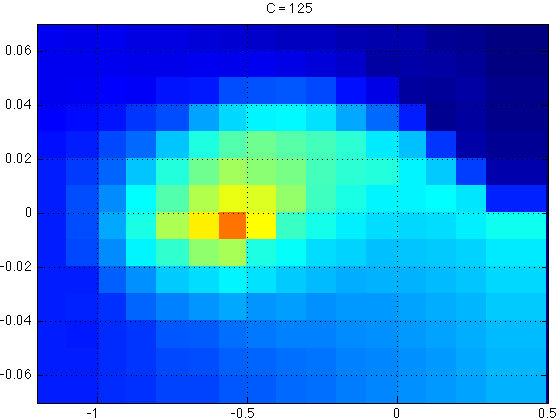}  
   \includegraphics[width=0.328\linewidth]{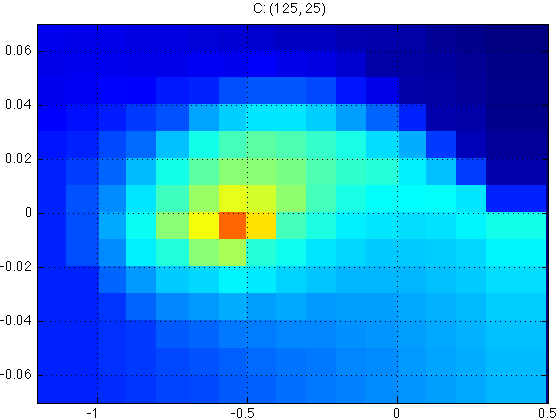}  \\*[5pt]
   \includegraphics[width=0.328\linewidth]{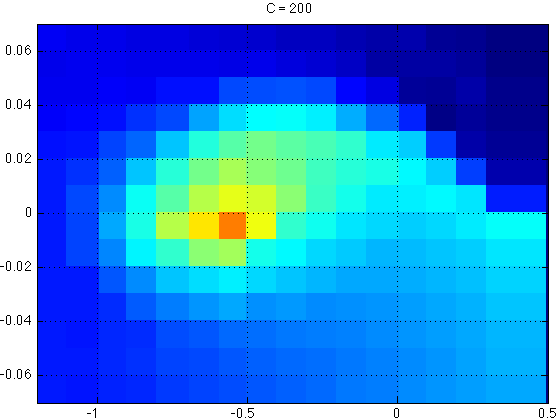} 
   \includegraphics[width=0.328\linewidth]{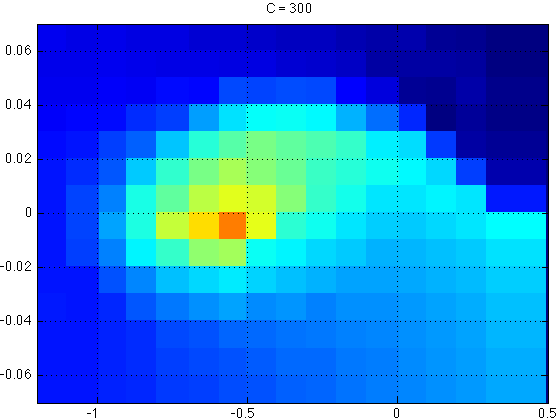}  
   \includegraphics[width=0.328\linewidth]{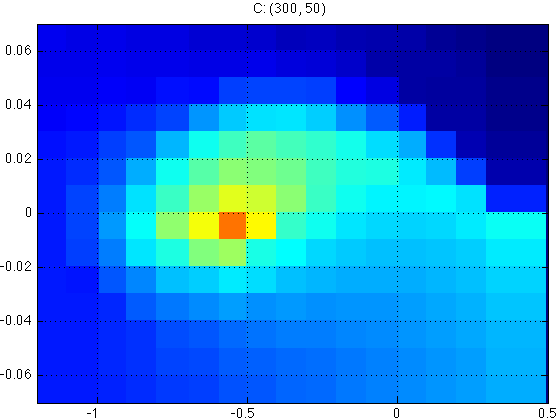}  
   \caption{Visualized in the color images are approximations of $- v_{\pi}$ obtained by LSTD with different schemes of setting $\lambda$, where the coloring scheme is as that shown in Figure~\ref{fig-mcar1} (right). The choices of $\lambda$ for each image, from left to right and top to bottom, are as follows. Top row: $C=0$ (equivalent to LSTD($0$)), $C=5$, Retrace. Middle row: $C=25$, $C=125$, $C\colon (125, 25)$. Bottom row: $C=200$, $C=300$, $C\colon (300, 50)$.} \label{fig-mcar3a}
\end{figure}%

We use tile-coding \citep{SUB} to generate $145$ binary features for our experiments.\footnote{Two tilings are used: the first (second) comprises of $64$ ($81$) uneven-sized rectangles that cover the state space. Together they produce a total of $145$ binary features. The details of the coding scheme are as described in \cite[p.\ 28]{etd-exp16}.} 
The approximate value functions obtained with LSTD algorithms are thus piecewise constant.
For comparison, we also build a discrete approximate model by state aggregation. The discretization is done at a resolution comparable to our tile-coding scheme, and the dynamics and rewards of this model are calculated based on data collected under the behavior policy. The solution of the discrete approximate model is shown in Figure~\ref{fig-mcar2} (right) (the coloring scheme for this and the subsequent images are the same as shown in Figure~\ref{fig-mcar1}). It is similar to the approximate value function calculated by LSTD($0$), which is shown in Figure~\ref{fig-mcar3a} (first image, top row). As will be seen shortly, with positive $\lambda$, the approximation quality of LSTD improves. Thus the discrete model approximation approach is not as effective as the TD method in this case.

We now report the results of our experiments on the Mountain Car problem. 

\smallskip
\noindent {\bf First Experiment:} In this experiment, we compare three ways of setting $\lambda$: (i) Retrace with $\beta=1$ (cf.\ Example~\ref{ex-retrace}); (ii) our simple scaling scheme with parameter $C$ used in the previous experiments (cf.\ (\ref{eq-ex1c}) and Example~\ref{ex-scaling}); and (iii) a composite scheme of the type discussed at the end of Section~\ref{sec-3.2}, which partitions the state space into two sets
\footnote{The first set consists of those states (position, velocity) with either $\text{position} \leq -0.9$ or $\text{velocity} \geq 0.04$. The rest of the states belong to the second set.}
and applies the simple scaling scheme with parameters $C_1, C_2$ for the first and second set, respectively. When referring to this composite scheme in the figures, we will use the designation $C\colon (C_1, C_2)$.

We ran LSTD with different ways of setting $\lambda$ just mentioned, on the same state trajectory generated by the behavior policy, for $6 \times 10^5$ effective iterations (cf.\ Footnote~\ref{footnote-weights}). Some of the approximate value functions obtained at the end of the run are visualized as images in Figure~\ref{fig-mcar3a}. It can be seen that the result of Retrace is similar to that of the simple scaling scheme with a small $C$, and as we increase $C$, the approximation from the scaling scheme improves.

To compare more precisely the approximation errors and see how they change over time for each algorithm, we did $10$ independent runs, each of which consists of $6 \times 10^5$ effective iterations. The results are shown in Figures~\ref{fig-mcar3b}-\ref{fig-mcar4b}. Plotted in Figure~\ref{fig-mcar3b} for each algorithm are the mean and standard deviation of the approximation errors for the $10$ approximate value functions obtained by that algorithm at the end of the $10$ runs.  We can see from this figure the improvement in approximation quality as $C$ increases. We can also see that the result of Retrace is in between those of $C=0$ and $C=5$, which is consistent with what the images in the top row of Figure~\ref{fig-mcar3a} tell us.

\begin{figure}[t]
\centering
   \includegraphics[width=1\linewidth]{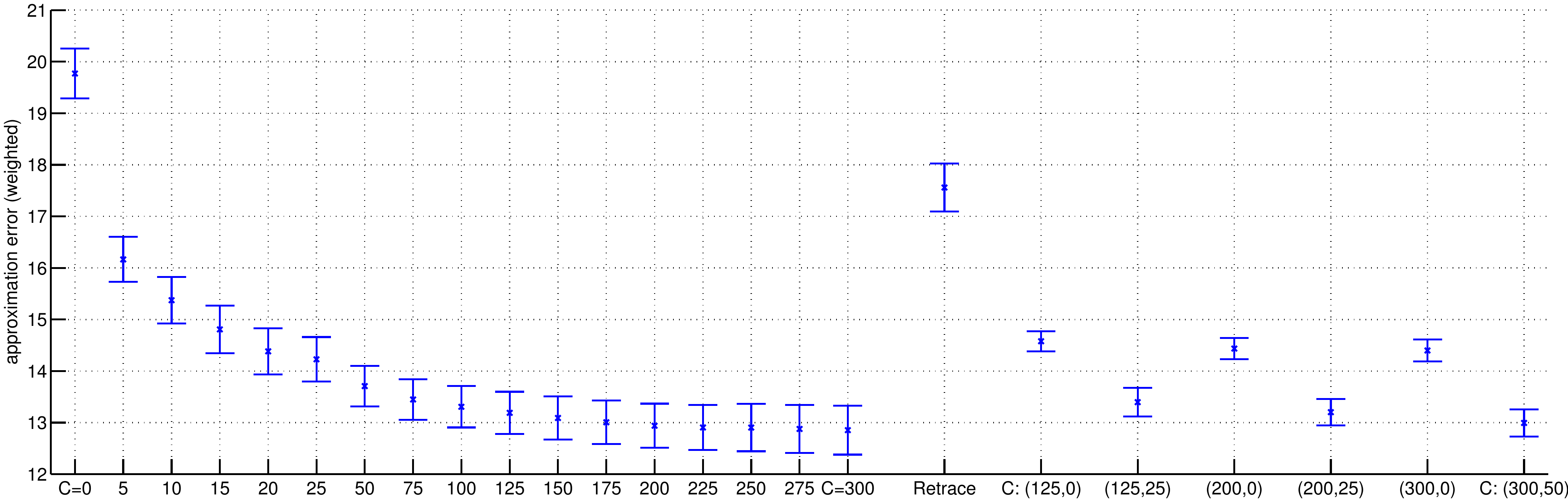}
\caption{Compare the approximation error of LSTD for different schemes of setting $\lambda$.} \label{fig-mcar3b}
\end{figure}%
\begin{figure}[!h]
   \centering
   \includegraphics[width=0.46\linewidth]{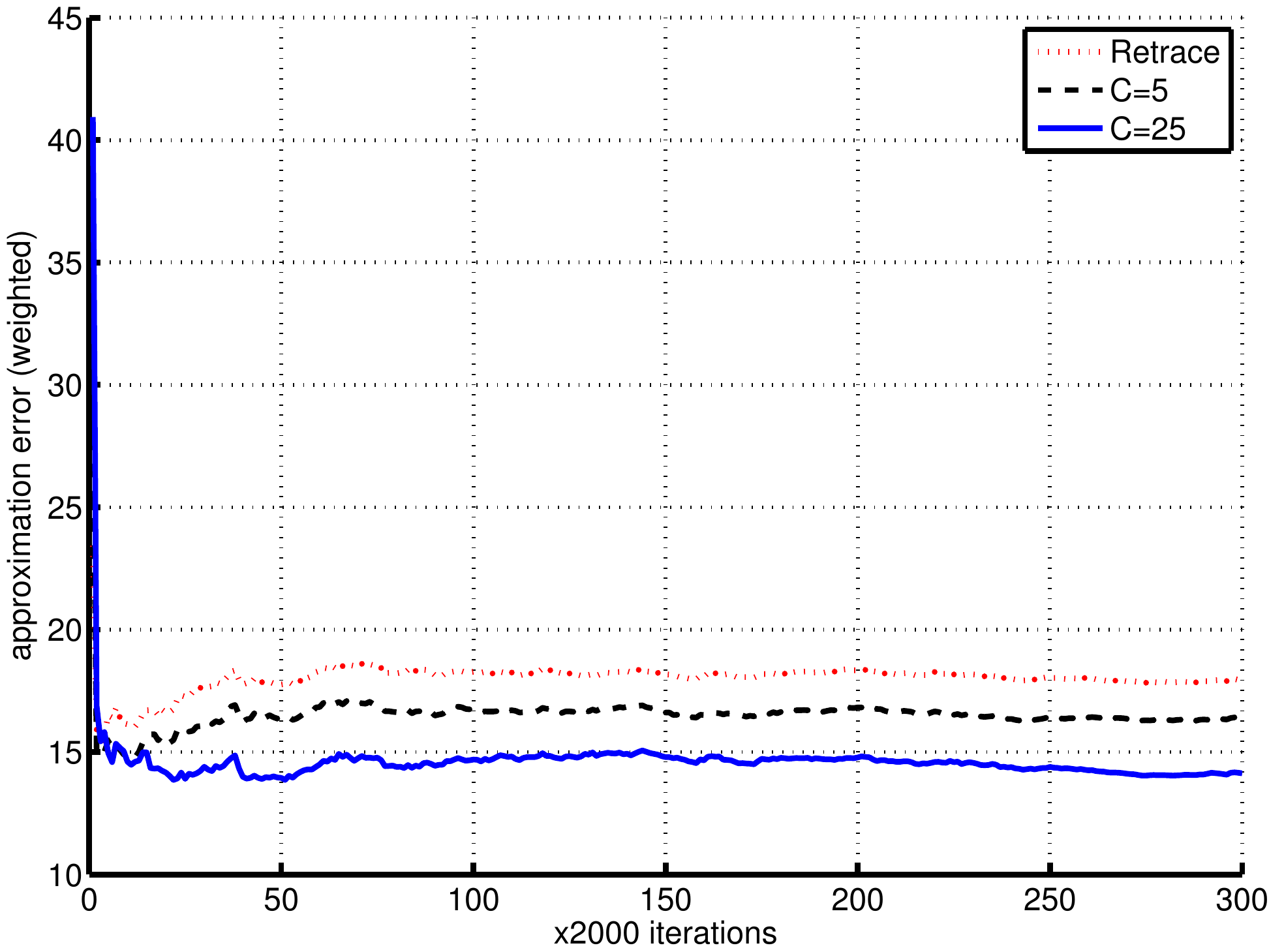} \quad
   \includegraphics[width=0.46\linewidth]{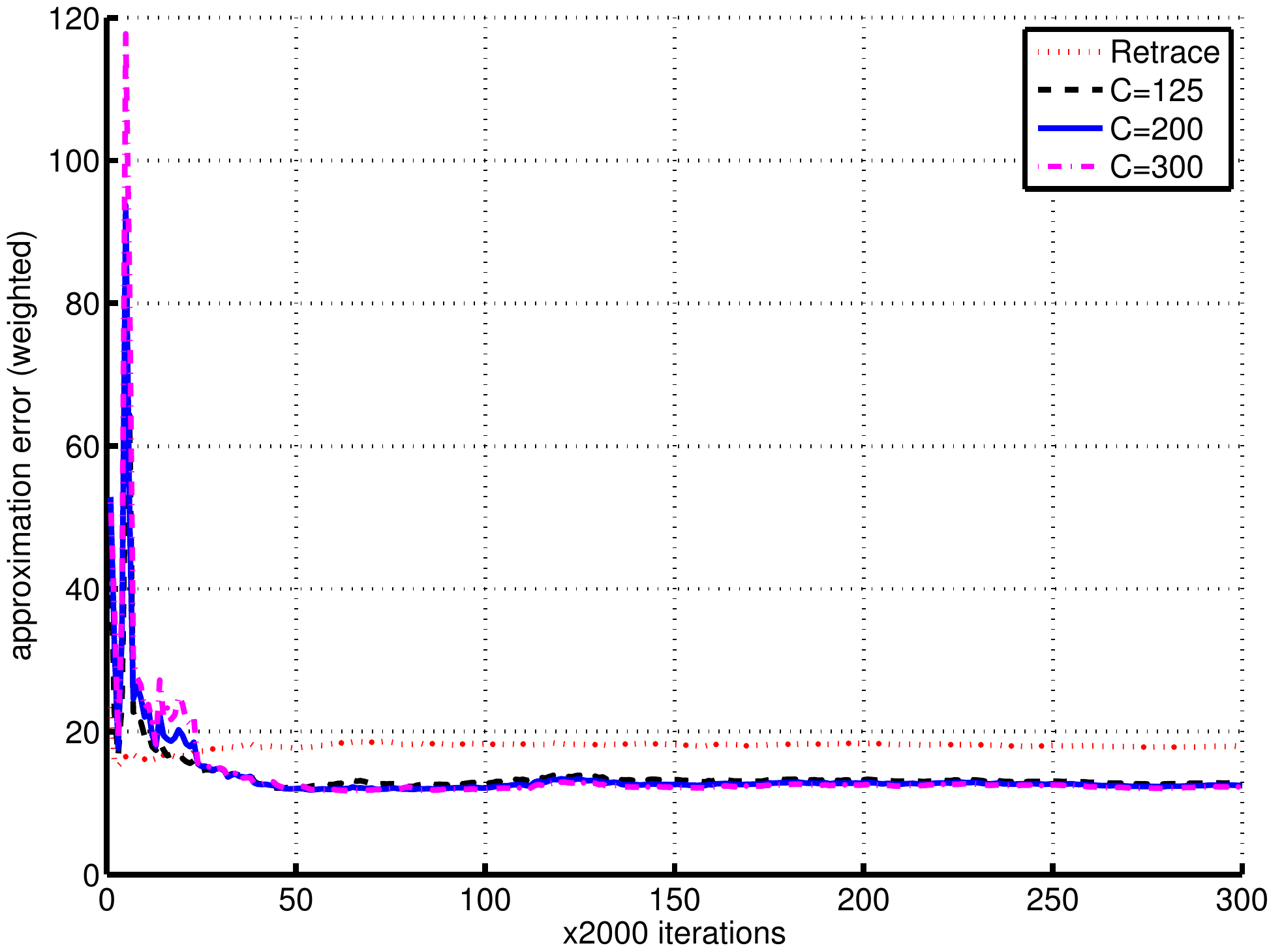}\vspace*{-0.2cm}
   \caption{Compare the temporal behavior of LSTD for different schemes of setting $\lambda$.} \label{fig-mcar4a}
\end{figure}%
\begin{figure}[!h]
   \centering
   \includegraphics[width=0.3285\linewidth]{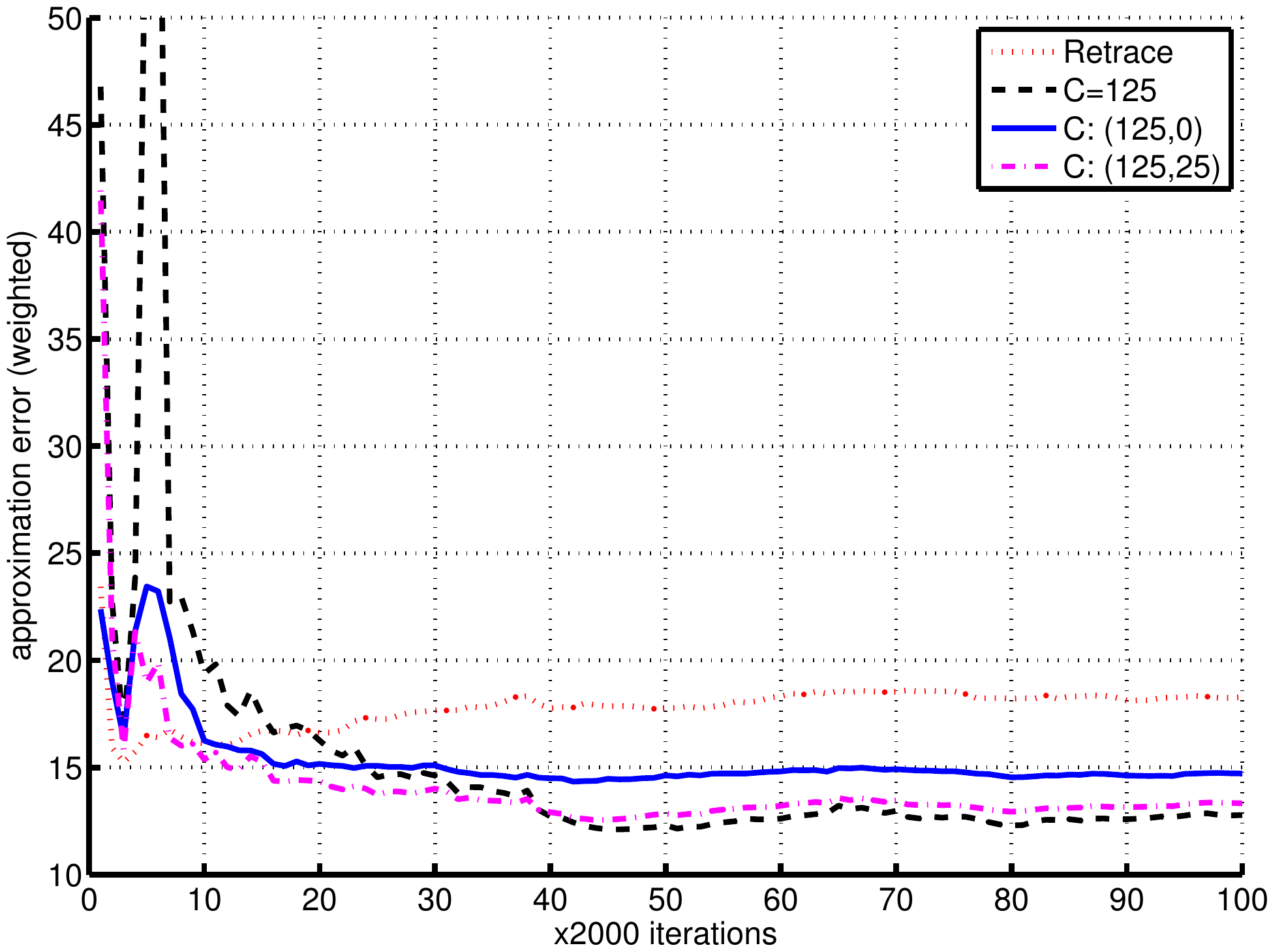} 
   \includegraphics[width=0.3285\linewidth]{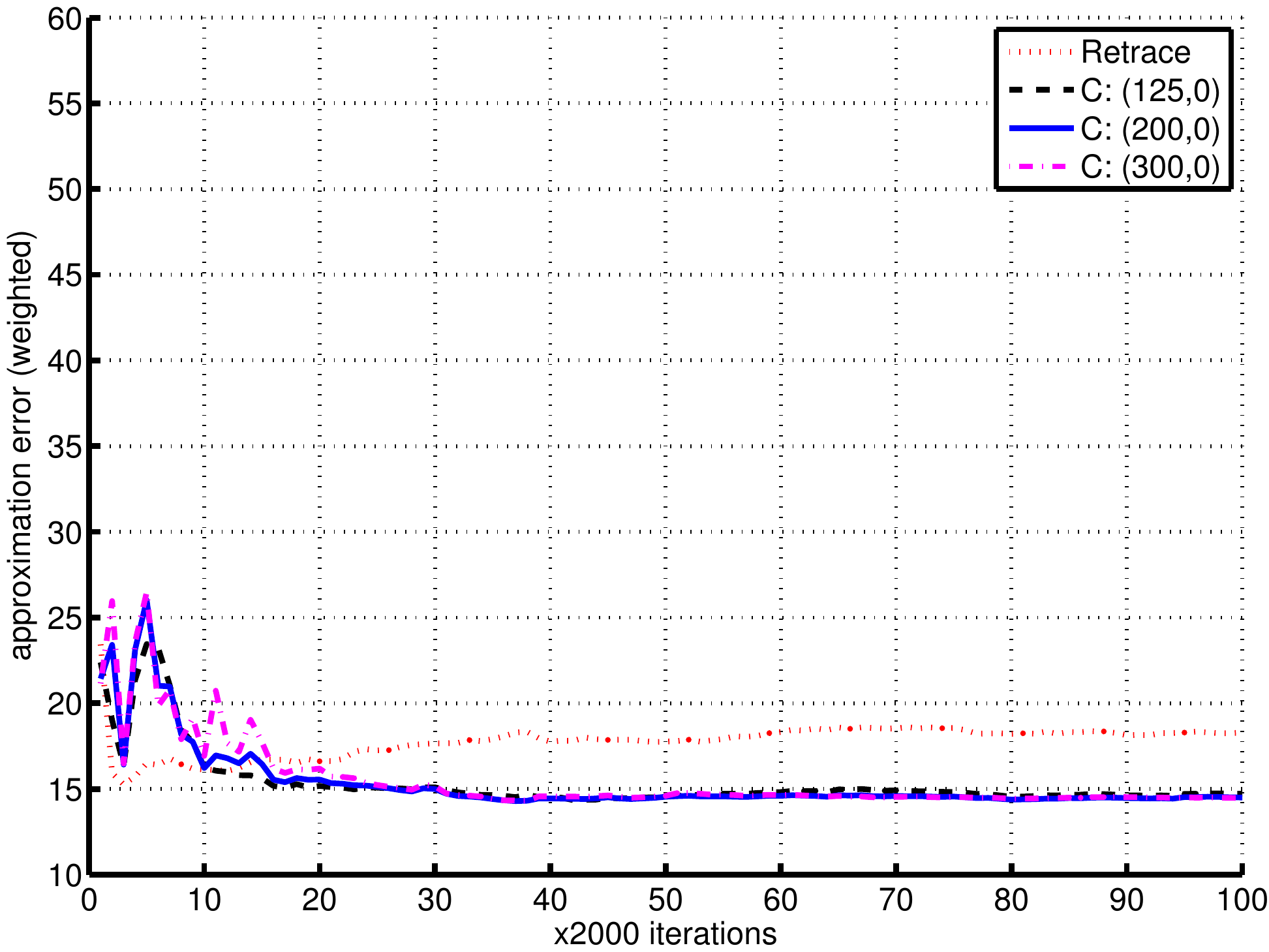}  
    \includegraphics[width=0.3285\linewidth]{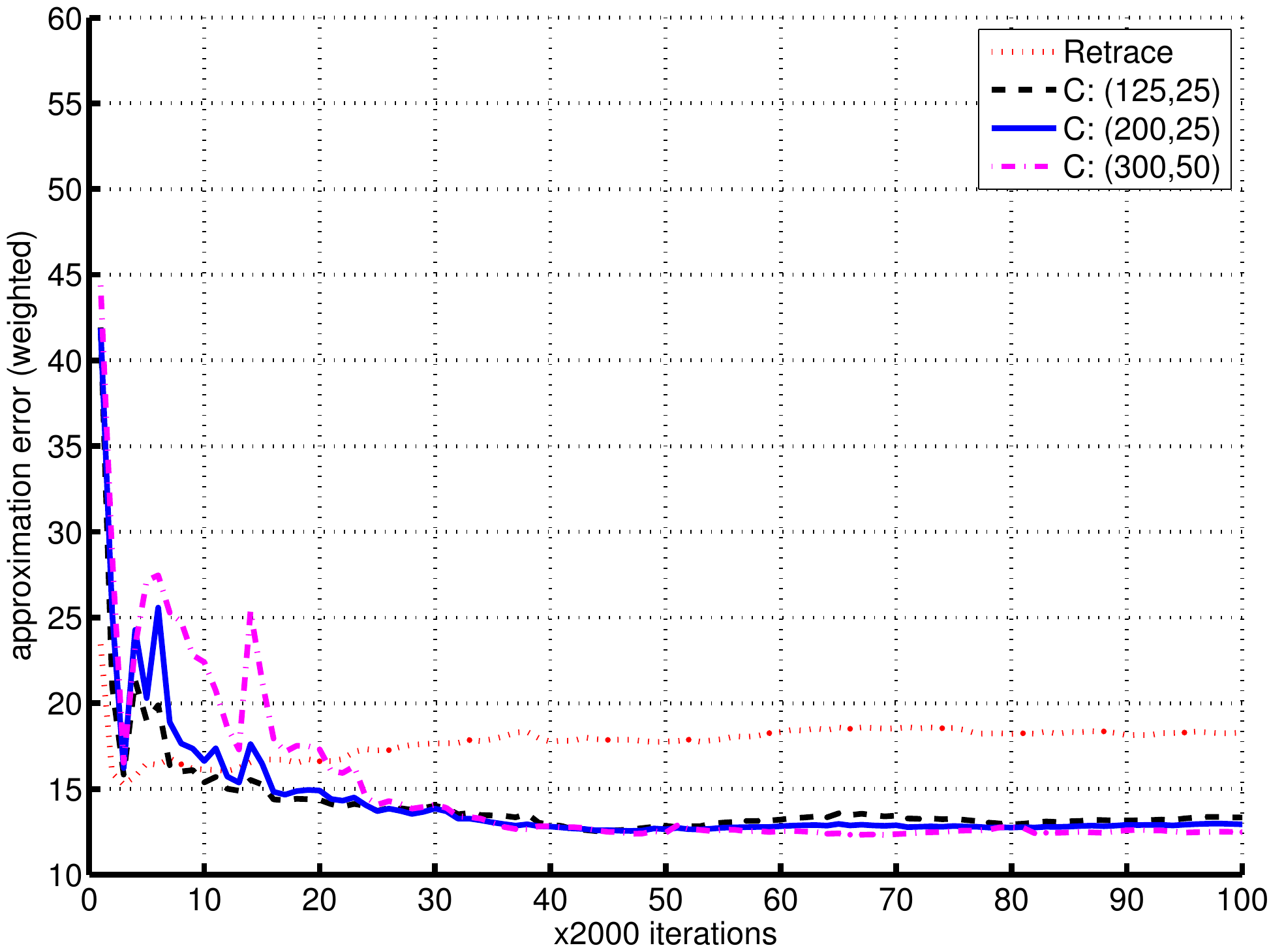}\vspace*{-0.2cm}
   \caption{Compare the temporal behavior of LSTD for different schemes of setting $\lambda$.} \label{fig-mcar4b}
\end{figure}%

Plotted in Figures~\ref{fig-mcar4a}-\ref{fig-mcar4b} are the approximation errors calculated per $2000$ effective iterations for each algorithm, during one of the $10$ experimental runs. 
Figure~\ref{fig-mcar4a} (left) shows how Retrace compares with the simple scaling with $C=5$ and $C=25$. It can be seen that the latter two achieved better approximation quality than Retrace without increases in variance. Figure~\ref{fig-mcar4a} (right) shows that for larger values of $C$, variances also became larger initially; however, after about $5 \times 10^4$ iterations, these schemes overtook Retrace, yielding better approximations. 

Figure~\ref{fig-mcar4b} shows how the composite scheme of setting $\lambda$ performs. 
Comparing the plots in this figure with the right plot in Figure~\ref{fig-mcar4a}, it can be seen that the composite schemes helped in reducing variances, and in about $2 \times 10^4$ iterations the schemes $C\colon(125,0)$ and $C\colon(125,25)$ overtook Retrace and yielded better approximations. Together with Figure~\ref{fig-mcar3b}, Figure~\ref{fig-mcar4b} shows clearly the bias-variance trade-off of using composite schemes in this problem.

\medskip
\noindent {\bf Second Experiment:} Similarly to the previous experiment, we now compare our proposed method with several other ways of setting $\lambda$ for the LSTD algorithm as well as with a constrained variant of LSTD: (i) Retrace with $\beta=1$ as before; (ii) constant $\lambda$; (iii) constrained LSTD with constant $\lambda$; 
and (iv) the simple scaling scheme with parameter $C$. For constant $\lambda \in [0,1]$, the constrained LSTD($\lambda$) used in this experiment evolves the trace vectors as off-policy LSTD($\lambda$) does, but it forms and solves the linear equation $\tfrac{1}{t} \sum_{k=0}^{t-1}[ \e_k ]_{50} \cdot \delta_k (v) = 0, v = \Phi \theta$ instead, where the function $[\cdot]_{50}$ truncates each component of the trace vector to be within the interval $[-50, 50]$. (Such an algorithm follows naturally from the ergodicity of the state-trace process and the approximation of an unbounded integrable function by a bounded one. For a detailed discussion, see \cite[Section 3.2]{etd-wkconv} or \cite[Section 3.3]{gtd-conv17}.)

\begin{figure}[thb]
\centering
   \vspace*{-0.3cm}\hspace*{-2cm} \includegraphics[width=1.24\linewidth]{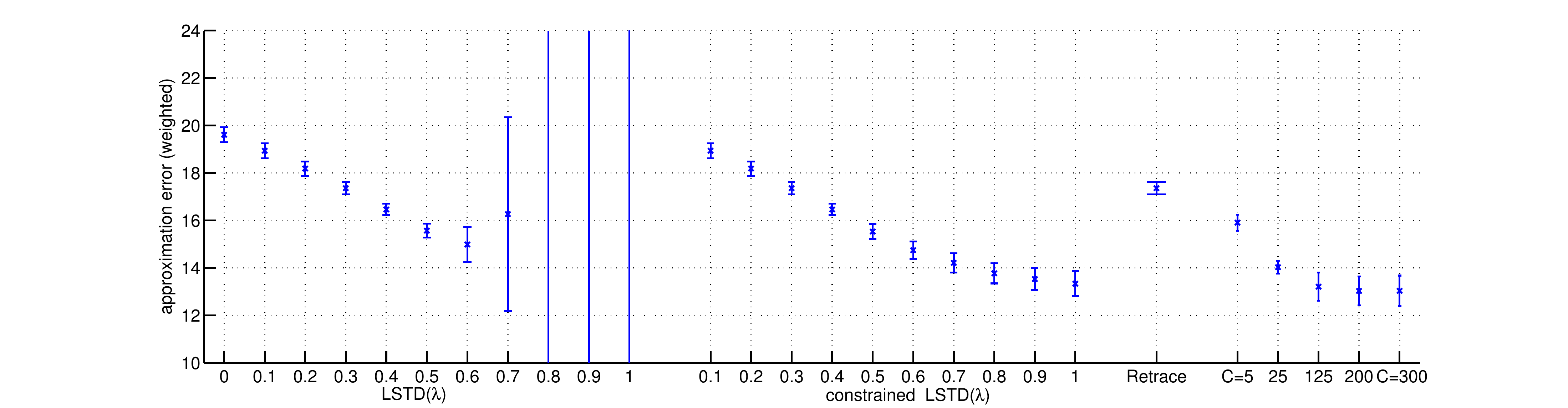}\vspace*{-0.3cm}
\caption{Compare the approximation errors of several LSTD algorithms.} \label{fig-mcar5}
\end{figure}%
\begin{figure}[thb]
   \centering
   \includegraphics[width=0.46\linewidth]{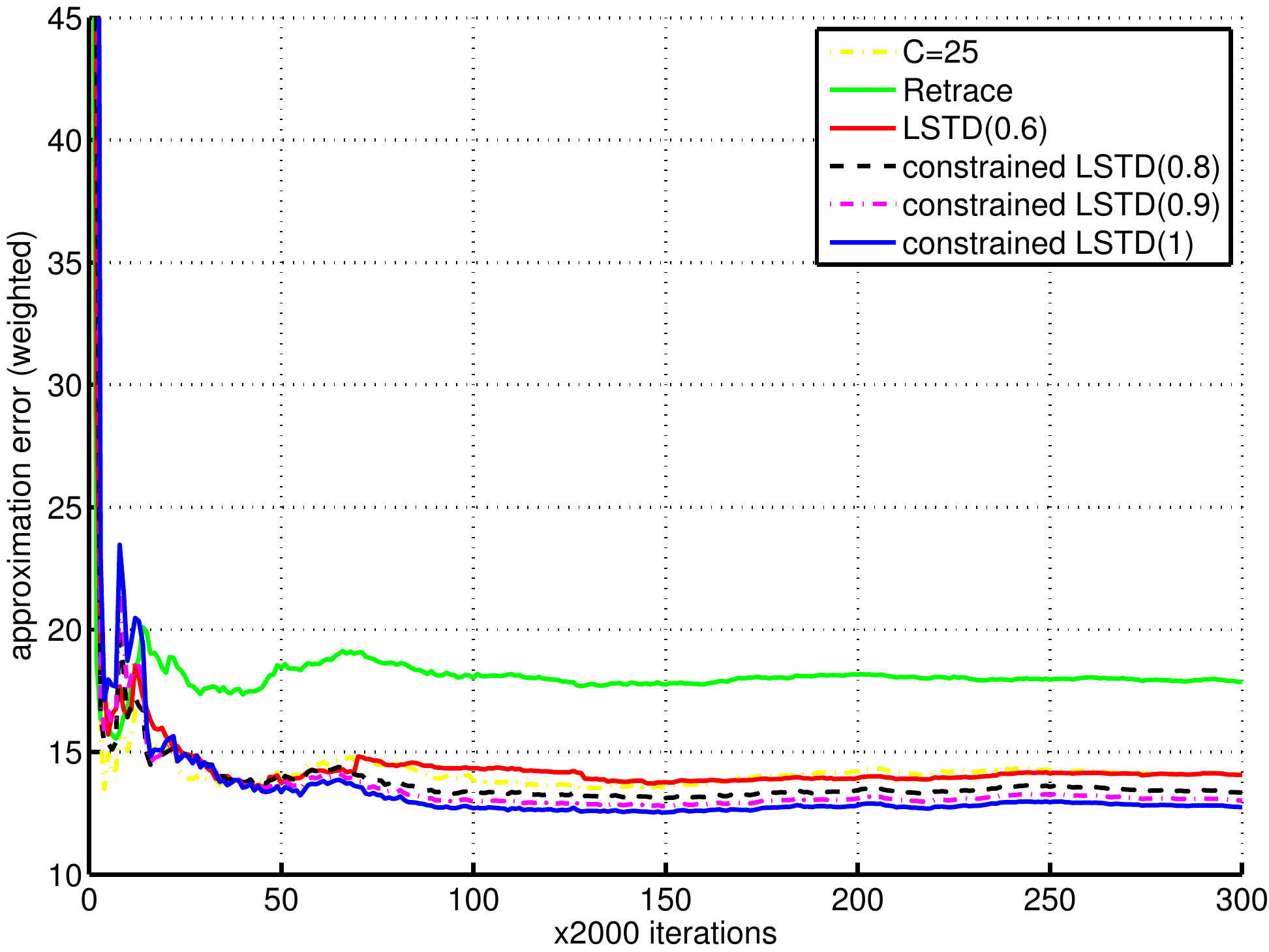} \qquad
   \includegraphics[width=0.46\linewidth]{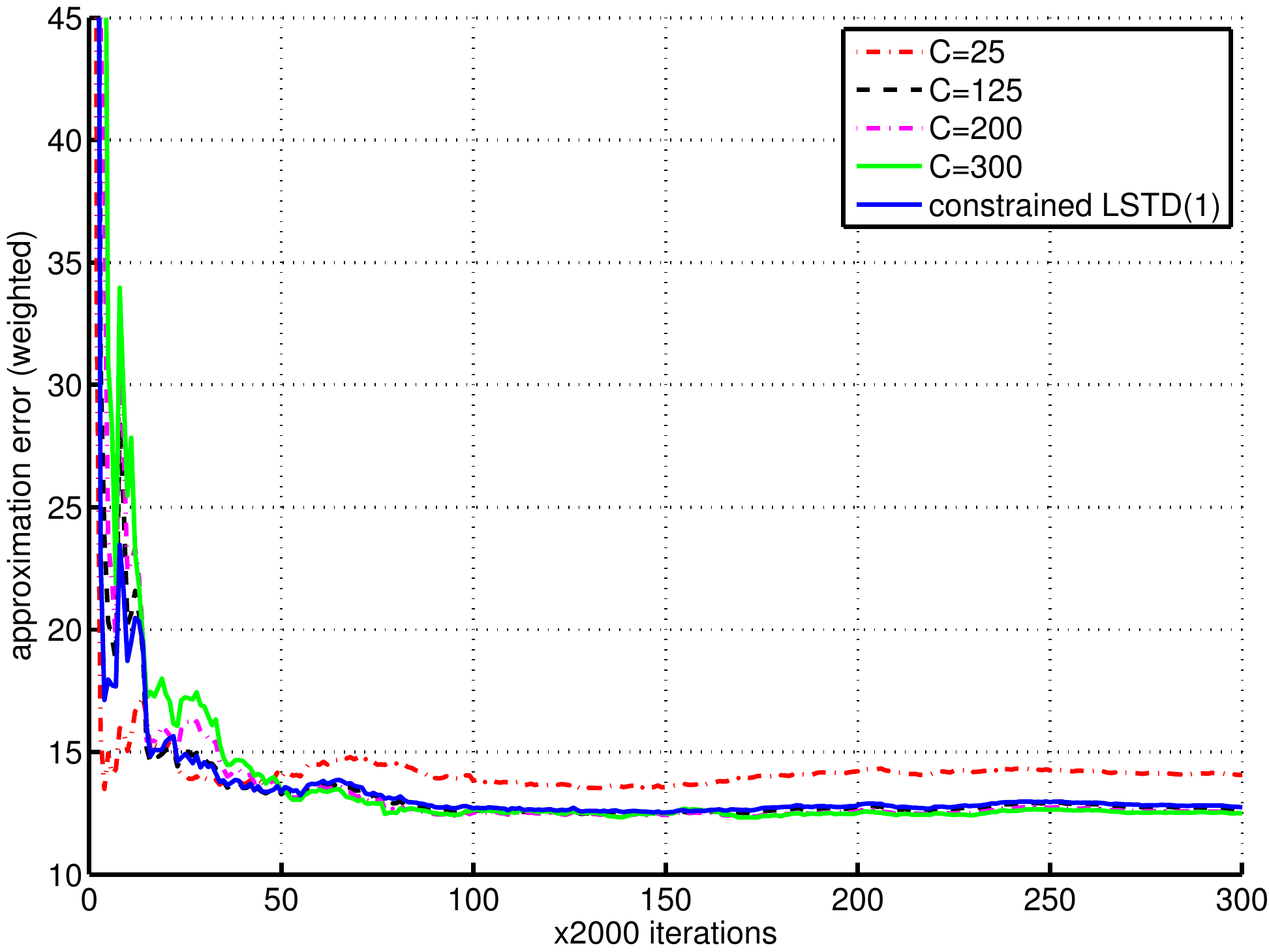} 
   \caption{Compare the temporal behavior of several LSTD algorithms.} \label{fig-mcar5b}
\end{figure}%

Figure~\ref{fig-mcar5} is similar to Figure~\ref{fig-mcar3b} and shows, for each algorithm, the mean and standard deviation of the approximation errors of $10$ approximate value functions obtained at the end of $10$ independent experimental runs (each of which consists of $6 \times 10^5$ effective iterations). The horizontal axis indicates the algorithms and their parameters.
As can be seen from this figure, for constant small $\lambda$, LSTD($\lambda$) performed well in this problem, and LSTD($0.3$) and Retrace are comparable.
For constant $\lambda > 0.7$, LSTD($\lambda$) failed to give sensible results, and LSTD($0.7$) started to show this unreliable behavior. This behavior of LSTD($\lambda$) is related to what we observed in Figure~\ref{fig-cmpsol}(right) in the small toy problem, and it is, we think, due to the high variance issue, which becomes more severe as $\lambda$ gets larger.
Constrained LSTD($\lambda$) is much more reliable, and it did consistently well for all values of $\lambda$ tested. LSTD with evolving $\lambda$ also did well, and with $C > 125$, it achieved a slightly better approximation quality than constrained LSTD($1$).

Figure~\ref{fig-mcar5b} compares the temporal behavior of several algorithms during one experimental run. Plotted are the approximation errors calculated per $2000$ iterations for each algorithm in the comparison. It can be seen that in this problem constrained LSTD($\lambda$) did not suffer from large variances even with large $\lambda$ values, and compared with constrained LSTD($1$), the behavior of LSTD with evolving $\lambda$ was also reasonable for large $C$.

\medskip
\noindent {\bf Third Experiment:} In this experiment we first compare the simple scaling scheme and Retrace, where both schemes now use an additional parameter $\beta \in [0, 1]$, as discussed in Examples~\ref{ex-scaling}-\ref{ex-retrace}. Recall that when $\beta=1$, they reduce to the schemes that we already compared in the previous experiments. For both schemes, as $\beta$ becomes smaller, we expect the approximation quality to drop but the variance to get smaller.

\begin{figure}[!t]
   \centering
   \includegraphics[width=0.3285\linewidth]{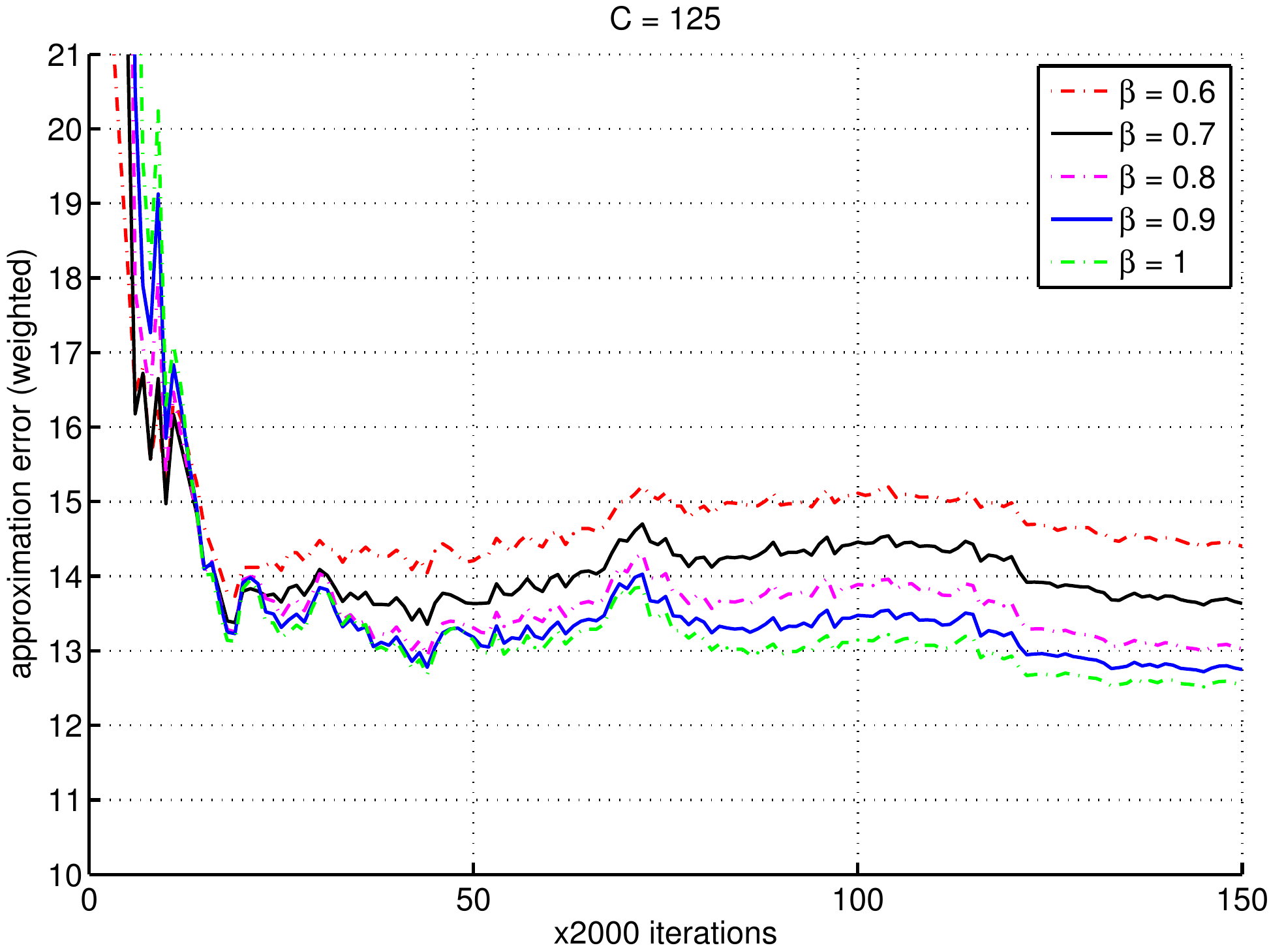} 
   \includegraphics[width=0.3285\linewidth]{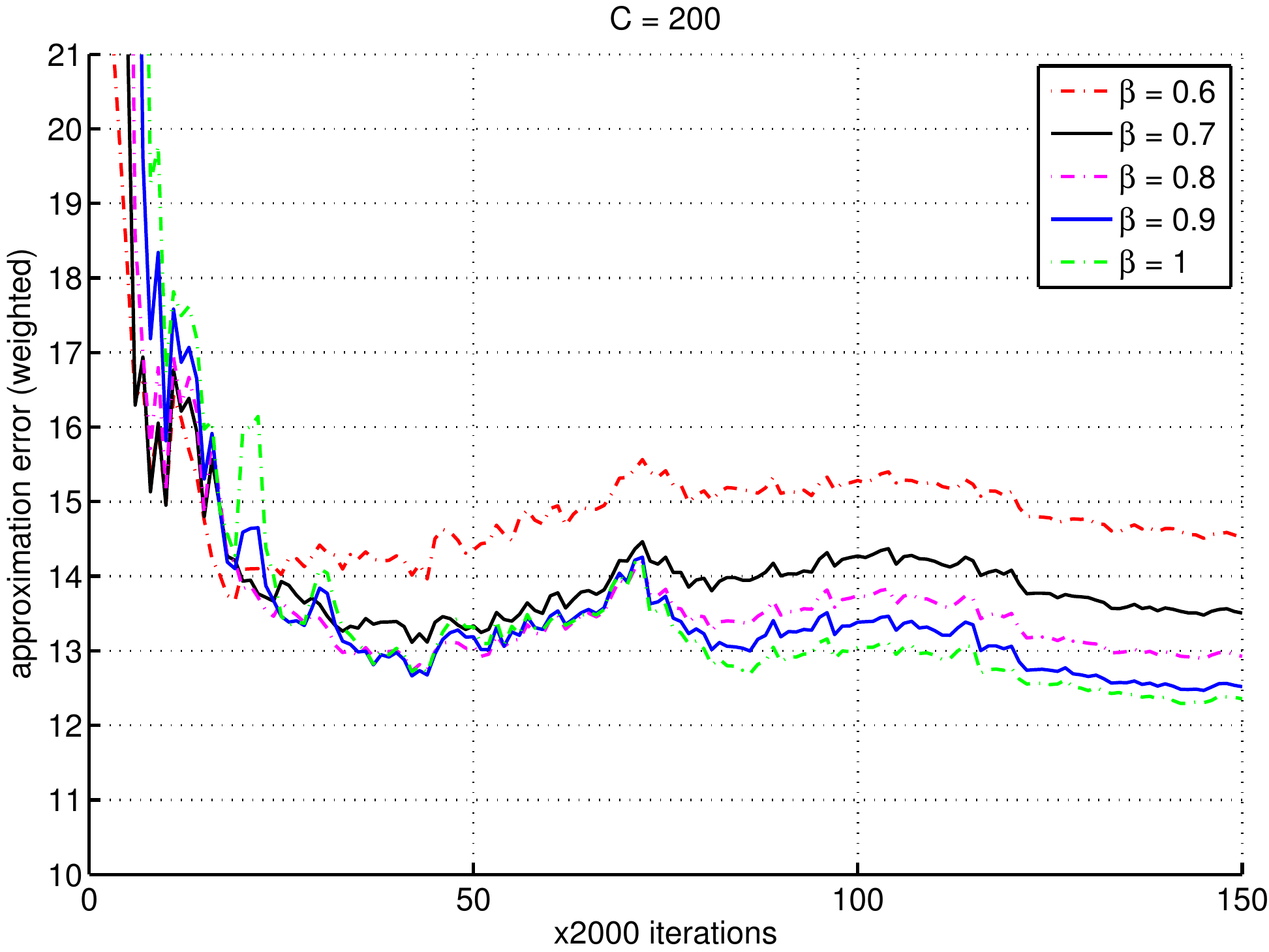}  
   \includegraphics[width=0.3285\linewidth]{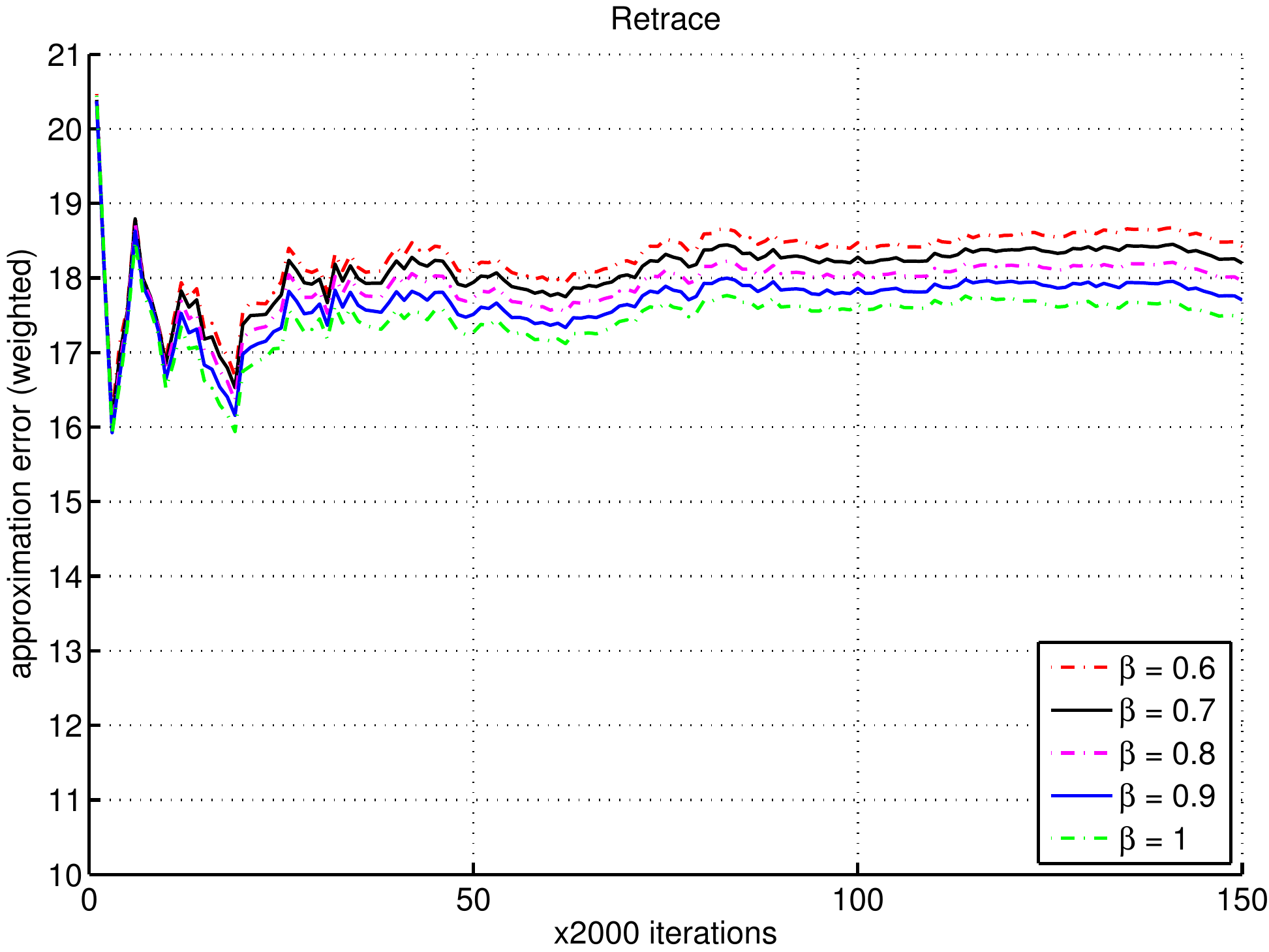}  
   \caption{Compare the approximation error and temporal behavior of LSTD for different schemes of setting $\lambda$. From left to right: $C=125$, $C=200$, Retrace.} \label{fig-mcar6}
\end{figure}%
\begin{figure}[!h]
   \centering
   \includegraphics[width=0.46\linewidth]{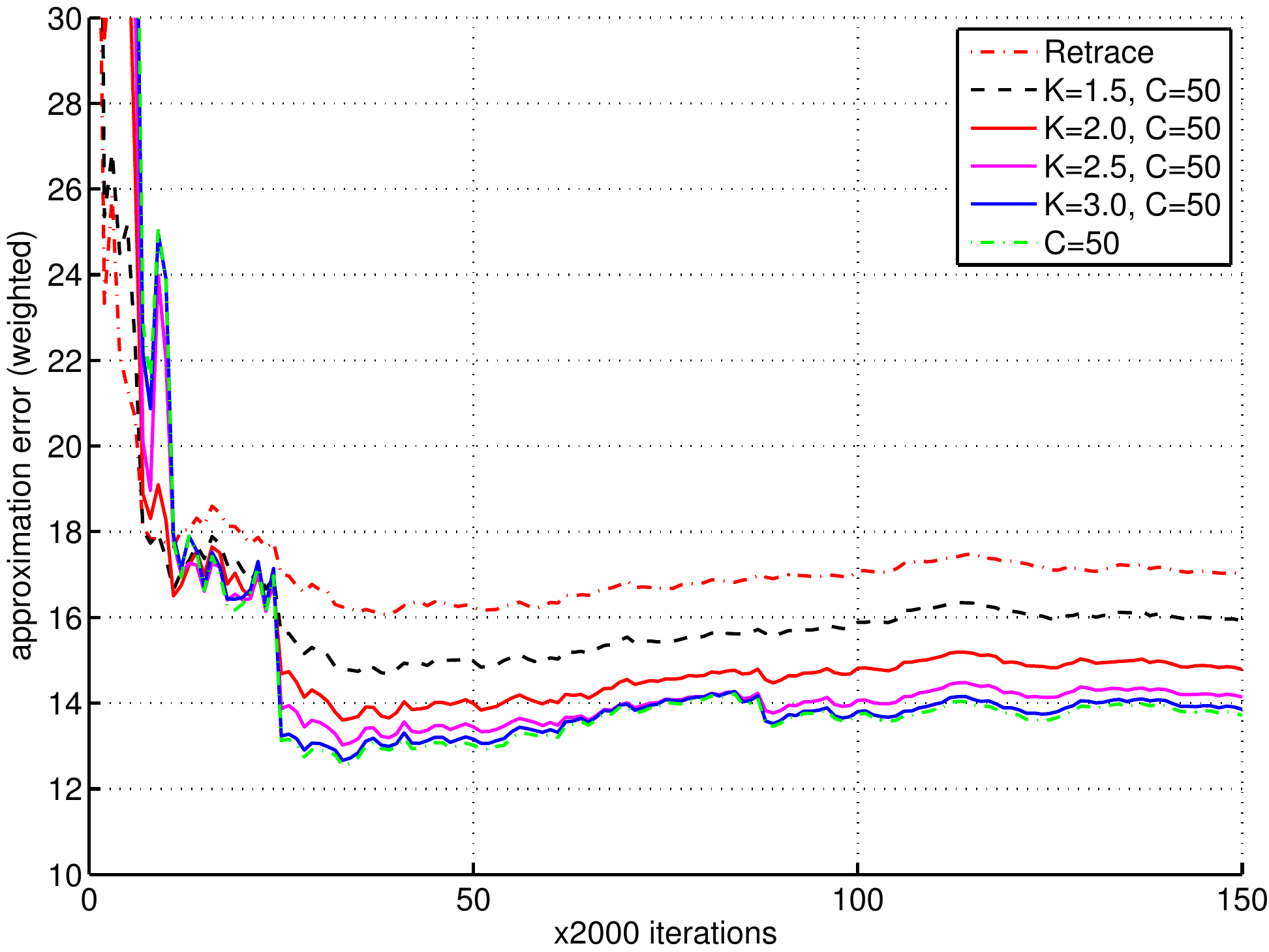} \qquad
   \includegraphics[width=0.46\linewidth]{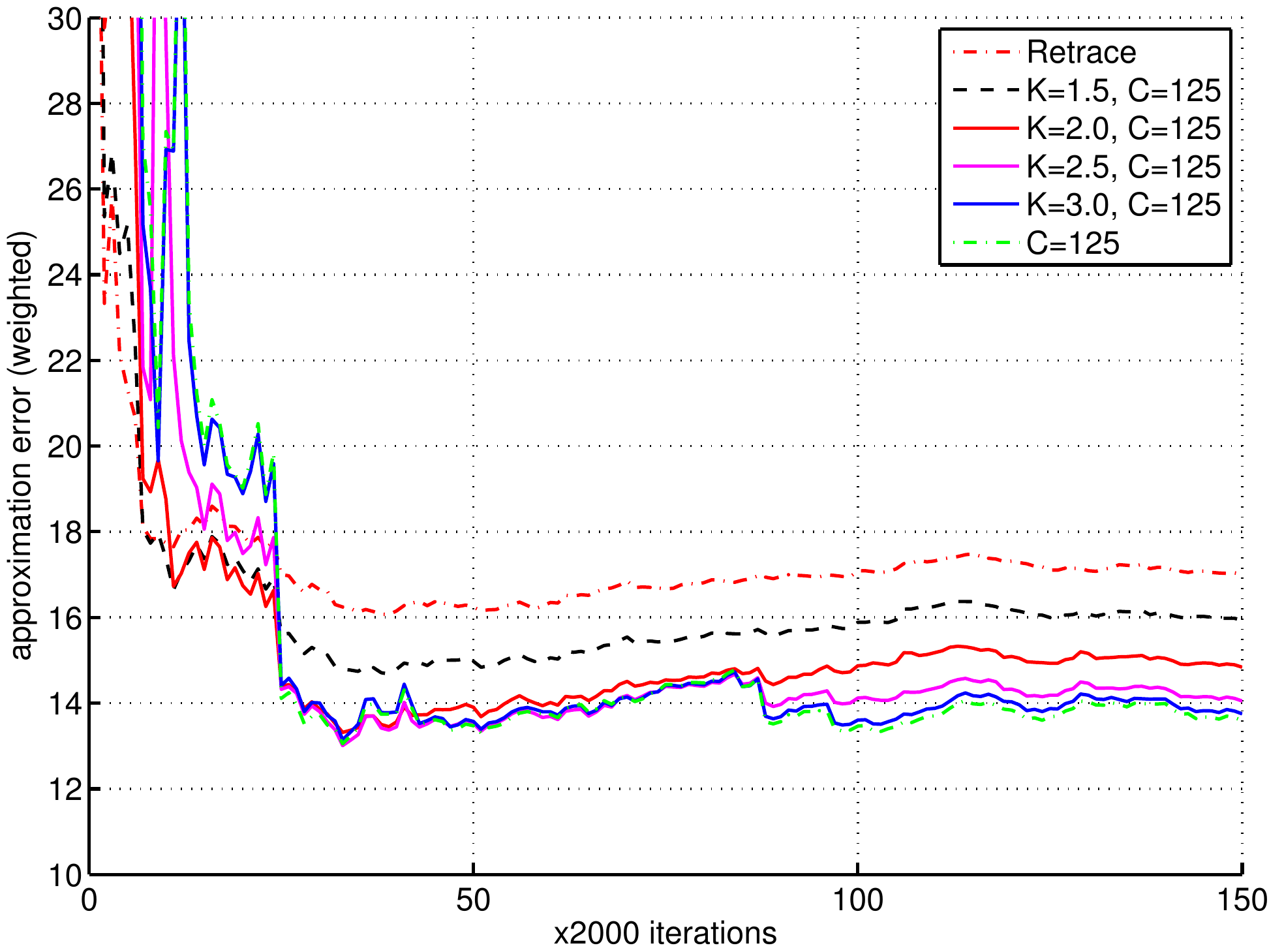} 
   \caption{Compare the approximation error and temporal behavior for some variations on Retrace.} \label{fig-mcar7}
\end{figure}%

Plotted in Figure~\ref{fig-mcar6} for $C=125, C=200$ and Retrace are the results obtained from a single experimental run consisting of $3 \times 10^5$ effective iterations. As before the approximation errors were calculated per $2000$ iterations. The results do show the expected bias-variance trade-off effects of the parameter $\beta$, during the initial period of the experimental run, although the effects on Retrace turned out to be smaller and hard to discern at the scale of the plot.

Next we test some of the variations on Retrace discussed in Example~\ref{ex-retrace}. Specifically, we consider (\ref{eq-var-retrace1})-(\ref{eq-var-retrace2}) with parameters $\beta=0.9$, $K \in \{1.5, 2.0, 2.5, 3.0\}$ and $C \in \{50, 125\}$. Plotted in Figure~\ref{fig-mcar7} are the results from one experimental run of $3 \times 10^5$ effective iterations. For comparison, the plots also show the behavior of Retrace and the simple scaling scheme with the same values of $C$ during that run (these algorithms used the same $\beta=0.9$). As expected and can be seen from the figure, the approximation quality improves with $K$ and $C$. While the variances also tend to increase during the initial part of the run, the variants that truncate the importance sampling ratios by $K=1.5$ performed comparably to Retrace initially, soon overtook Retrace and achieved better approximation quality.

\section{Conclusion} \label{sec-discussion}

We developed in this paper a new scheme of setting the $\lambda$-parameters for off-policy TD learning, using the ideas of randomized stopping times and generalized Bellman equations for MDPs. Like the two recently proposed algorithms Retrace \citep{offpolicytd-mshb} and ABQ \citep{abq}, our scheme keeps the traces bounded to reduce variances, but it is much more general and flexible. To study its theoretical properties, we analyzed the resulting state-trace process and established convergence and solution properties for the associated LSTD algorithm, and these results have prepared the ground for convergence analysis of the gradient-based implementation of our proposed scheme \citep{gtd-conv17}. In addition we did a preliminary numerical study. It showed that with the proposed scheme LSTD can outperform several existing off-policy LSTD algorithms. It also demonstrated that in order to achieve better bias-variance trade-offs in off-policy learning, it is helpful to have more flexibility in choosing the $\lambda$-parameters and to allow for large $\lambda$ values. 
Future research is to conduct a more extensive numerical study of both least-squares based and gradient-based algorithms, with more versatile ways of using the memory states and $\lambda$-parameters, in off-policy learning applications.

\section*{Acknowledgements}
An abridged version of this article appeared first at \emph{The $30$th Canadian Conference on Artificial Intelligence (CAI)}, May 2017, and a reformatted copy is to appear in \emph{Journal on Machine Learning Research (JMLR)}, 2018.
We thank several anonymous reviewers for CAI and JMLR for their helpful feedback. 
We also thank Dr.\ Ajin Joseph and Dr.\ Martha Steenstrup for reading parts of our manuscript and providing comments that helped improve our presentation. 
This research was supported by a grant from Alberta Innovates---Technology Futures. 


\bigskip
\bigskip

{
\addcontentsline{toc}{section}{Appendices}
\appendix
\appendixpage
\renewcommand{\setthesection}{\appendixname \Alph{section}} 

\section{Proof of Theorem \ref{thm-gbo}} \label{appsec-1}

We prove Theorem~\ref{thm-gbo} in this appendix. Recall from Section~\ref{sec-3.1} that the generalized Bellman operator $T$ associated with a randomized stopping time $\tau$ satisfies that $v_{\pi} = T v_{\pi}$.
By (\ref{eq-gbe1}), the substochastic matrix $\tilde P$ in this affine operator $T$ is given by
\begin{equation} \label{gbo-prf1}
   \tilde P_{ss'} = \E^\pi_s \big[  \gamma_1^{\tau} \, \I(S_{\tau} = s') \big], \qquad   s, s' \in \S.
\end{equation}
We can extend $\tilde P$ to a transition matrix $\tilde P^e$ by adding an additional absorbing state $\Delta$ to the system, so that $\tilde P^e_{\Delta\Delta} = 1$ and
\begin{equation}
    \tilde P^e_{s\Delta} = 1 - \textstyle{\sum_{s' \in \S} \tilde P_{ss'} } = 1 - \E^\pi_s \big[  \gamma_1^{\tau} \big], \qquad s \in \S. \notag
\end{equation}    

Both conclusions of Theorem~\ref{thm-gbo} will follow immediately if we show that $I - \tilde P$ is invertible. 
Indeed, if $I - \tilde P$ is invertible, then $v = Tv$ has a unique solution, which must be $v_\pi$ since $v_\pi$ is always a solution of this equation.
In addition, if $I - \tilde P$ is invertible, then since $\tilde P$ is a substochastic matrix, the spectral radius of $\tilde P$ must be less than $1$. 
Consider adding to $\tilde P$ a small enough perturbation $\epsilon M$, where $M$ is the matrix of all ones and $\epsilon$ is a sufficiently small positive number so that the spectral radius of $\tilde P + \epsilon M$ is less than $1$. 
Applying \cite[Theorem 1.1]{Sen06} to the nonnegative primitive matrix $\tilde P + \epsilon M$, we have that $(\tilde P + \epsilon M) w < w$ where $w$ is a positive eigenvector of the matrix $\tilde P + \epsilon M$ corresponding to a positive eigenvalue that is strictly less than $1$.
Consequently $ \tilde P w \leq  (\tilde P + \epsilon M) w < w$, implying that $\tilde P$ is a linear contraction w.r.t.\ a weighted sup-norm (with weights $w$), which is the second conclusion of the theorem.

Hence, to prove Theorem~\ref{thm-gbo}, it suffices to show that the inverse $(I - \tilde P)^{-1}$ exists, which is equivalent to that for the Markov chain on $\S \cup \{\Delta\}$ with transition matrix $\tilde P^e$, all the states in $\S$ are transient (see e.g., \citealt[Appendix A.4]{puterman94}).

We prove this by contradiction. Suppose it is not true, and let $\tilde \S \subset \S$ be a recurrent class of the Markov chain. 
Then for all $s \in \tilde \S$, $\tilde P^\e_{ss'} = 0$ for $s' \not\in \tilde \S$ (i.e., the submatrix of $\tilde \P$ corresponding to $\tilde \S$ is a transition matrix). In particular, $\tilde P^e_{s\Delta} = 0$,
so
\begin{equation} \label{gbo-prfa}
  \E^\pi_s \big[  \gamma_1^{\tau} \big] = 1 - \tilde P^e_{s\Delta} = 1, \qquad \forall \, s \in \tilde \S,
\end{equation}  
implying that given $S_0 \in \tilde \S$, $\gamma_1^{\tau} = 1$ a.s.\ (since $\gamma_1^\tau \in [0,1]$). 
Then by (\ref{gbo-prf1}),
\begin{equation} \label{gbo-prfb}
   \tilde P_{ss'} = \E^\pi_s \big[  \I(S_{\tau} = s') \big], \qquad \forall \, s, s' \in \tilde \S.
\end{equation}  

Observe from (\ref{gbo-prfa}) that given $S_0 \in \tilde \S$, the event $\{\tau = \infty\}$ has zero probability. This is because under Condition~\ref{cond-pol}(i), $\gamma_1^t \asto 0$ (as $t \to \infty$) and consequently $\gamma_1^\infty : = \prod_{t=1}^\infty \gamma_t = 0$ a.s. Indeed, the existence of the inverse $(I - \P \Gamma)^{-1} = \sum_{t = 0}^\infty (\P \Gamma)^t$ under Condition~\ref{cond-pol}(i) implies $(\P \Gamma)^t \to 0$ as $t \to \infty$. Since the $s$th entry of $(\P \Gamma)^t \1$ equals $\E^\pi_s [ \gamma_1^t]$, we have $\E^\pi_s [ \gamma_1^t] \to 0$ as $t \to \infty$. As the nonnegative sequence $\{\gamma_1^t\}_{t \geq 1}$ is nonincreasing (since each $\gamma_t \in [0,1]$), this implies, by Fatou's lemma \cite[Lemma 4.3.3]{Dud02}, that $\gamma_1^t \asto 0$. Thus if $S_0 \in \tilde \S$, $\tau$ is almost surely finite. 

We now consider a Markov chain $\{S_t\}$ with transition matrix $\P$ and $S_0 \in \tilde \S$. We will extract from it a Markov chain $\{\tilde S_k\}_{k \geq  0}$ with transition matrix $\tilde P^e$ on the recurrent class $\tilde \S$, by employing multiple stopping times. We will show $\gamma_1^\infty = 1$ a.s., contrary to the fact $\gamma_1^\infty = 0$ a.s.\ just discussed.

For ease of explanation, let us imagine that there is a device that if we give it a sequence of states $S_0, S_1, \ldots$ generated according to $\P$, it will output the stopping decision at a random time $\tau$ that is exactly the randomized stopping time $\tau$ associated with the operator $T$. (Because $\tau$ is a randomized stopping time for $\{S_t\}$, a device that correctly implements $\tau$ does not affect the evolution of $\{S_t\}$, so we can give $\{S_t\}$ to the device as inputs.) 

Let $\{S_t\}$ start from some state $S_0=s \in \tilde \S$. 
We use the device just mentioned to generate the first randomized stopping time $\tau_1$. 
We then reset the device so that it now ``sees'' the time-shifted process $\{S_{\tau_1+t'} \mid t' \geq 0\}$ with the initial state being $S_{\tau_1}$.
We wait till the device makes another stopping decision, and we designate that time by $\tau_2$. 
We repeat this procedure as soon as the device makes yet another stopping decision. This gives us a nondecreasing sequence 
$0 \leq \tau_1 \leq \tau_2 \leq \cdots$.

Since $S_0 = s \in \tilde \S$, by what we proved earlier, $\tau_1$ is almost surely finite. Let $\tilde S_1 = S_{\tau_1}$ and $\tilde S_0 = S_0$. 
The transition from $\tilde S_0$ to $\tilde S_1$ is according to the transition matrix $\tilde P^e$ by construction (cf.\ (\ref{gbo-prfb})).
Since $\tilde \S$ is a recurrent class for $\tilde P^e$, we must have $\tilde S_1 \in \tilde \S$ almost surely. 
Then, repeating the same argument and using induction, we have that almost surely, for all $k \geq 1$, $\tau_k$ is defined and finite and $\tilde S_k : = S_{\tau_k} \in \tilde \S$.
Thus we obtain an infinite sequence $\{\tilde S_k\}$, which is a recurrent Markov chain on $\tilde \S$ with its transition matrix given by the corresponding submatrix of $\tilde P^e$.

Now in view of (\ref{gbo-prfa}), almost surely, 
\begin{equation} \label{gbo-prfc}
   \gamma_1^{\tau_1} = 1, \quad \gamma_{\tau_1+1}^{\tau_2} = 1, \ \ \cdots, \ \  \gamma_{\tau_k+1}^{\tau_{k+1}} = 1, \ \ \cdots
\end{equation}   
(recall that if $\tau_{k} + 1 > \tau_{k+1}$, $\gamma_{\tau_k+1}^{\tau_{k+1}} = 1$ by definition). 
Consider first a simpler case of the randomized stopping time $\tau$ that defines $T$: for any initial state $S_0$, $\tau \geq 1$ a.s. Then, $\{\tau_k\}$ is strictly increasing, and by multiplying the variables in (\ref{gbo-prfc}) together, we have $\gamma_1^\infty = 1$ a.s. 
For the general case of $\tau$ assumed in the theorem, $\tau = 0$ is possible, but $\Pr^\pi(\tau \geq 1 \mid S_0 = s) > 0$ for all states $s \in \S$.
This means that the event of $\tau_k$ being the same for all $k$ greater than some (random) $\bar k$ has probability zero. So $\{\tau_k\}$ must converge to $+\infty$ almost surely, and we again obtain, by multiplying the variables in (\ref{gbo-prfc}) together, that $\gamma_1^\infty = 1$ a.s. 
This contradicts the fact proved earlier, namely, that for any initial state $S_0$, $\gamma_1^\infty = 0$ a.s. So the assumption of a recurrent class $\tilde \S \subset \S$ for $\tilde P^e$ must be false. This proves Theorem~\ref{thm-gbo}.

\section{Oblique Projection Viewpoint and Error Bound for TD} \label{appsec-oblproj}

In this appendix we first explain Scherrer's interpretation of TD solutions as oblique projections~\citep{bruno-oblproj}, and we then give approximation error bounds for TD similar to those given by \cite{yb-errbd}, which do not rely on contraction properties. We will explain these properties of TD in the context of generalized Bellman operators discussed in this paper. Although this was not the framework used in \citep{bruno-oblproj,yb-errbd} and our setup here is more general than the one discussed in those previous papers, the arguments and reasoning are essentially the same.

\subsection{Solutions of TD as Oblique Projections of the Value Function} \label{sec-B.1}

Let us start with the projected Bellman equation associated with TD/LSTD that we noted in Remark~\ref{rem-1}:
$$T v - v \perp_{\zeta_\S} \L_\phi, \ \ \  v \in \L_\phi,$$
where $T$ is a generalized Bellman operator with $v_\pi$ as its unique fixed point, and $\L_\phi$ is the approximation subspace. 
We can write this equation equivalently as
\begin{equation} \label{eq-projbellman1}
   v = \Pi_{\zeta_\S} T v,
\end{equation}
where $\Pi_{\zeta_\S}$ denotes the projection onto the approximation subspace $\L_\phi$ with respect to the $\zeta_\S$-weighted Euclidean norm.
In the subsequent derivations, we will not use the fact that $\zeta_\S$ is the invariant probability measure induced by the behavior policy on $\S$, so the analyses we give in this appendix apply to any weighted Euclidean norm $\| \cdot\|_{\zeta_\S}$.

Scherrer \citeyearpar{bruno-oblproj} first realized that the solution of the projected Bellman equation (\ref{eq-projbellman1}) can be viewed as an \emph{oblique projection of the value function $v_\pi$} on the approximation subspace $\L_\phi$. This viewpoint provides an intuitive geometric interpretation of the TD solution and explains conceptually the source of its approximation bias. Analytically, this view also gives tight bounds on the approximation bias, as we will elaborate later in Section~\ref{sec-B.2}.

An oblique projection is defined by two nonorthogonal subspaces of equal dimensions: it is the projection onto the first subspace orthogonally to the second \citep{saad}, as illustrated in Figure~\ref{fig-oblproj}. More precisely, for any two $n$-dimensional subspaces $\L_1, \L_2$ of $\re^N$ such that no vector in $\L_2$ is orthogonal to $\L_1$, there is an associated \emph{oblique projection operator} $\Pi_{\L_1\L_2} : \re^N \to \L_1$ defined by
\begin{equation}
   \Pi_{\L_1\L_2} \, x \in \L_1, \qquad x - \Pi_{\L_1\L_2} \, x \perp \L_2, \qquad \forall \, x \in \re^N. 
\end{equation}   
If the two subspaces are the same: $\L_1=\L_2$ or if $x$ lies in $\L_1$, then the oblique projection $\Pi_{\L_1\L_2} x$ is the same as $\Pi_{\L_1} x$, the orthogonal projection of $x$ onto $\L_1$. In general this need not be the case and $\Pi_{\L_1\L_2} \, x \not= \Pi_{\L_1} x$ typically (cf.\ Figure~\ref{fig-oblproj}). A matrix representation of the projection operator $\Pi_{\L_1\L_2}$ is given by 
\begin{equation} \label{rep-oblproj}
  \Pi_{\L_1\L_2} = \Phi_1 \big(\Phi_2^\top \Phi_1 \big)^{-1} \Phi_2^\top,
\end{equation}  
where $\Phi_1$ and $\Phi_2$ are $N \times n$ matrices whose columns form a basis of $\L_1$ and $\L_2$, respectively (see~\citealt[Chap.\ 1.12]{saad}). For comparison, a matrix representation of the orthogonal projection operator $\Pi_{\L_1}$ is $\Pi_{\L_1} = \Phi_1 \big(\Phi_1^\top \Phi_1 \big)^{-1} \Phi_1^\top$. (Below we will use the same notation for a projection operator and its matrix representations.) 

\begin{figure}[!t] 
   \centering
   \includegraphics[width=0.5\linewidth]{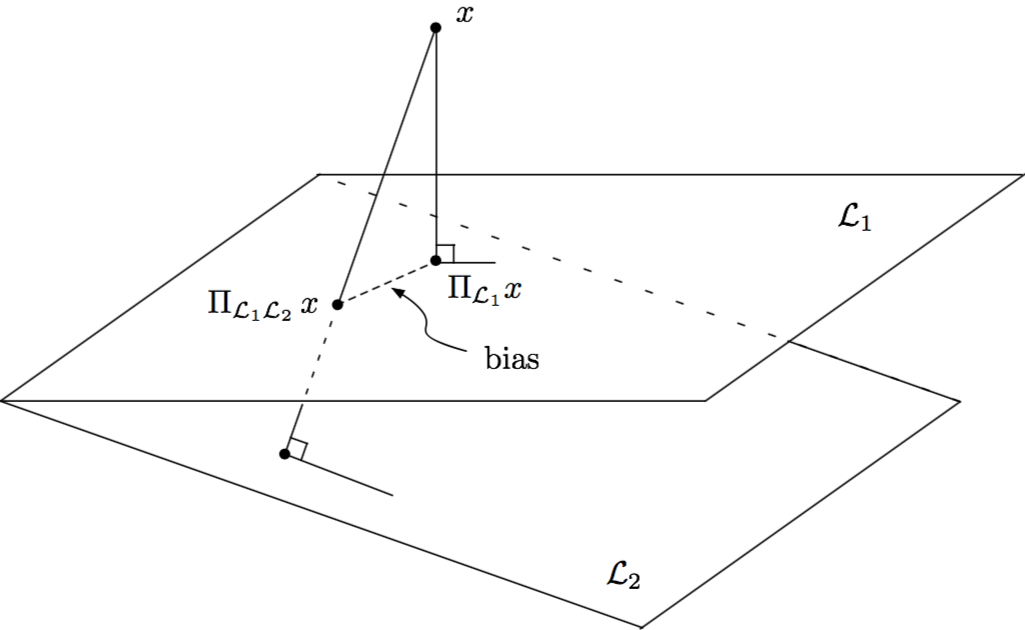} 
   \caption{Oblique projection of $x$ onto the subspace $\L_1$ orthogonally to the subspace $\L_2$.} 
   \label{fig-oblproj}
\end{figure}

Back to the projected Bellman equation~(\ref{eq-projbellman1}), let us assume it has a unique solution $\va$ and express $\va$ in terms of $v_{\pi}$. We have $\va = \Pi_{\zeta_\S} T \va$.
By Theorem~\ref{thm-gbo}, $v_{\pi}$ is the unique solution of $v = Tv$. Recall that the generalized Bellman operator $T$ is affine and can be expressed as $T v = \tilde r_\pi+ \tilde P_\pi v$ for some vector $\tilde r_\pi$ and substochastic matrix $\tilde P_\pi$.
It follows that $\tilde r_\pi = (I - \tilde P_\pi) \,v_\pi$ and
\begin{equation} \label{eq-vtd1}
\big(I - \Pi_{\zeta_\S} \tilde P_\pi \big) \, \va = \Pi_{\zeta_\S} \tilde r_\pi =  \Pi_{\zeta_\S} \big(I - \tilde P_\pi \big) \, v_\pi. 
\end{equation}
Let the columns of $\Phi$ form a basis of the approximation subspace $\L_\phi$, and let $D$ be the diagonal matrix with $\zeta_{\S}$ as its diagonal elements. Then a matrix representation of the projection operator $\Pi_{\zeta_\S}$ is 
$\Pi_{\zeta_\S}  = \Phi \big( \Phi^\top D \, \Phi \big)^{-1} \Phi^\top D.$ 
Using this representation of $\Pi_{\zeta_\S}$ and the fact $\va \in \L_\phi$, we obtain from (\ref{eq-vtd1}) an expression of $\va$ in terms of $v_\pi$:
\begin{equation} \label{eq-vtd2}
  \va = \Phi \big( \Phi^\top D (I - \tilde P_\pi) \, \Phi \big)^{-1} \Phi^\top D (I - \tilde P_\pi) \, v_\pi.
\end{equation}
(Here the invertibility of the matrix $\Phi^\top D (I - \tilde P_\pi) \Phi$ is equivalent to our assumption that $\va$ is the unique solution of (\ref{eq-projbellman1}).)

Let us compare the expression (\ref{eq-vtd2}) with (\ref{rep-oblproj}). We see that if the geometry on $\Re^N$ is determined by the usual Euclidean norm $\| \cdot\|_2$, then $\va$ is an oblique projection of $v_\pi$ for the two subspaces $\L_1 = \L_\phi$ and $\L_2 = \text{column-space}\big((I - \tilde P_\pi)^\top D \Phi\big)$. 

Alternatively, consider the case where the geometry on $\Re^N$ is determined by some weighted Euclidean norm $\| \cdot\|_{\xi}$ with weights $\xi$ (for example, $\xi=\zeta_\S$, which is one of the cases of interest in off-policy learning). In this case we can first scale each coordinate by the square root $\sqrt{\xi(s)}$ of its weight to reduce the case to that of the norm $\|\cdot\|_2$. Specifically, let $\perp_{\xi}$ denote orthogonality with respect to $\|\cdot\|_\xi$, and let $\Xi$ denote the diagonal matrix that has $\xi$ as its diagonal elements.
For the linear mapping $h: v \mapsto \Xi^{1/2} v$, we have
\begin{equation} \label{eq-hscale}
 \| v\|_{\xi} = \| h(v) \|_2, \qquad \text{and} \quad v_1 \perp_{\xi} v_2 \ \ \ \Leftrightarrow \ \ \ h(v_1) \perp h(v_2).
\end{equation}   
The second relation means that $\bar v$ is an oblique projection of $v$ for two subspaces $\L_1, \L_2$ with respect to $\|\cdot\|_{\xi}$ 
(i.e., $\bar v \in \L_1$ and $ v - \bar v \perp_{\xi} \L_2$), 
if and only if $h(\bar v)$ is an oblique projection of $h(v)$ for the two subspaces $h(\L_1), h(\L_2)$ with respect to $\|\cdot\|_2$:
$$ h(\bar v) \in h(\L_1), \qquad h(v) - h(\bar v) \perp h(\L_2).$$

With these facts in mind, we rewrite (\ref{eq-vtd2}) equivalently as follows:
\begin{align}
 \Xi^{1/2}  \va & = \Xi^{1/2} \Phi \cdot \big( \Phi^\top D  (I - \tilde P_\pi) \, \Xi^{-1} \Xi^{1/2}  \cdot \Xi^{1/2} \Phi \big)^{-1} \cdot \Phi^\top D (I - \tilde P_\pi) \, \Xi^{-1} \Xi^{1/2}  \cdot \Xi^{1/2}  v_\pi \notag \\
 \Longrightarrow \quad h(\va) & = h(\Phi) \cdot (h(\Phi_2)^\top \cdot h(\Phi) \big)^{-1} h(\Phi_2)^\top \cdot h(v_\pi),  \label{eq-vtd3a}
\end{align}
where 
\begin{equation}
  \Phi_2 = \Xi^{-1}  (I - \tilde P_\pi)^\top D  \Phi \label{eq-vtd3b}
\end{equation}  
and $h$ applied to a matrix denotes the result of applying $h$ to each column of that matrix. 
Comparing (\ref{eq-vtd3a}) with (\ref{rep-oblproj}), we see that $h(\va)$ is an oblique projection of $h(v_\pi)$ with respect to $\| \cdot \|_2$ for the two subspaces 
$h(\L_1)$, $h(\L_2)$, where
\begin{equation} \label{eq-vtd3c}
 \L_1 = \L_\phi= \text{column-space}(\Phi), \qquad \L_2 = \text{column-space}(\Phi_2).
\end{equation}
So based on the discussion earlier, with respect to the weighted Euclidean norm $\| \cdot\|_{\xi}$ on $\Re^N$, $\va$ is an oblique projection of $v_\pi$ for the two subspaces $\L_1, \L_2$.

Note that by (\ref{eq-vtd3b}), the second subspace $\L_2$ defining the above oblique projection is the image of the approximation subspace $\L_\phi$ under the linear transformation $\Xi^{-1}  (I - \tilde P_\pi)^\top \!D$. Thus $\L_2$ \emph{depends on the dynamics induced by the target policy as well as the generalized Bellman operator $T$ that we choose}. Relating the oblique projection interpretation of $\va$ to Figure~\ref{fig-oblproj}, we can see where the approximation bias of TD, $\va - \Pi_{\xi} v_\pi$, comes from.

\subsection{Approximation Error Bound} \label{sec-B.2}

We now consider the approximation error of $\va$ and use the oblique projection viewpoint to derive a sharp bound on the approximation bias $\va - \Pi_{\xi} v_\pi$.
Before proceeding, however, let us first remind the reader that unless the norm $\|\cdot\|_{\zeta_\S}$ for the projection operator $\Pi_{\zeta_\S}$ is purposefully chosen, the composition of $\Pi_{\zeta_\S}$ with a generalized Bellman operator $T$ is usually not a contraction, and thus error bounds for projected generalized Bellman equations usually cannot be obtained with contraction-based arguments. This is the case even for on-policy learning, as the following example shows.

\begin{myexample}[Non-contractive $\Pi_{\zeta_\S} T$] \rm \label{counter-ex-contraction}
If the target policy $\pi$ induces an irreducible Markov chain with invariant probability measure $\zeta_\S$, and if $T$ is the Bellman operator for TD($\lambda$) with a constant $\lambda$, then, as Tsitsiklis and Van Roy \citeyearpar{tr-disc} showed, $\Pi_{\zeta_\S} T$ is a contraction operator w.r.t.\ the weighted Euclidean norm $\| \cdot \|_{\zeta_\S}$. Consequently, the matrix $\Phi^\top D (\tilde P_\pi - I) \, \Phi$ associated with the TD($\lambda$) algorithm is negative definite \citep{tr-disc}. The derivation of the contraction property of $\Pi_{\zeta_\S} T$ in this case relies critically on the inequality $\zeta_{\S}^\top \tilde P_{\pi}  < \zeta_{\S}^\top$. 
This inequality generally does not hold for the substochastic matrix $\tilde P_{\pi}$ in the generalized Bellman operator $T$, when $\lambda$ is not constant. 
So, for non-constant $\lambda$, we can no longer expect $\Pi_{\zeta_\S} T$ to be a contraction or the matrix $\Phi^\top D (\tilde P_\pi - I) \, \Phi$ to be negative definite. 

As an example, consider a simple two-state problem in which the system under the target policy $\pi$ moves from one state to another in a cycle. 
Let $\pi^o=\pi$, let the discount factor $\gamma$ be a constant, and let $\lambda$ be a function of states with $\lambda(1) = 0$, $\lambda(2) = 1$.
Then $\zeta_{\S}^\top=(0.5, 0.5)$ and
$\tilde P_\pi = \left( \begin{matrix} 
      \gamma^2 & 0 \\
      \gamma & 0 
      \end{matrix} \right)$.
For $\gamma$ near $1$, e.g., $\gamma = 0.95$, and for $\Phi$ as given below, we can calculate the $\zeta_\S$-weighted norm of $\Pi_{\zeta_\S} \tilde P_\pi$ and the matrix associated with TD($\lambda$):
\begin{align*}
& \Phi  = \left( \begin{matrix} 
      3 & 1 \\
      1 & 1 
      \end{matrix} \right), \qquad    \big\| \Pi_{\zeta_\S} \tilde P_\pi \big\|_{\zeta_\S} =  \big\| \tilde P_\pi \big\|_{\zeta_\S} \approx 1.31 > 1,  \\      
&  \Phi^\top D (\tilde P_{\pi} - I) \,\Phi  = \left( \begin{matrix} 
      0.4862  &  -0.1713 \\
   0.7787   &  -0.0738
      \end{matrix} \right). 
\end{align*}      
The latter matrix is not negative definite; in fact, its eigenvalues have positive real parts, so it is not even a Hurwitz matrix (TD($\lambda$) can diverge in this case). In the above, we have $\Pi_{\zeta_\S} \tilde P_\pi = \tilde P_\pi$, and while $\Pi_{\zeta_\S} \tilde P_\pi$ is not a contraction w.r.t.\ $\| \cdot\|_{\zeta_\S}$, it is still a contraction w.r.t.\ some matrix norm. 
If we now let $\Phi = (3, 1)^\top$ instead, then the spectral radius of the matrix $\Pi_{\zeta_\S} \tilde P_\pi$ comes out as $\sigma(\Pi_{\zeta_\S} \tilde P_\pi) \approx  1.10  > 1$, so $\Pi_{\zeta_\S} \tilde P_\pi$ (and hence $\Pi_{\zeta_\S} T$) cannot be a contraction w.r.t.\ any matrix norm. 
\qed \end{myexample}

We now proceed to bound the bias term $\va - \Pi_{\xi} v_\pi$ relative to $\| v_\pi - \Pi_{\xi} v_\pi\|_{\xi}$, the distance between $v_\pi$ and the approximation subspace measured with respect to $\|\cdot\|_\xi$. It is more transparent to derive the bound for the case of a general oblique projection operator $\Pi_{\L_1\L_2}$ with respect to the usual Euclidean norm $\| \cdot\|_2$, so let us do that first and then use the linear transformation $h(\cdot)$ to translate the result to TD, as we did earlier in the preceding subsection.

We bound the bias $\| \Pi_{\L_1\L_2} x - \Pi_{\L_1} x \|_2$ relative to $\| x - \Pi_{\L_1} x\|_2$, by calculating
\begin{equation} \label{eq-k1}
 \k : =  \sup_{x \in \re^N} \frac{\| \Pi_{\L_1\L_2} x - \Pi_{\L_1} x \|_2}{\| x - \Pi_{\L_1} x\|_2} = \sup_{x \in \re^N} \frac{\| \Pi_{\L_1\L_2} (x - \Pi_{\L_1} x ) \|_2}{\| x - \Pi_{\L_1} x\|_2} 
\end{equation} 
(where we treat $0/0=0$). This constant $\k$ depends on the two subspaces $\L_1$, $\L_2$, and reflects the ``angle'' between them. 
It has several equivalent expressions, e.g.,
\begin{equation}  \label{eq-k2}
    \k =  \sup_{x \perp \L_1, \, \| x \|_2 = 1} \| \Pi_{\L_1\L_2} x \|_2 = \| \Pi_{\L_1\L_2} (I - \Pi_{\L_1}) \|_2,
\end{equation}
or with $\sigma(F)$ denoting the spectral radius of a square matrix $F$,
\begin{equation}
\qquad \k = \sqrt{ \sigma \Big( \Pi_{\L_1\L_2} (I - \Pi_{\L_1})  \cdot (I - \Pi_{\L_1})^\top \, \Pi_{\L_1\L_2}^\top \Big)} = \sqrt{ \sigma \big( \Pi_{\L_1\L_2}\Pi_{\L_2\L_1} - \Pi_{\L_1} \big)}.
  \label{eq-k3}
\end{equation}
(After the definition of $\k$, each expression of $\k$ in (\ref{eq-k1})-(\ref{eq-k3}) follows from the preceding one; in particular, for the last expression in (\ref{eq-k3}), we used the fact that $\Pi_{\L_1}^\top  = \Pi_{\L_1}, \Pi_{\L_1\L_2}^\top = \Pi_{\L_2\L_1}$, and $\Pi_{\L_1} \Pi_{\L_2\L_1} = \Pi_{\L_1}$.)

We can express $\k$ in terms of the spectral radius of an $n \times n$ matrix, similarly to what was done in \citep{yb-errbd}. In particular, we take the last expression of $\k$ in (\ref{eq-k3}) and rewrite the symmetric matrix in that expression using the matrix representations of the projection operators as follows:
$$ \Pi_{\L_1\L_2}\Pi_{\L_2\L_1} - \Pi_{\L_1} = \Phi_1 (\Phi_2^\top \Phi_1)^{-1} \Phi_2^\top \cdot \Phi_2 (\Phi_1^\top \Phi_2)^{-1} \Phi_1^\top - \Phi_1 (\Phi_1^\top \Phi_1)^{-1} \Phi_1^\top.  $$
By a result in matrix theory~\cite[Theorem 1.3.20]{hj85-matrix}, for any $N \times n$ matrix $F_1$ and $n \times N$ matrix $F_2$,
$\sigma(F_1 F_2) = \sigma(F_2 F_1)$. Applying this result to the preceding expression with $F_1 = \Phi_1$, we have that $\sigma \big(\Pi_{\L_1\L_2}\Pi_{\L_2\L_1} - \Pi_{\L_1}\big)$ is equal to the spectral radius of the matrix
\begin{align*}
    (\Phi_2^\top \Phi_1)^{-1} (\Phi_2^\top \Phi_2) (\Phi_1^\top \Phi_2)^{-1} (\Phi_1^\top \Phi_1) - I. 
\end{align*}    
Combing this with (\ref{eq-k3}), we obtain that
\begin{equation} \label{eq-k4}
   \k^2 =  \sigma(F) - 1, \qquad \text{where} \ \ \ F= (\Phi_2^\top \Phi_1)^{-1} (\Phi_2^\top \Phi_2) (\Phi_1^\top \Phi_2)^{-1} (\Phi_1^\top \Phi_1).
\end{equation}   
Thus, for the above $\k$, the bound below holds for all $x \in \re^N$ and with equality attained at some $x$:
\begin{equation} \label{eq-k4b}
  \| \Pi_{\L_1\L_2} x - \Pi_{\L_1} x \|_2 \leq \k \, \| x - \Pi_{\L_1} x\|_2.
\end{equation} 

We now translate the result (\ref{eq-k4})-(\ref{eq-k4b}) to our TD context. We want to bound the relative bias $\big\|\va - \Pi_{\xi} v_\pi \big\|_{\xi}/\big\| v_\pi - \Pi_{\xi} v_\pi \big\|_{\xi}$. As discussed earlier in Section~\ref{sec-B.1}, we can replace $\| \cdot\|_\xi$ with $\|\cdot\|_2$ by using the linear transformation $h(\cdot)$ to scale the coordinates. In particular, by the two relations given in (\ref{eq-hscale}),
\begin{equation} \label{eq-errbd0}
  \frac{\big\|\va - \Pi_{\xi} v_\pi \big\|_{\xi}}{\big\| v_\pi - \Pi_{\xi} v_\pi \big\|_{\xi}} = \frac{\big\| h(\va) - \Pi \, h(v_\pi) \big\|_2}{\big\| h(v_\pi) - \Pi \, h(v_\pi) \big\|_2},
\end{equation}  
where $\Pi$ on the r.h.s.\ stands for the orthogonal projection onto $h(\L_\phi)$ with respect to $\|\cdot\|_2$.
As shown by (\ref{eq-vtd3a}), $h(\va)$ is an oblique projection of $h(v_\pi)$ for the two subspaces $h(\L_\phi)$ and $h(\L_2)$, 
where $\L_2$ is given by (\ref{eq-vtd3b})-(\ref{eq-vtd3c}). 
According to (\ref{eq-k4}), the constant $\k$ for this oblique projection is $\sqrt{\sigma(F) -1}$ 
where, if we take $\Phi_1 = \Phi$ and $\Phi_2$ as defined by (\ref{eq-vtd3b}), $F$ is now given by the expression in (\ref{eq-k4}) with $h(\Phi_1), h(\Phi_2)$ in place of $\Phi_1, \Phi_2$, respectively.
Thus, by (\ref{eq-k4b}) and (\ref{eq-errbd0}) we obtain that
\begin{equation} \label{eq-errbd}
  \big\|\va - \Pi_{\xi} v_\pi \big\|_{\xi} \leq \k \, \big\| v_\pi - \Pi_{\xi} v_\pi \big\|_{\xi} \qquad \text{for} \ \ \k =  \sqrt{\sigma(F) -1},
\end{equation}
where $F$ is an $n\ \times n$ matrix given by
\begin{equation}
 F= \big[h(\Phi_2)^\top h(\Phi_1)\big]^{-1} \cdot \big[h(\Phi_2)^\top h(\Phi_2)\big] \cdot \big[h(\Phi_1)^\top h(\Phi_2)\big]^{-1} \cdot \big[h(\Phi_1)^\top h(\Phi_1)\big]
\end{equation}
for
\begin{equation}
   \Phi_1 = \Phi, \qquad \Phi_2 = \Xi^{-1}  (I - \tilde P_\pi)^\top D  \Phi,
\end{equation}  
or more explicitly, after substituting the expressions of $h(\Phi_1)$ and $h(\Phi_2)$ in the formula of $F$ and removing $h$, we have 
$$F = \big(\Psi^\top \Phi\big)^{-1} \big(\Psi^\top \Xi^{-1} \Psi\big) \big(\Phi^\top \Psi\big)^{-1} \big(\Phi^\top \Xi \Phi\big), \quad \text{where} \ \ \ 
 \Psi = (I - \tilde P_\pi)^\top D  \Phi.$$

Note that by the definition of $\k$, the bound (\ref{eq-errbd}) is a worst-case bound that depends only on the two subspaces involved in the oblique projection operator. 
In other words, given the approximation subspace $\L_\phi$, the dynamics described by $\tilde P_\pi$, the projection norm $\|\cdot\|_{\zeta_\S}$ and the norm $\|\cdot\|_\xi$ for measuring the approximation quality, the bound (\ref{eq-errbd}) is attained by a worst-case choice of the rewards $\tilde r_\pi$ for the target policy. In this sense, the bound (\ref{eq-errbd}) is tight.

} 

\addcontentsline{toc}{section}{References}
\bibliographystyle{apa}
\let\oldbibliography\thebibliography
\renewcommand{\thebibliography}[1]{%
  \oldbibliography{#1}%
  \setlength{\itemsep}{0pt}%
}
{\fontsize{9}{11} \selectfont
\bibliography{gbe_final_bib}
}

\end{document}